\ifdefined\ORStyle
  \documentclass[opre,dblanonrev]{or-style/informs4}
  \RequirePackage{tgtermes}
  \IfFileExists{newtxtext.sty}{\RequirePackage{newtxtext}}{}
  \IfFileExists{newtxmath.sty}{\RequirePackage{newtxmath}}{}
  \RequirePackage{bm}

  \OneAndAHalfSpacedXII           % current OPRE default line spacing
  \TheoremsNumberedThrough        % Theorem 1, Lemma 1, Theorem 2, ...
  \ECRepeatTheorems
  \EquationsNumberedThrough       % (1), (2), ...
  \MANUSCRIPTNO{}                 % blank for a new submission

  \RequirePackage{endnotes}
  \let\footnote=\endnote

  \usepackage[numbers]{natbib}
  \bibpunct[, ]{[}{]}{,}{n}{}{,}%
  \def\bibfont{\small}%
  \def\bibsep{\smallskipamount}%
  \usepackage{macros/packages-or}
  \usepackage{hyperref}
  \usepackage{macros/editing-macros}
  \usepackage{macros/statistics-macros-or}
  \usepackage{macros/formatting}
  \usepackage{macros/misc_last_macros}

  \usepackage{pgfplotstable}
  \usepackage{subcaption}
  \usepackage{float}
  \usepackage{tabularx}
  \usepackage{overpic}
  \usepackage{tikz}
  \usepackage{rotating}
  \usepackage{enumitem}
  \IfFileExists{macros/formatting-or.sty}{\usepackage{macros/formatting-or}}{}
\else
  \documentclass[11pt]{article}
  \usepackage[numbers]{natbib}
  \usepackage{macros/packages}
  \usepackage{macros/editing-macros}
  \usepackage{macros/formatting}
  \usepackage{macros/statistics-macros}
  \usepackage{macros/misc_last_macros}
\fi

\usepackage{placeins}
\usepackage{needspace}
\newcommand{\appendixcontentsstart}{\let\contentsline\appendixcontentsline}
\graphicspath{{./}{or-style/}}      % or-style/ so the class finds informs_Logo

\ifdefined\ORStyle
  \RUNAUTHOR{Cai, Fonseca, Hou, and Namkoong}
  \makeatletter
  \if@BLINDREV
    \ECRUNAUTHOR{}
    \ECAUpunct{}
  \fi
  \makeatother
  \RUNTITLE{Constrained Learning for Causal Inference
 
}
  \TITLE{Constrained Learning for Causal Inference
 
}
  \ARTICLEAUTHORS{%
    \AUTHOR{Tiffany (Tianhui) Cai}
    \AFF{Columbia University, \EMAIL{tiffany.cai@columbia.edu}}
    \AUTHOR{Yuri Fonseca}
    \AFF{University College London, \EMAIL{y.fonseca@ucl.ac.uk}}
    \AUTHOR{Kaiwen Hou}
    \AFF{UC Berkeley, \EMAIL{kaiwen.hou@berkeley.edu}}
    \AUTHOR{Hongseok Namkoong}
    \AFF{Decision, Risk, and Operations Division, Columbia Business School,
      \EMAIL{namkoong@gsb.columbia.edu}}
  }
  \ABSTRACT{Debiasing methods such as augmented inverse propensity weighting (AIPW), and
targeted maximum likelihood estimation (TMLE) enjoy asymptotic properties like
semiparametric efficiency and double robustness, but can produce unstable
estimates in practice that require ad hoc adjustments (e.g., truncating
propensity scores). In contrast, simple plug-ins can remain stable but lack
these asymptotic guarantees. To achieve the best of both worlds---a plug-in
that enjoys strong asymptotic guarantees---we propose a constrained learning
framework that trains a nuisance model to minimize prediction error subject to
the constraint that the estimated first-order error of the resulting plug-in
is zero. To compare different debiasing methods that share the same classical
limit, we study a stylized high-dimensional regression problem where nuisance
estimation errors do not vanish asymptotically. Our unified analysis covers
both $d<n$ and $d>n$, as well as ridge regularization, and characterizes how
overlap affects the estimators' limiting distributions. Under sufficient
overlap, our estimator has smaller asymptotic variance than AIPW and TMLE,
whereas when overlap deteriorates so much that AIPW and TMLE are no longer
root-$n$ consistent, constrained learning still retains the direct plug-in's
root-$n$ limit. Empirically, across a range of experimental settings including
those with text-based covariates and language models,
we observe our estimator outperforms classical debiasing methods in challenging settings 
with limited overlap between treatment and control, and performs similarly otherwise.

}
\fi

\begin{document}

\ifdefined\ORStyle
  %% ------------------- INFORMS front matter ------------------------ %%
  \maketitle
\else
  %% ------------------- Article front matter ------------------------ %%
  % Control whitespace around equations
  \abovedisplayskip=8pt plus0pt minus3pt
  \belowdisplayskip=8pt plus0pt minus3pt

  \begin{center}
    {\huge Constrained Learning for Causal Inference
 
} \\
    \vspace{.5cm} {\large Tiffany (Tianhui) Cai$^{\dagger}$ ~~~
      Yuri Fonseca$^{\ddagger}$ ~~~ Kaiwen Hou$^+$~~~
      Hongseok Namkoong$^\dagger$} \\
    \vspace{.2cm}
    {\large $^\dagger$Columbia University, $^\ddagger$University College
      London, $^+$UC Berkeley} \\
    \vspace{.2cm}
    \texttt{tiffany.cai@columbia.edu, y.fonseca@ucl.ac.uk,
      kaiwen.hou@berkeley.edu, namkoong@gsb.columbia.edu}
  \end{center}

  \begin{abstract}
    Debiasing methods such as augmented inverse propensity weighting (AIPW), and
targeted maximum likelihood estimation (TMLE) enjoy asymptotic properties like
semiparametric efficiency and double robustness, but can produce unstable
estimates in practice that require ad hoc adjustments (e.g., truncating
propensity scores). In contrast, simple plug-ins can remain stable but lack
these asymptotic guarantees. To achieve the best of both worlds---a plug-in
that enjoys strong asymptotic guarantees---we propose a constrained learning
framework that trains a nuisance model to minimize prediction error subject to
the constraint that the estimated first-order error of the resulting plug-in
is zero. To compare different debiasing methods that share the same classical
limit, we study a stylized high-dimensional regression problem where nuisance
estimation errors do not vanish asymptotically. Our unified analysis covers
both $d<n$ and $d>n$, as well as ridge regularization, and characterizes how
overlap affects the estimators' limiting distributions. Under sufficient
overlap, our estimator has smaller asymptotic variance than AIPW and TMLE,
whereas when overlap deteriorates so much that AIPW and TMLE are no longer
root-$n$ consistent, constrained learning still retains the direct plug-in's
root-$n$ limit. Empirically, across a range of experimental settings including
those with text-based covariates and language models,
we observe our estimator outperforms classical debiasing methods in challenging settings 
with limited overlap between treatment and control, and performs similarly otherwise.

%%% Local Variables:
%%% mode: latex
%%% TeX-master: "../main"
%%% End:

  \end{abstract}
\fi

%% ========================== Main sections =========================== %%

\section{Introduction}
\label{section:intro}

Many causal and semiparametric estimands depend on high-dimensional nuisance
functions that must be learned before the target can be estimated. A direct
plug-in estimator substitutes flexible estimates of these functions into the
target functional, but errors in nuisance estimation can induce first-order
bias. Debiasing methods use a first-order distributional expansion to estimate
and remove this error~\citep{BickelKlRiWe98,pfanzagl1985contributions,
  ChernozhukovChDeDuHaNeRo18,VanDerLaanRo11,van2006targeted,Newey94}, using
either one-step estimation or targeting; when estimating the average treatment
effect (ATE), these corrections respectively produce the augmented inverse
probability weighted estimator (AIPW~\citep{Bang2005DoublyRE}) and targeted
maximum likelihood estimation (TMLE~\citep{van2006targeted,VanDerLaanRo11}).
These estimators enjoy strong asymptotic guarantees such as semiparametric
efficiency and double robustness. These guarantees compare estimators through
their first-order limiting distributions, but estimators with the same limit
can behave differently at the sample sizes used in practice.

Asymptotic optimality is therefore not the end of the story. In a widely used
simulation, \citet{KangSc07} show that debiased estimators can have large
finite-sample errors when both nuisance models are misspecified, particularly
when fitted propensity weights become extreme, while a direct outcome-model
plug-in remains stable. More broadly, the finite-sample performance of
first-order corrections can depend strongly on their implementation. In
practice, analysts therefore use additional heuristics and assumptions, such
as truncating propensity scores for AIPW or using an empirically scaled,
logistic-link update for TMLE~\citep{ding2023course,cole2008constructing,VanDerLaanRo11}.

In this work, we focus on the stability offered by simple plug-ins, whose
final estimate averages fitted predictions. To construct asymptotically
optimal estimators that retain this plug-in form, we propose a constrained
learning framework that incorporates first-order corrections directly when
training the nuisance models (\Cref{section:methods}). Our construction applies
when the first-order error of a plug-in estimator can be imposed as a
constraint on a learned nuisance function. Rather than fit the nuisance
function for prediction and then correct the resulting plug-in as standard
debiasing methods do, we train it to minimize prediction error subject to the
constraint that the estimated first-order error is zero.  For a nuisance
function $\mu$ in a model class $\mc{F}$, the generic constrained problem is
\begin{equation}
  \label{eqn:opt}
  \what \mu^C \in  \argmin_{\mu \in \mc{F}} \Bigg\{ \mbox{PredictionError}_{\what{P}_{\rm train}}(\mu):
  \mbox{First-order error of plug-in from }\mu \mbox{ on } \what{P}_{\rm eval} \text{ is } 0
  \Bigg\}
\end{equation}
where $\widehat P_{\rm train}$ and $\widehat P_{\rm eval}$ denote the empirical
distributions of the training and evaluation splits, respectively.
The constrained learning problem can be implemented with flexible model
classes; we give concrete algorithms for linear models, gradient boosted
trees, and neural networks in \Cref{section:methodology}. While the framework is
not specific to a particular estimand, we use the average treatment effect
(ATE) as a running example because it makes the constraint concrete. For other
continuous linear functionals of an outcome model, the same construction uses
the corresponding Riesz representer in the first-order
constraint~\citep{chernozhukov2022riesznet}.

The \emph{C-Learner} estimator is then given by plugging $\what\mu^C$ into the
target functional and evaluating the resulting plug-in on
$\what P_{\rm eval}$. In our running example of estimating the ATE, this gives
$\what\psi^C=\what P_{\rm eval}[\what\mu^C(X)]$ where $\mu(X)$ is an outcome
model. Under classical regularity and nuisance-rate conditions, this plug-in
enjoys the usual guarantees of debiasing, such as semiparametric efficiency
and double robustness. But compared to
standard debiasing methods, the constrained learning
framework~\eqref{eqn:opt} changes how the first-order correction is computed:
whereas one-step estimation (AIPW) adds an estimated first-order term to a
fitted plug-in and targeting (TMLE) adjusts a fitted nuisance model along a
prespecified direction, the C-Learner instead searches the full chosen model
class $\mc{F}$ for the best-fitting nuisance model that satisfies the
first-order constraint. The correction therefore enters model training, and
the final estimate remains a plug-in of the trained nuisance model.

To isolate when the C-Learner~\eqref{eqn:opt} departs from standard debiasing
methods, we study Gaussian linear regression with known propensity scores
(\Cref{section:high_dim}), where the dimension grows proportionally to the
number of observed outcomes. This is a standard setting in high-dimensional
statistics~\citep{breimanfreedman1983,belkinetal2020,hastieetal2022,dicker2016,
  dobribanwager2018}. Nuisance estimation error
persists in this regime, allowing estimators with the same classical
asymptotic optimality guarantee to have different limiting distributions.

This model lets us characterize high-dimensional regimes that are typically treated
separately~\citep{breimanfreedman1983,belkinetal2020,hastieetal2022,dicker2016,
  dobribanwager2018}, including $d<n$ and $d>n$, as well as ridge
regularization in both cases and at $d=n$. Here the comparison is between
the number of regression coefficients and the number of observed training
outcomes; \Cref{section:high_dim} makes the fold size and observation
probability explicit.
In particular, we develop a unified asymptotic analysis of two-fold
cross-fitted estimators for C-Learner, AIPW, TMLE, and direct plug-ins, even
in regimes where overlap disappears as $n \to \infty$. Under sufficient
overlap, our exact CLT shows that the C-Learner has strictly smaller
asymptotic variance than AIPW and TMLE. When overlap deteriorates so the
average inverse propensity diverges, AIPW and TMLE are no longer root-$n$
consistent, whereas the C-Learner retains the direct plug-in's root-$n$
Gaussian limit.

In our stylized model, the main mathematical structure that drives this
unifying analysis is the reduction of the four estimators to a common form:
the original plug-in plus a multiple of its estimated first-order error.  AIPW
uses the full correction, TMLE matches this asymptotically, whereas the
C-Learner chooses its correction strength through the constrained regression.
This data-dependent scaling allows the C-learner to adapt to the strength of
the regression fit and varying overlap levels, so that it automatically
reverts to the direct plug-in when the propensity-based first-order correction
becomes unstable.

Our analysis makes the comparison between asymptotically optimal
estimators tractable, but does not establish a general ordering of debiasing
methods. In particular, it does not account for propensity estimation or
outcome-model misspecification. To complement our theory, we provide an
extensive empirical evaluation across simulated and semisynthetic tabular
problems and a text setting in which we fine-tune language models
(\Cref{section:experiments}). The experiments cover linear models, gradient
boosted trees, and neural networks across different sample sizes and levels of
overlap. We
evaluate point estimation and confidence interval coverage, and find that the C-Learner improves on
basic one-step and targeting methods in challenging settings while performing similarly otherwise.
We also develop and
evaluate a dual formulation that constrains the propensity model rather than
the outcome model (\Cref{section:balance,section:ks_ipw}).

\paragraph{Related work} Deferring a discussion of one-step estimation and
targeting to \Cref{section:background}, here we discuss related works outside
of these standard debiasing methods.  Machine learning approaches for
nuisance estimation have received a lot of attention
\citep{athey2016recursive,athey2018approximate,athey2019generalized,oprescu2019orthogonal,hahn2019bayesian,wager2018estimation}.
These approaches can naturally be combined with first-order correction methods
like one-step estimation and targeting, with recent
work~\citep{shi2019adapting,chernozhukov2022riesznet} incorporating these
ideas directly into model training through novel objectives and architectures.
The C-Learner can also incorporate innovations in nuisance estimation; for
example, in \Cref{section:ihdp_nn}, we demonstrate how it can effectively use
Riesz representers learned by RieszNet~\citep{chernozhukov2022riesznet}.  A
version of the constrained learner~\eqref{eqn:opt} defined with propensity
scores instead of outcome models (in which the ``plug-in'' is the inverse
propensity weighted estimator) has intimate connections to balancing
estimators, which we discuss in detail in \Cref{section:balance}.

Closest to our main theoretical result are proportional high-dimensional
analyses of plug-in and standard debiased estimators, which all build on past
exact characterizations of high-dimensional least squares and ridge
regression~\citep{breimanfreedman1983,belkinetal2020,hastieetal2022,
  dicker2016,dobribanwager2018}. \citet{yadlowsky2022} derive single-split
variance limits for the plug-in, IPW, AIPW, and TMLE but do not establish
their limiting distributions. \citet{jiangetal2025} allow both the outcome and
logistic propensity models to depend on the high-dimensional covariates, and
account for estimation of both nuisances in a three-fold AIPW CLT. Their
analysis allows a broader propensity model than ours and requires enough
observations in each training fold ($d<n$) for the unregularized outcome
fits to exist.  In the $d>n$ regime, \citet{celentanowainwright2023}
construct a debiasing method that combines penalized estimates of the outcome
and missingness models under a linear-GLM specification with known covariate
covariance, while \citet{chenetal2024} use moment equations to bypass nuisance
estimation under related distributional structure. \citet{mcgrathetal2025}
derive asymptotic variances for bias-corrected ridge estimators of expected
conditional covariance, a related doubly robust functional. Assuming known
propensities independent of the growing covariate block lets us isolate
outcome-model estimation and develop a unified analysis of direct plug-ins,
TMLE, AIPW, and the C-Learner. The analysis covers $d<n$ and $d>n$, with
ridge regularization in both cases and at $d=n$, including when overlap
deteriorates as the sample size grows. Our exact high-dimensional CLTs for cross-fitted estimators may
therefore be of independent interest even outside of constrained learning.

Related work on weak overlap studies the nonstandard limits caused by extreme
propensity weights. \citet{mawang2020} characterize Gaussian and non-Gaussian
limits for inverse probability weighting as a function of the propensity tail
and the amount of trimming, and develop inference that accounts for trimming
bias. \citet{heilerkazak2021} derive stable limits for inverse probability
weighting and doubly robust estimators under irregular identification, with
inference based on an adaptive $m$-out-of-$n$ bootstrap. Our analysis combines
a different description of weak overlap, in which observation probabilities
change with sample size, with proportional growth of the outcome-model
dimension. We impose moment conditions that retain Gaussian limits and compare
how much noise from the first-order correction enters the target estimate. The
root-$n$ limit of the C-Learner relies on the shared coefficients that
determine the outcome mean remaining identified even as the mean inverse
propensity diverges.

%%% Local Variables:
%%% mode: latex
%%% TeX-master: "../main"
%%% End:

\section{Background}
\label{section:background}

Although the framework~\eqref{eqn:opt} is general, we illustrate it in the
context of estimating the average treatment effect (ATE).  Semiparametric
statistics is notoriously mathematically dense, so as an expository device, we
give a heuristic and high-level introduction here and leave a more rigorous
treatment to \Cref{section:theory_text}.

\subsection{Average Treatment Effect and Missing Outcomes}
\label{section:ate}

Consider binary treatments (actions) $A\in\{0,1\}$, covariates
$X\in \mathbb R^d$, and potential outcomes $Y(1), Y(0)\in \R$ under treatment
($A=1$) and control ($A=0$), respectively.  The key difficulty in causal
inference is that we do not observe counterfactuals: we only observe the
outcome $Y := Y(A)$ corresponding to the observed binary treatment.  Let
$Z:=(X,A,Y)$. Based on i.i.d. observations $Z \sim P$, our goal is to estimate
the average treatment effect $\psi(P) = P[Y(1) - Y(0)]:=\int [Y(1) -
Y(0)]dP$. Here $P$ denotes both the joint probability measure and the
expectation operator under this measure.  We require standard identification
conditions that make this goal feasible: (i) $Y=Y(A)$ (SUTVA), (ii)
$(Y(1),Y(0))\perp A \mid X$ (ignorability), and (iii)
$0<P(A=1\mid X)<1$ a.s. (positivity) \citep{ding2023course}.
Classical efficiency guarantees typically strengthen positivity to a uniform
overlap bound, $\eta\leq P(A=1\mid X)\leq1-\eta$ a.s. for some $\eta>0$.
We use this stronger condition for the classical analysis and relax it in
\Cref{section:high_dim}.

To simplify our exposition, we assume $Y(0):=0$ throughout so that we are
estimating the mean of a censored outcome (censored when $A=0$).  This setting
is also known as mean missing outcome \citep{Kennedy22}, which is the focus of
some of our experiments. In this case, we can write the ATE as a functional of
the joint measure $P$, as
\begin{equation}
    \label{eqn:ate}
    \psi(P) := P[Y(1)]=P\left[P[Y\mid A=1, X]\right]
    = P[\mu\opt(X)].
\end{equation}
Here, the outcome model $\mu\opt(X):=P[Y\mid A=1, X]$ and propensity scores
$\pi\opt(X) := P(A=1\mid X)$ are key nuisance parameters; notably, they are
high-dimensional in contrast to the single-dimensional ATE~\eqref{eqn:ate}.

The corresponding nuisance estimators $\what \mu(X)$ and $\what \pi(X)$ can be
implemented as ML models. Throughout, we split the data into a \emph{training
  (auxiliary) fold} and an \emph{evaluation (main) fold}: nuisance models are
fit on the training fold, whose empirical measure we denote
$\what{P}_{\rm train}$, and estimators are formed on the evaluation fold,
whose empirical measure we denote $\what{P}_{\rm eval}$; for a function $f$,
we write $\what{P}_{\rm eval}[f(Z)]$ for the average of $f$ over the
evaluation fold.

\subsection{First-Order Correction and Asymptotic Optimality}
\label{section:optimality_background}

By analyzing the error from blindly trusting an ML-based estimate of the
nuisance parameters to estimate the ATE~\eqref{eqn:ate}, we can develop better
estimation methods.  We begin by noting that the joint distribution over the
observed data can be decomposed as $P= P_{Y, A | X} \times P_X$.  Since the
marginal $P_X$ can be simply estimated with an empirical distribution
(plug-in), we focus on the error induced by approximating $P_{Y, A | X}$ using
an estimate $\what{P}_{Y, A | X}$.
\begin{align*}
  P_X[\what{\mu}(X)] - P_X [\mu\opt(X)]
  & =     P_X [\what{P}[Y \mid A = 1, X]]
    - P_X [P[Y \mid A = 1, X]] \\
  & = \psi(P_X \times \what{P}_{Y, A | X}) -
    \psi(P_X \times P_{Y, A | X}).
\end{align*}

For the statistical functional $Q \mapsto \psi(Q)$, let $\varphi_X$ be its
canonical gradient (a.k.a. efficient influence function) with respect to
$Q_{Y, A|X}$; without loss of generality, we require $Q[\varphi_X(Z; Q) | X] =0$
for all $Q$.
Then, a (distributional) first-order Taylor expansion gives
\begin{align}
  \psi(P_X \times \what{P}_{Y, A | X})
  -  \psi(P_X \times P_{Y, A | X})
  & = \iint \varphi_X(Z; \what{P}) d (\what{P}_{Y, A|X} - P_{Y, A|X}) dP_X + R_2(\what{P}, P) \nonumber \\
  & = - \iint \varphi_X(Z; \what{P}) dP_{Y, A|X} dP_X + R_2(\what{P}, P) \nonumber \\
  & = - P [ \varphi_X(Z; \what{P}) ]
    + R_2(\what{P}, P)
    \label{eqn:taylor}
\end{align}
where $R_2$ is a second-order remainder term.
Conclude that the first-order error term of the plug-in approach $P_X[\what{\mu}(X)]$ is given by $-P [ \varphi_X(Z; \what{P}) ]$.

A common approach in semiparametric statistics to better estimate $\psi(P)$ is
to explicitly correct for this first-order error term.  The first-order
correction leads to asymptotic optimality properties like semiparametric
efficiency---providing the shortest possible confidence interval---and double
robustness---achieving estimator consistency even if only one of
$\what\mu(X), \what\pi(X)$ is consistent.
For the ATE~\eqref{eqn:ate}, it is well-known that
\begin{equation}
\label{eqn:ate-pathwise}
\varphi_X(Z; P)
= \frac{A}{P[A =1 | X]} (Y - P[Y | A = 1, X])
= \frac{A}{\pi\opt(X)} (Y - \mu\opt(X))
\end{equation}
The full efficient influence function also includes the contribution
from estimating the covariate distribution:
\(\varphi(Z;P)=\varphi_X(Z;P)+\mu\opt(X)-\psi(P)\).
Practically speaking, the methods below that empirically correct for this
projected version produce the same guarantees as empirically correcting for
the full canonical gradient \cite{VanDerLaanRo11,van2006targeted}.

\paragraph{One-Step Estimation (Augmented Inverse Propensity Weighting, AIPW)~\cite{Bang2005DoublyRE}}
One of the most prominent first-order correction approaches modifies the
plug-in estimator by moving the first-order error term to the left-hand side
in the Taylor expansion~\eqref{eqn:taylor}. The resulting estimator achieves
second-order error rates:
\begin{align*}
  P\left[ \what{\mu}(X)
  + \varphi_X(Z; \what{P})\right]
  - P[\mu\opt(X)] = R_2(\what{P}, P).
\end{align*}
Using the evaluation-fold empirical distribution $\what{P}_{\rm eval}$ to approximate $P$ in
the one-step debiased estimator, we arrive at the augmented inverse propensity weighted estimator
\begin{align}
  \what\psi^{\textrm{\rm AIPW}}
  & := \what{P}_{\rm eval}\left[\what{\mu}(X)
    + \varphi_X(Z; \what{P}) \right] %\nonumber \\
  % &
      = \what{P}_{\rm eval}\left[\what{\mu}(X)\right]
      + \what{P}_{\rm eval}\left[\frac{A}{\what{\pi}(X)}
      (Y - \what{\mu}(X))\right] .
      \label{eqn:debiased}
\end{align}

\paragraph{Targeting (Targeted Maximum Likelihood
  Estimation)~\cite{VanDerLaanRo11,van2006targeted}}
Targeting takes an alternative and more general approach to first-order
correction. It commits to the use of the plug-in estimator, and constructs a
tailored adjustment (``fluctuation'') to the existing nuisance parameter
estimate to set the first-order error to zero, where the magnitude of the
adjustment is solved by maximizing likelihood (hence the name). The
fluctuation takes place in the direction of a task-specific random variable
(``clever covariate''), which takes different forms depending on the
estimand. For illustrative purposes, in this discussion, we focus on a
formulation of targeting for unbounded outcomes under the squared loss
function.  In our experiments, we also include a commonly used variant of
targeting that increases the stability of the estimator by assuming bounded
outcomes, and modeling the outcomes with a logistic link.

We begin by using standard notation for TMLE, which has notation for both
$A=0$ and $A=1$. We apply it to our setting with $Y(0):=0$ (so that
$\what\mu(X)\defeq \what\mu(1,X)$, as in \Cref{section:ate}) and consider general
outcomes $Y$ (including unbounded continuous $Y$). The form of the canonical
gradient~\eqref{eqn:ate-pathwise} motivates the use of
$H(A,X):=\frac{A}{ \what{\pi}(X)}$ as the clever
covariate:
targeting uses an adjusted nuisance estimate
$\what{\mu}(A,X)+\epsilon^\star H(A,X)$ %\frac{A}{ \what{\pi}(X)}$
in place of $\what{\mu}(A,X)$ in the plug-in for
$\psi(P)=P[\mu\opt(1,X)]$, where
$\epsilon^\star$ is chosen to solve the targeted maximum likelihood problem
\begin{equation}
    \label{eqn:targeting}
    \epsilon^\star:=\argmin_{\epsilon \in \R}
    ~~\what{P}_{\rm eval}\left[  A \left(
    Y - \what{\mu}(A,X) - \epsilon \frac{A}{\what\pi(X)}\right)^2\right].
\end{equation}
From the KKT conditions, the solution removes the finite-sample estimate of
the first-order error term~\eqref{eqn:taylor} of the plug-in estimator using
$\what{\mu}(A,X) + \epsilon^\star \frac{A}{
  \what{\pi}(X)}$. Solving for $\epsilon\opt$, we obtain
\begin{equation}
    \label{eqn:tmle-eps}
    \epsilon^\star =
    \left(
    \what{P}_{\rm eval}\left[\frac{A}{\what{\pi}(X)^2}\right]
    \right)^{-1}
    \what{P}_{\rm eval}\left[
    \frac{A}{\what{\pi}(X)} (Y - \what{\mu}(X))\right].
\end{equation}
Thus, we arrive at an explicit formula for the targeted maximum likelihood estimator
\begin{align}
    \what{\psi}^{\rm TMLE}
    & \defeq \what{P}_{\rm eval}\Big[\what\mu(1,X)+\epsilon^\star H(1,X)\Big]
    =
    \what{P}_{\rm eval}\left[\what{\mu}(X)
    + \epsilon^\star \frac{1}{\what{\pi}(X)}\right] \nonumber \\
    & = \what{P}_{\rm eval}\left[\what{\mu}(X)\right]
    + \frac{\what{P}_{\rm eval}\left[\frac{1}{\what{\pi}(X)}\right]}{\what{P}_{\rm eval}\left[\frac{A}{\what{\pi}(X)^2}\right]}
    \cdot \what{P}_{\rm eval}\left[
    \frac{A}{\what{\pi}(X)} (Y - \what{\mu}(X))\right].
        \label{eqn:tmle}
\end{align}

Instead of modifying the outcome model in the direction of the clever
covariate in a post-processing step~\eqref{eqn:targeting}, we can instead
regularize the usual training objective for the outcome model using a similar
linear adjustment term throughout training, an approach referred to as
\emph{targeted
  regularization}~\cite{chernozhukov2022riesznet,shi2019adapting}.

%%% Local Variables:
%%% mode: latex
%%% TeX-master: "../main"
%%% End:

\section{Constrained Learning Framework}
\label{section:methods}

One-step estimation and targeting first fit a nuisance model and then adjust
the estimator or fitted model. We instead propose a \emph{constrained learning
  framework} that trains the nuisance model over the full chosen model class
subject to the constraint that the first-order error term~\eqref{eqn:taylor}
is zero. Our method, which we call the C-Learner, adapts machine learning
models to the downstream semiparametric estimation task during training. We
use the ATE in a mean missing outcome formulation to state the method
concretely. The same construction applies to other estimands that are
continuous linear functionals of an outcome model: the first-order constraint
uses an estimate of the estimand's Riesz representer in place of the inverse
propensity weight \citep{chernozhukov2022riesznet}.  For one-step estimation
and targeting, in the sections before, we would use a fitted outcome model
$\what\mu$ that minimizes the squared prediction loss (or equivalently
maximizes likelihood, assuming outcomes are Gaussian), with
\begin{align}
\label{eq:regular_outcome}
  \what{\mu} \in \argmin_{\mu \in \mc{F}} \what{P}_{\rm train}[ A (Y - \mu(X))^2],
\end{align}
fitted on the training fold $\what{P}_{\rm train}$. Instead, the C-Learner solves the constrained optimization problem
\begin{align}
\label{eq:c-learner}
  & \what{\mu}^C
    \in \argmin_{\mu\in\mathcal F} \left\{
    \what{P}_{\rm train}[ A (Y - \mu(X))^2] :
    \what{P}_{\rm eval}\left[
    \frac{A}{\what{\pi}(X)}
    (Y-\mu(X))\right]=0
    \right\},
  \\
  & \what{\psi}^{\textrm {C-Learner}}
    :=
    \what{P}_{\rm eval}\left[\what{\mu}^C(X)\right]
\label{eq:c-learner-est}
\end{align}
so that the constraint in \Cref{eq:c-learner} is applied to the same
data used in the plug-in~\eqref{eq:c-learner-est}.  The model class
$\mathcal F$ can be chosen for the prediction problem and the data modality.

\paragraph{Connections to existing first-order corrections}
To build intuition for the framework, \Cref{fig:constrained_schematic}
compares the model classes used by different first-order correction methods.
All three methods use the same held-out first-order correction term, but they
differ in how deeply it changes the outcome model. One-step estimation leaves the fitted
model unchanged and corrects only the final estimate. Targeting updates the
model along a prespecified one-dimensional direction. These are shallow
corrections because most of the fitted outcome model remains fixed. The
C-Learner instead makes a more fundamental correction: it searches the full
chosen model class for the best-fitting model that satisfies the first-order
constraint.
\begin{figure}[t]
    \centering
    \includegraphics[width=0.85\linewidth]{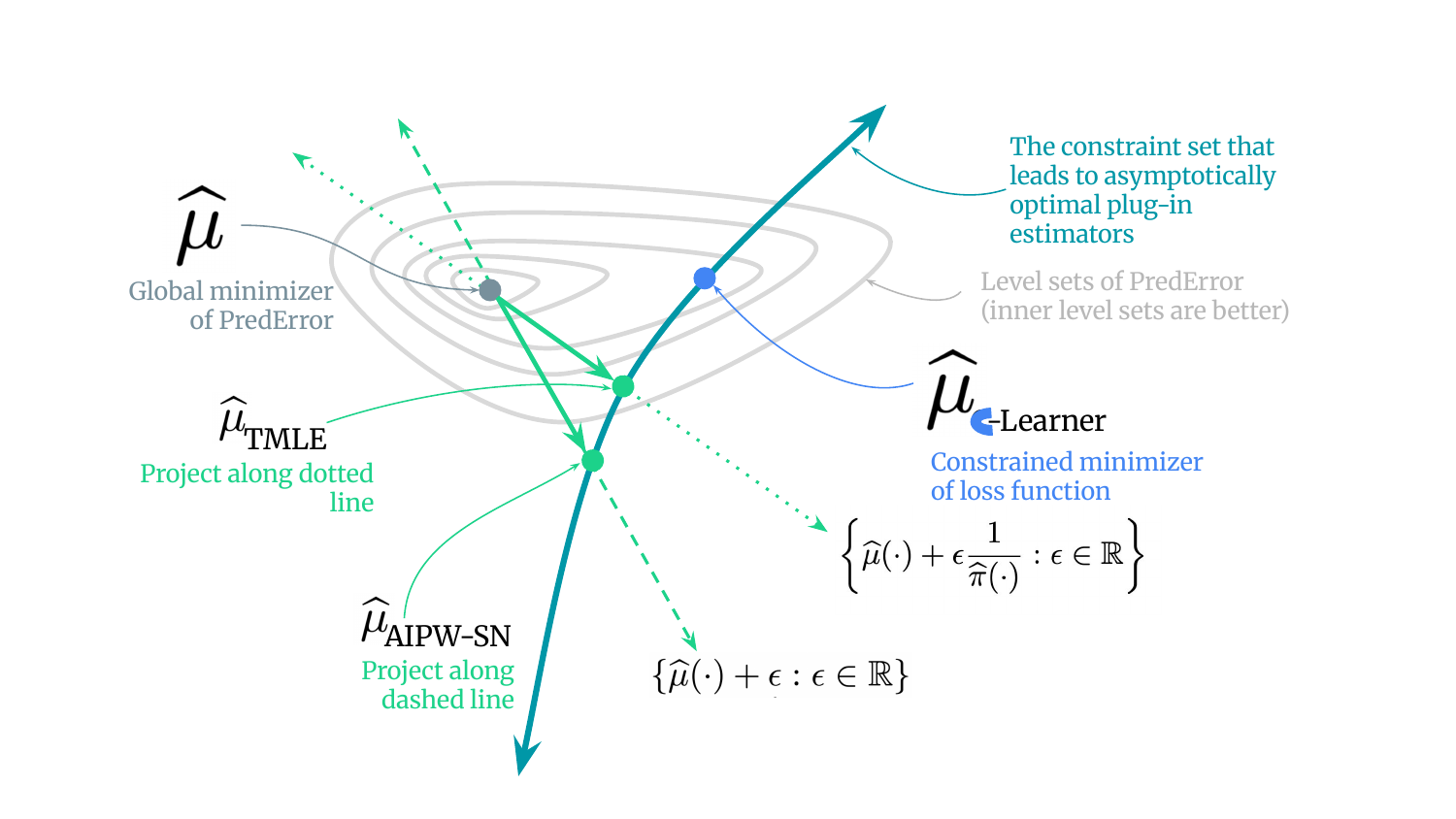}
    \caption{Constrained learning compared with self-normalized AIPW and
    targeting. The unconstrained outcome model $\what\mu$ minimizes prediction
    error, whose level sets are shown in gray. Self-normalized AIPW and targeting
    restrict the update to particular directions and therefore reach different
    points on the set of outcome models that satisfy the first-order constraint.
    The C-Learner minimizes prediction error over the full chosen model class
    subject to the same constraint.}
    \label{fig:constrained_schematic}
\end{figure}

\begin{samepage}
  Restricted choices of the model class recover familiar first-order
  correction methods. For a trained outcome model $\what\mu$, the
  constant-shift class
  $\mathcal F := \{\what{\mu}(\cdot)+\epsilon: \epsilon \in \R\}$ yields the
  \emph{self-normalized AIPW}
  \begin{align}
    \what{\psi}^{\textrm {AIPW-SN}}
    = \what{P}_{\rm eval}\left[\what{\mu}(X)\right] +
    \left(\what{P}_{\rm eval}\left[
    \frac{A}{\what{\pi}(X)}\right] \right)^{-1}
    \what{P}_{\rm eval}\left[\frac{A}{\what{\pi}(X)}(Y-\what\mu(X))\right], \label{eq:normalized_aipw}
  \end{align}
  which differs from AIPW~\eqref{eqn:debiased} only by its normalization. The
  class $\mathcal F := \{\what{\mu}(\cdot)+\epsilon/\what{\pi}(\cdot):
  \epsilon \in \R\}$ yields the squared-loss targeting estimator
  \eqref{eqn:tmle}. We reserve ``C-Learner'' for the estimator that optimizes
  over the full chosen model class. Exchanging the roles of the two nuisance
  models gives a dual formulation related to covariate balancing propensity
  scores~\citep{imai2014covariate,tan2010bounded,zhao2019covariate}; see
  \Cref{section:balance}.
\end{samepage}

The constraint~\eqref{eq:c-learner} ensures that the corresponding plug-in
estimator \eqref{eq:c-learner-est} has a finite-sample estimate of the
first-order error term~\eqref{eqn:taylor} of zero.  The training objective for $\what\mu^C$ thus optimizes the outcome model fit
on the training data subject to a zero empirical first-order correction on
the evaluation data. In practice, there are several computational approaches to
(approximately) solve the stochastic optimization problem~\eqref{eq:c-learner}
depending on the function class $\mathcal F$.  In \Cref{section:methodology}, we
describe how to instantiate C-Learner using linear models, gradient boosted
regression trees, and neural networks. We empirically demonstrate these in
\Cref{section:experiments}.

\paragraph{Numerical stability without additional heuristics}
The same first-order constraint can produce different changes in the
reported mean. One-step estimation adds an inverse-weighted correction to
the plug-in~\eqref{eqn:debiased}, and targeting changes the fitted outcome
model along the inverse propensity~\eqref{eqn:targeting}. The C-Learner
instead lets prediction loss over the full model class determine the
adjustment. This choice preserves classical asymptotic optimality under the
usual conditions while allowing the reported mean to respond differently
to large propensity weights.

In settings with regions of low overlap between treatment and control, the
estimated probability of treatment $\what\pi(X)$ can be very close to 0 for some values
of $X$, so that $1/\what\pi(X)$ can be extremely large, causing numerical
instability in one-step estimation and targeting.  While there are variations
for both AIPW and TMLE that are more stable, e.g., by self-normalizing
propensity weights, truncating $\what\pi$ to avoid extreme values in AIPW, or
assuming $Y$ is bounded and modeling $Y$ as a (scaled) logistic function of
$(A,X)$ in TMLE (see \Cref{section:ks}), the C-Learner allows the fitted
outcome model to determine how satisfying the first-order constraint
changes the estimated mean. Large weights still enter the constraint, but
need not produce a comparably large change in the final plug-in.

Simple outcome-model plug-ins can remain stable in these settings
(e.g.,~\cite{KangSc07}), and our experiments show that the C-Learner retains
this advantage (\Cref{section:experiments}). Our theory explains how the
constraint can preserve that stability: in the high-dimensional linear
model, the C-Learner scales the usual first-order correction by a
problem-dependent factor that accounts for overlap and the regression fit.
As overlap deteriorates, the
C-Learner automatically discounts the first-order correction. Under the
identification and moment conditions in \Cref{section:high_dim}, this
correction changes the direct plug-in by less than its root-$n$ sampling
error, while the noise in the AIPW and TMLE corrections grows relative to
that scale. The C-Learner therefore has the same limiting distribution as
the direct method. A complementary finite-sample comparison in
\Cref{section:theory_simple} shows that, with nonintegrable inverse
propensities and an outcome-model class controlled by a square-integrable
envelope, a feasible C-Learner extended by zero when no solution is selected
has finite mean squared error, while AIPW and TMLE have infinite mean squared
error.

\paragraph{Classical asymptotic guarantees}
Like other first-order correction methods, the C-Learner can be
cross-fit~\cite{ChernozhukovChDeDuHaNeRo18}: split the data into $K$ folds,
use each fold in turn as the evaluation fold with the remaining folds as the
training fold, and average the resulting estimates, so that all of the data is
used. This is how we implement the C-Learner in our experiments, where we
additionally reserve a validation split $\what{P}_{\rm val}$ for model
selection (with $\what{P}_{\rm eval}=\what{P}_{\rm val}$ except where
otherwise specified).
\begin{proposition}[Informal]
  \label{prop:classical-guarantees}
  Under the overlap, moment, nuisance-rate, and evaluation-fold stability
  conditions in Appendix~\ref{section:theory_appendix}, the cross-fitted
  C-Learner is semiparametrically efficient. Under the corresponding
  consistency conditions, it is doubly robust.
\end{proposition}
\noindent Formal statements and proofs are deferred to
\Cref{section:theory_appendix}.
The double-robustness proof is in Appendix~\ref{section:double-robustness}.
Since this standard asymptotic regime
establishes the same efficient limit for the C-Learner, AIPW, and TMLE and
cannot distinguish the three methods, the next section studies a setting where
nuisance errors remain first order and shows that the C-Learner can improve on
these asymptotically optimal methods while retaining the stability of the
direct plug-in.

%%% Local Variables:
%%% mode: latex
%%% TeX-master: "../main"
%%% End:
       % Section 3: constrained learning framework
\section{High-dimensional Asymptotics}
\label{section:high_dim}

When nuisance estimation errors are second order, debiased estimators share
the same first-order limit, so classical asymptotics cannot distinguish their
performance. We therefore study a stylized Gaussian linear regression
problem where the number of coefficients grows proportionally to the number
of observed outcomes and outcome-model estimation errors remain first order.
To isolate the effect of estimating the outcome model using the C-Learner,
we assume a correctly specified linear outcome model and known propensities
that depend on fixed-dimensional covariates independent of the growing
Gaussian covariates.

We compare the C-Learner, AIPW, TMLE, and direct plug-ins for \(d<n\) and
\(d>n\), and with ridge regularization in both cases, including \(d=n\).  When
\(d<n\), the observed training outcomes identify a unique least-squares
coefficient vector, but prediction error persists as dimension grows with
sample size. When \(d>n\), many coefficient vectors fit every observed
training outcome exactly, so the training data cannot identify the full
outcome model. (We use these dimension vs. sample comparisons colloquially:
the relevant sample size is the effective number of observed outcomes per
training fold, defined below.)

Our main result is twofold. First, under sufficient overlap, the C-Learner
achieves strictly smaller asymptotic variance than AIPW and TMLE. Second, when
overlap deteriorates, AIPW and TMLE fail to be root-\(n\) consistent, whereas
the C-Learner automatically discounts the unstable first-order correction and
retains the direct plug-in's root-\(n\) limit.  Concretely, we study the
standard Gaussian linear regression setting from the high-dimensional
statistics literature~\citep{breimanfreedman1983,belkinetal2020,
  hastieetal2022,dicker2016,dobribanwager2018}. As in \Cref{section:ate}, let
\(Y(0)=0\) and \(Y=Y(A)\), so the observed data are \(Z=(X,A,Y)\), together
with the known observation probability \(\pi_n\opt\). For every \(n\), we
observe~\(2n\) independent units with \(X=(X_0,X_n)\) and
\begin{equation}
\begin{gathered}
  X_n\sim\Normal(0,\Id_{d_n})\ \mbox{independent of }(X_0,\pi_n\opt),
  \qquad A\mid X,\pi_n\opt\sim\Bern(\pi_n\opt),\\
  Y(1)={X_0}^\top\alpha\opt+{X_n}^\top\beta_n\opt+\eps,
  \qquad \eps\sim\Normal(0,\sigma^2).
\end{gathered}
  \label{eq:model}
\end{equation}
The subscripts label covariate blocks, with \(X_{i,0}\) and \(X_{i,n}\)
denoting the blocks of observation \(i\). The block \(X_0\in\R^q\) has fixed
dimension, a distribution that does not change with \(n\), and first
coordinate~\(1\). The dimension \(d_n\) of \(X_n\) grows with \(n\).  The
known true propensity \(\pi_n\opt\in(0,1]\) may depend on additional
fixed-dimensional covariates beyond those in the outcome model. These
underlying covariates have a distribution that does not change with \(n\),
while their observation probabilities may change. The noise \(\eps\) is
independent of \((X,\pi_n\opt,A)\), with \(\sigma^2>0\). The coefficient~\(\alpha\opt\) is fixed, whereas \(\beta_n\opt\in\R^{d_n}\) satisfies
\(\norm{\beta_n\opt}^2\to\sigma_\star^2<\infty\).

Our estimand is \(\psi=P[Y(1)]={\alpha\opt}^\top P[X_0]\).  Assume that
\(P[\pi_n\opt]\) converges to a number in \((0,1)\), and let
\begin{equation}
  n_{\mathrm{eff}}:=nP[\pi_n\opt],\qquad
  \gamma:=\lim\frac{d_n}{n_{\mathrm{eff}}}\in(0,\infty).
  \label{eq:gamma}
\end{equation}
Here \(n_{\mathrm{eff}}\) is the expected number of observed outcomes per
training fold, and \(\gamma\) is the limiting ratio of model size to this
effective sample size. The comparisons \(d<n\), \(d>n\), and \(d=n\) refer to
\(\gamma<1\), \(\gamma>1\), and \(\gamma=1\), respectively.  With this
scaling, model~\eqref{eq:model} reduces to the standard Gaussian regression
settings studied by
\citet{breimanfreedman1983,belkinetal2020,hastieetal2022,dicker2016,
  dobribanwager2018}, with a random number of observed training outcomes
(\Cref{section:coefficient-limits}). We build on their regression limits to
derive (novel) exact CLTs for the four population-mean
estimators. \citet{jiangetal2025} allow estimated propensities that depend on
the high-dimensional covariates in their AIPW analysis, but require more
observed training outcomes than regression coefficients. Our assumption of
known propensities independent of the growing Gaussian block lets us compare
all four estimators in a unified manner, across the regression regimes above
and including under deteriorating overlap.

\paragraph{Estimators}
We split the \(2n\) units into a \emph{training (auxiliary) fold} \(\Ptrain\)
and an \emph{evaluation (main) fold} \(\Peval\) of size \(n\) each, using the
same symbols for their empirical measures (\Cref{section:methods}). On the
training fold, we fit
\begin{equation}
  (\what\alpha_1,\what\beta_1)
  =\argmin_{\alpha\in\R^q,\,\beta\in\R^{d_n}}
  \left(\Ptrain\left[A(Y-{X_0}^\top\alpha-{X_n}^\top\beta)^2\right]
       +\lambda\norm\beta^2\right).
  \label{eq:ridge-fit}
\end{equation}
The resulting outcome model is
\(\what\mu_1(x)={x_0}^\top\what\alpha_1+{x_n}^\top\what\beta_1\).
The penalty does not depend on \(\alpha\). When \(d_n<n_{\mathrm{eff}}\)
(\(\lambda=0\), \(\gamma<1\)), there are more observed training outcomes
than coefficients, and ordinary least squares has a unique solution with
probability tending to one. Its prediction error nevertheless persists
because the number of coefficients grows with the sample size. When
\(d_n>n_{\mathrm{eff}}\) (\(\lambda=0\), \(\gamma>1\)), there are more coefficients
than observed training outcomes, so many coefficient vectors achieve zero
training loss. We select the solution with the smallest \(\norm\beta\).
Ridge regression (\(\lambda>0\)) adds the penalty \(\lambda\norm\beta^2\)
and gives a unique coefficient vector whenever the observed training matrix for \(X_0\) has rank \(q\), an event whose probability tends to one under our assumptions. This applies
both when \(d_n<n_{\mathrm{eff}}\) and when \(d_n>n_{\mathrm{eff}}\), as well
as at the boundary \(d_n=n_{\mathrm{eff}}\). We exclude the unregularized
boundary \(\gamma=1\), where the limiting prediction error can diverge.

The direct method reports \(\what\psi_{1,\DM}=\Peval\what\mu_1\).
AIPW adds the first-order correction,
\(\what\psi_{1,\AIPW}=\Peval\what\mu_1+
\Peval[(A/\pi_n\opt)(Y-\what\mu_1(X))]\).
Squared-loss TMLE replaces \(\what\mu_1\) by
\(\what\mu_1+\what\zeta/\pi_n\opt\), choosing \(\what\zeta\) so that
\(\Peval[(A/\pi_n\opt)(Y-\what\mu_1(X)-\what\zeta/\pi_n\opt)]=0\),
and reports \(\what\psi_{1,\TMLE}=\Peval[\what\mu_1+\what\zeta/\pi_n\opt]\).
The C-Learner~\eqref{eq:c-learner} trains the whole outcome model subject
to a zero first-order correction on the held-out fold,
\begin{equation}
\begin{aligned}
  (\what\alpha_1^C,\what\beta_1^C)
  \in\argmin_{\alpha,\beta}\quad
     &\Ptrain\left[A(Y-{X_0}^\top\alpha-{X_n}^\top\beta)^2\right]
       +\lambda\norm\beta^2\\
  \textnormal{subject to}\quad
     &\Peval\left[\frac A{\pi_n\opt}
          (Y-{X_0}^\top\alpha-{X_n}^\top\beta)\right]=0,
\end{aligned}
  \label{eq:c-program}
\end{equation}
with the same minimum-norm selection when \(\lambda=0\) and \(\gamma>1\).
It reports \(\what\psi_{1,\CL}=\Peval\what\mu_1^C\), the plug-in mean of
the constrained model.

We use two-fold cross-fitting. The subscript \(1\) denotes the split above;
swapping the training and evaluation folds produces a second copy of each
quantity, marked with the subscript \(2\). The cross-fitted estimators are
\begin{equation*}
  \what\psi_\bullet
    =\tfrac12(\what\psi_{1,\bullet}+\what\psi_{2,\bullet}),
  \qquad \bullet\in\{\DM,\AIPW,\TMLE,\CL\}.
\end{equation*}

\Cref{section:high-dimensional-clt} states the CLTs.
\Cref{section:reduction} explains the comparison by expressing all four
estimators as a direct plug-in plus a multiple of the first-order correction.
The multiplier depends on \(\gamma\), \(\lambda\), and
\(P[(\pi_n\opt)^{-1}]\).

\subsection{Exact CLT When Low-dimensional Propensities Are Known}
\label{section:high-dimensional-clt}

In the stylized analytic setup~\eqref{eq:model}, overlap enters through the
mean inverse propensity,
\(P[(\pi_n\opt)^{-1}]\allowbreak=P[(A/\pi_n\opt)^2]\), which is also the
second moment of the weight in the first-order correction. When this mean
converges to a finite limit, the correction changes the asymptotic variance;
when it diverges, the correction changes the \emph{rate of convergence} (we
call these regimes \emph{regular vs. deteriorating overlap}).
\begin{assumption}[Overlap and moments]
  \label{ass:overlap-geometry}
  The sequence \(P[(\pi_n\opt)^{-1}]\) converges to a finite limit or
  diverges to infinity, \(P[\norm{X_0}^4]<\infty\), and
  \begin{equation*}
  \begin{gathered}
    \frac{P[(\pi_n\opt)^{-1}]}n\to0,\qquad
    P\left[\frac{\norm{X_0}^2}{\pi_n\opt}\right]
       =O\left(P[(\pi_n\opt)^{-1}]\right),\qquad
    P[(\pi_n\opt)^{-3}]
       =o\left(n\left(P[(\pi_n\opt)^{-1}]\right)^2\right).
  \end{gathered}
  \end{equation*}
  Moreover, we assume
  \begin{equation*}
    \Sigma:=\lim P[\pi_n\opt X_0X_0^\top]\succ0.
  \end{equation*}
\end{assumption}
\noindent These moment conditions ensure that weighted averages concentrate
even when their variance diverges, which holds for example for a fixed
propensity bounded away from zero.  The contribution of \(P[X_0]^\top(\what\alpha_1-\alpha\opt)\)
to the asymptotic variance depends on
\begin{equation*}
  \sigma_X^2:=P[X_0]^\top\Sigma^{-1}P[X_0]
      \geq\frac1{\lim P[\pi_n\opt]}>1.
\end{equation*}
This quantity remains finite as overlap deteriorates because \(\Sigma\) stays
positive definite.
Observation probabilities may deteriorate in parts of the population
while the observed data continue to identify the shared outcome coefficients.

\begin{table}[!ht]
  \centering
  \small \setlength{\tabcolsep}{6pt} \renewcommand{\arraystretch}{1.15}
  \begin{tabular}{
    >{\centering\arraybackslash}m{\dimexpr.22\textwidth-2\tabcolsep-\arrayrulewidth\relax}|
    >{\centering\arraybackslash}m{\dimexpr.37\textwidth-2\tabcolsep-\arrayrulewidth\relax}|
    >{\centering\arraybackslash}m{\dimexpr.26\textwidth-2\tabcolsep-\arrayrulewidth\relax}|
    >{\centering\arraybackslash}m{\dimexpr.15\textwidth-2\tabcolsep\relax}}

    \toprule
    \rowcolor{black!8}
    \textbf{Regime} & \(\boldsymbol{\sigma_{\mathrm{fit}}^2(\gamma,\lambda)}\)
    & \(\boldsymbol{\sigma_{\mathrm{res}}^2(\gamma,\lambda)}\)
    & \(\boldsymbol{\mathsf{adj}(\gamma,\lambda)}\)\\
    \midrule
    \(\lambda=0,\ \gamma<1\)
    & \(\dfrac{\sigma^2\gamma}{1-\gamma}\)
    & \(\dfrac{\sigma^2}{1-\gamma}\)
    & \(1/\gamma\)\\[2pt]
    \midrule
    \(\lambda=0,\ \gamma>1\)
    & \(\dfrac{\sigma_\star^2(\gamma-1)}{\gamma^2}
                +\dfrac{\sigma^2}{\gamma-1}\)
    & \(\gamma\sigma_{\mathrm{fit}}^2(\gamma,0)\)
    & \(1\)\\[2pt]
    \midrule
    \(\lambda>0,\ \gamma>0\)
    & \(\begin{gathered}
        \sigma_\star^2t^2(-m_\gamma'(t)-m_\gamma(t)^2)\\
        {}+\sigma^2\gamma(m_\gamma(t)+t m_\gamma'(t))
      \end{gathered}\)
    & \(\begin{gathered}
        \sigma^2+\sigma_{\mathrm{fit}}^2(\gamma,\lambda)\\
        {}+\sigma_\star^2t^2m_\gamma(t)^2
      \end{gathered}\)
    & \(1+\dfrac1{\gamma m_\gamma(t)}\)\\[2pt]
    \bottomrule
  \end{tabular}
  \captionsetup{position=bottom,skip=6pt,font=small}
  \caption{Variance constants and C-Learner adjustment in the three regression
    regimes. For ridge, \(t=\lim\lambda/P[\pi_n\opt]>0\) and
    \(m_\gamma(t)=\int(x+t)^{-1}\,dF_\gamma(x)
      =\frac{\sqrt{(t+1-\gamma)^2+4\gamma t}-(t+1-\gamma)}{2\gamma t}\)
    is the Stieltjes transform of the Marchenko--Pastur law \(F_\gamma\),
    with \(-m_\gamma'(t)=\int(x+t)^{-2}\,dF_\gamma(x)\).}
  \label{tab:regime-constants}
\end{table}

The outcome regression enters the CLTs through two quantities. For \(k=1,2\),
\begin{equation*}
  \tfrac12\norm{\what\beta_1-\what\beta_2}^2
  \cp\sigma_{\mathrm{fit}}^2(\gamma,\lambda),\qquad
  \sigma^2+\norm{\what\beta_k-\beta_n\opt}^2
  \cp\sigma_{\mathrm{res}}^2(\gamma,\lambda).
\end{equation*}
For an independent observation, the conditional variance of
\(\eps-X_n^\top(\what\beta_k-\beta_n\opt)\) converges to
\(\sigma_{\mathrm{res}}^2(\gamma,\lambda)\). The adjustment
\(\mathsf{adj}(\gamma,\lambda)\) describes how the training design and
regularization affect the C-Learner's correction of the population mean.
\Cref{tab:regime-constants} gives all three explicitly; in every regime,
\(0<\sigma_{\mathrm{fit}}^2(\gamma,\lambda)
<\sigma_{\mathrm{res}}^2(\gamma,\lambda)\) and
\(\mathsf{adj}(\gamma,\lambda)\geq1\).

\begin{theorem}[Two-fold cross-fitted CLTs]
  \label{thm:comparison}
  \label{thm:main}
  Suppose~\eqref{eq:model}--\eqref{eq:gamma} and
  \Cref{ass:overlap-geometry} hold in any of the three regimes in
  \Cref{tab:regime-constants}. Define
  \begin{equation*}
    V_\DM={\alpha\opt}^\top\cov(X_0)\alpha\opt
    +\sigma^2+\sigma_\star^2
    +(\sigma_X^2-1)\sigma_{\mathrm{res}}^2(\gamma,\lambda)
    +2\sigma_{\mathrm{fit}}^2(\gamma,\lambda).
  \end{equation*}
\begin{enumerate}[label=\textup{(\alph*)},leftmargin=2.2em,itemsep=4pt]
\item\label{item:sufficient-overlap}
  \textup{(Sufficient overlap, \(P[(\pi_n\opt)^{-1}]\to\kappa<\infty\).)}
  Define the variance function for \(s\in\R\),
  \begin{equation}
    V(s)=V_\DM-2s(1+(1-s)\sigma_X^2)\sigma_{\mathrm{fit}}^2(\gamma,\lambda)
    +s^2(\kappa-\sigma_X^2)\sigma_{\mathrm{res}}^2(\gamma,\lambda).
    \label{eq:V}
  \end{equation}
  All four estimators satisfy
  \(\sqrt{2n}(\what\psi-\psi)\cd\Normal(0,V)\), where
  \begin{equation*}
  \begin{aligned}
    V_\AIPW=V_\TMLE=V(1),
    ~~~\mbox{and}~~~
    V_\CL=V\left(\frac{1+\sigma_X^2\mathsf{adj}(\gamma,\lambda)}
      {\kappa+\sigma_X^2\mathsf{adj}(\gamma,\lambda)}\right)
    <V_\AIPW.
  \end{aligned}
  \end{equation*}
\item\label{item:deteriorating-overlap}
  \textup{(Deteriorating overlap, \(P[(\pi_n\opt)^{-1}]\to\infty\).)}
  AIPW and TMLE are \textbf{NOT} root-\(n\) consistent
  \begin{equation*}
  \begin{gathered}
    \what\psi_\TMLE-\what\psi_\AIPW
       =\op\left(n^{-1/2}  P[(\pi_n\opt)^{-1}]^{1/2}\right),
    \quad\sqrt{\frac{2n}{P[(\pi_n\opt)^{-1}]}}
       (\what\psi_\AIPW-\psi)
       \cd\Normal(0,\sigma_{\mathrm{res}}^2(\gamma,\lambda)),
  \end{gathered}
  \end{equation*}
  whereas the C-Learner achieves the same first-order error as the direct method
  \begin{equation*}
    \what\psi_\CL-\what\psi_\DM=\op(n^{-1/2}),
    \qquad
    \sqrt{2n}(\what\psi_\DM-\psi)\cd\Normal(0,V_\DM).
  \end{equation*}
\end{enumerate}
\end{theorem}

When there is sufficient overlap, \Cref{thm:comparison} gives a strict
variance improvement over AIPW and TMLE for \(d_n<n_{\mathrm{eff}}\) and
\(d_n>n_{\mathrm{eff}}\), and with ridge regularization in both cases and at
\(\gamma=1\). When \(d_n>n_{\mathrm{eff}}\) and \(\lambda=0\), many
coefficient vectors fit every observed training outcome exactly, yet give
different predictions on the evaluation fold. The C-Learner's constraint
acts as a selection mechanism among these interpolators: it uses the
remaining degrees of freedom to set the held-out first-order correction
to zero while preserving every observed training prediction. The
minimum-norm rule then selects among the interpolators that satisfy the
constraint. Fitting every observed training outcome therefore leaves a
choice that affects estimation of the population mean, and the theorem
quantifies the variance improvement from this constrained selection.

Under deteriorating overlap, \Cref{thm:comparison} provides a comparison in
convergence rates.  The C-Learner retains the direct method's root-\(n\)
Gaussian limit, whereas AIPW and TMLE converge at rate
\(\sqrt{n/P[(\pi_n\opt)^{-1}]}\).  For a concrete example, take \(X_0=1\) and
let half the population be observed with probability \(n^{-1/2}\) and the
other half with probability \(1/2\). Then
\(P[(\pi_n\opt)^{-1}]=\tfrac12\sqrt n+1\), so AIPW and TMLE converge at rate
\(n^{1/4}\), while the C-Learner and direct plug-in retain the root-\(n\)
rate. \begin{yfrevision}\yf{In \Cref{section:theory_simple}, we give a
complementary finite-sample comparison when the inverse propensity is
nonintegrable. Under the conditions and zero-extension convention of
\Cref{thm:simple_variance}, a square-integrable envelope for the
outcome-model class ensures finite mean squared error for C-Learner,
whereas AIPW and squared-loss TMLE have infinite mean squared
error.}\end{yfrevision}

\subsection{Analytic Insights and Proof Sketch}
\label{section:reduction}

Our main mathematical idea is to reduce all estimators to a variant of the
one-step first-order correction template discussed in
\Cref{section:optimality_background}: a direct plug-in plus an estimated bias
correction, as in~\eqref{eqn:debiased}. In the linear model, all four
estimators reduce to this template, with different multiples of the same
correction
\begin{equation}
  \what\psi_1(s)=\Peval\what\mu_1
      +s\,\Peval\left[\frac A{\pi_n\opt}(Y-\what\mu_1(X))\right],
  \qquad s\in\R.
  \label{eq:family}
\end{equation}
The cross-fitted estimator is
\(\what\psi(s)=\tfrac12(\what\psi_1(s)+\what\psi_2(s))\), where
\(\what\psi_2(s)\) exchanges the roles of the training and evaluation splits.
By definition, \(\what\psi_\DM=\what\psi(0)\) and
\(\what\psi_\AIPW=\what\psi(1)\).

The next lemma shows that TMLE and the C-Learner also belong to this family,
with a data-dependent multiplier for each fold. We then derive the family's
limiting variance and the multipliers' limits to explain both parts of
\Cref{thm:comparison}.
\begin{lemma}[Reduction]
  \label{lem:reduction}
  \label{lem:main-reduction}
  Suppose the evaluation fold contains an observed outcome and the observed
  training matrix for \(X_0\) has rank \(q\). Then squared-loss TMLE satisfies
  \begin{equation}
    \what\psi_{1,\TMLE}=\what\psi_1(\what s_{1,\TMLE}),
    \qquad
    \what s_{1,\TMLE}
       =\frac{\Peval[(\pi_n\opt)^{-1}]}
                    {\Peval[A/(\pi_n\opt)^2]}.
    \label{eq:tmle-family}
  \end{equation}
  If the training Hessian
  \(\Ptrain[AXX^\top]+\diag(0_q,\lambda\Id_{d_n})\) is positive
  definite, define
  \begin{equation*}
    \ell_1(x):=x^\top
    \big(\Ptrain[AXX^\top]+\diag(0_q,\lambda\Id_{d_n})\big)^{-1}
    \Peval\left[\frac A{\pi_n\opt}X\right].
  \end{equation*}
  Then \(\what\mu_1^C-\what\mu_1\) is proportional to \(\ell_1\), and
  \begin{equation}
    \what\psi_{1,\CL}=\what\psi_1(\what s_{1,\CL}),
    \qquad
    \what s_{1,\CL}
    =\frac{\Peval[\ell_1(X)]}
    {\Peval[(A/\pi_n\opt)\ell_1(X)]}.
    \label{eq:s1}
  \end{equation}
  The same representation holds for minimum-norm interpolation when the
  matrix of residuals from regressing each column of the observed \(X_n\)
  on the observed \(X_0\) has rank
  \(n\Ptrain[A]-q<d_n\) and its nullspace contains a direction that changes
  the held-out first-order correction. Equation~\eqref{eq:s1-minnorm} gives
  the formula for \(\ell_1\) in this case.
\end{lemma}
\noindent See \Cref{section:proof-of-reduction} for the derivation.

This reduction expresses the change in the reported mean from refitting the
whole outcome model as a multiple of the usual first-order correction.
To prove \Cref{thm:comparison}, we have two remaining steps. First, we derive
the joint limiting distribution of \(\Peval\what\mu_1\),
\(\Ptrain\what\mu_2\), and their first-order corrections to obtain a CLT
for the family~\eqref{eq:family}. The two estimates are dependent because
\(\Ptrain\) determines \(\what\mu_1\) and evaluates \(\what\mu_2\), with
the roles reversed for \(\Peval\). Under sufficient overlap, accounting
for this dependence gives the following CLT
(\Cref{section:proof-of-clt}).
\begin{theorem}[First-order correction family]
  \label{thm:family-clt}
  Under the regular-overlap conditions of \Cref{thm:comparison},
  the first-order correction in~\eqref{eq:family} is \(\Op(n^{-1/2})\), and for
  every fixed \(s\),
  \(\sqrt{2n}(\what\psi(s)-\psi)\cd\Normal(0,V(s))\),
  with \(V(s)\) given in~\eqref{eq:V}.
\end{theorem}
Second, we show that the data-dependent multipliers converge at the scale
needed to use this CLT. Gaussian regression limits applied to the ratio
in~\eqref{eq:s1} give
\begin{equation}
  \frac{P[(\pi_n\opt)^{-1}]+\sigma_X^2\mathsf{adj}(\gamma,\lambda)}
       {1+\sigma_X^2\mathsf{adj}(\gamma,\lambda)}\,\what s_{1,\CL}\cp1,
  \qquad \what s_{1,\TMLE}\cp1.
  \label{eq:main-multiplier}
\end{equation}
The same limits hold after exchanging the folds. Under sufficient overlap,
substituting the multiplier limits changes each estimate by \(\op(n^{-1/2})\).
\Cref{thm:family-clt} therefore gives the four limiting variances in
\Cref{thm:comparison}.

The coefficient limits are proved in \Cref{section:coefficient-limits},
followed by the family CLT in \Cref{section:proof-of-clt}.
\Cref{section:proof-of-main} derives the multiplier limits and completes the
comparison at both overlap scales. These proofs also give analytic insights
into how the correction multiplier affects variance and convergence rates.

The variance \(V(s)\) is minimized at
\begin{equation}
  s\opt
    =\frac{\sigma_{\mathrm{fit}}^2(\gamma,\lambda)(1+\sigma_X^2)}
           {2\sigma_X^2\sigma_{\mathrm{fit}}^2(\gamma,\lambda)
             +(\kappa-\sigma_X^2)\sigma_{\mathrm{res}}^2(\gamma,\lambda)}.
  \label{eq:s-opt}
\end{equation}
The C-Learner's limiting multiplier lies between this minimizer and one,
strictly below one. Strict convexity of \(V(s)\) therefore gives strictly
smaller variance than AIPW and TMLE. The multiplier equals \(s\opt\) exactly when
\(\kappa=\sigma_X^2\), \(\lambda=0\), and \(\gamma>1\).
For the constant-propensity example~\(\pi_n\opt=1/2\), this means the
interpolating C-Learner asymptotically uses three quarters of the usual
first-order correction, the variance-minimizing choice within this family.
AIPW and TMLE instead use a limiting multiplier of one.

Under deteriorating overlap,~\eqref{eq:main-multiplier} gives a
revealing scale
\begin{equation*}
  P[(\pi_n\opt)^{-1}] \cdot \what s_{1,\CL}
  \cp1+\sigma_X^2\mathsf{adj}(\gamma,\lambda).
\end{equation*}
Here, the first-order correction grows relative to the root-\(n\) scale, but
the C-Learner attenuates it by the reciprocal mean inverse propensity:
\begin{equation*}
\begin{aligned}
  \Peval\left[\frac A{\pi_n\opt}(Y-\what\mu_1(X))\right]
  =\Op\left(\sqrt{\frac{P[(\pi_n\opt)^{-1}]}n}\right),
  ~~~\mbox{and}~~~\what\psi_{1,\CL}-\what\psi_{1,\DM}
  =\Op\left(\frac1{\sqrt{nP[(\pi_n\opt)^{-1}]}}\right)
\end{aligned}
\end{equation*}
Thus \(\what\psi_{1,\CL}-\what\psi_{1,\DM}=\op(n^{-1/2})\), so the C-Learner
retains the direct method's root-\(n\) limit, whereas AIPW and TMLE retain the
leading noise from the first-order correction.

%%% Local Variables:
%%% mode: latex
%%% TeX-master: "../main"
%%% End:
 % Section 4: proportional high-dimensional theory
\section{Experiments}
\label{section:experiments}

In this section, we present a series of experiments to demonstrate how C-Learner\ifdefined\msom\else\footnote{\begin{yfrevision}
\yfdelete{Here C-Learner refers to estimators using solutions to the constrained optimization problem for a chosen model class.}
\yf{Here ``C-Learner'' includes approximate constrained-learning implementations
for chosen outcome models. Other methods also admit a C-Learner interpretation, as discussed in
\Cref{section:methods}.}\end{yfrevision}} \fi
is flexible across data types and model classes. \ifdefined\msom
\begin{yfrevision}\yf{Here C-Learner includes approximate constrained
implementations; see \Cref{section:methods}.}\end{yfrevision}\fi
We also show that C-Learner achieves good empirical performance among all these settings, especially variants with low overlap.
%\kh{Separate the point on ``flexibility" and ``empirical performance". The flexibility is across different spaces of covariates (tabular or unstructured data), and across different statistical models (linear models, trees, and neural networks)}
%\kh{We show C-Learner achieves better empirical performances among all these settings.}
%\tc{edited}
% \kh{Clarify what C-Learner here we are using, since AIPW-SN and TMLE are C-Learners as well.}
First, in \Cref{section:ks}, we consider well-studied tabular simulated settings where existing debiasing methods (e.g. AIPW) are known to perform poorly~\cite{KangSc07, robins2007comment}.
%In this section, we also show that the same holds for targeting if additional heuristics or assumptions on the data generating process are not used.
C-Learner has lower mean absolute error than AIPW, self-normalized AIPW, and squared-loss TMLE without truncation, although the direct plug-in remains better with linear outcome models.
\begin{yfrevision}\yf{We also compare the methods using gradient-boosted
trees. Finally, we also discuss a dual formulation of the C-Learner
 and compare it with covariate balancing. The details are presented in \Cref{section:balance}.}\end{yfrevision}
%In \Cref{section:ks_more_results}, we show that the same holds for a more flexible function class using gradient-boosted trees for the outcome models.

To test the scalability of our approach, in \Cref{section:civilcomments} we construct and study a high-dimensional setting with text features, where we
%Following~\hn{CITE our nonexistent NeurIPS paper}, w
fine-tune a language model \citep{distilbert} under our constrained learning framework. Again, we observe that C-Learner outperforms one-step estimation and targeting, especially in settings with low overlap. We summarize the results of our experiments in the above two settings in \Cref{fig:intro_bar_comparison}.

%\ifdefined\pnas
\ifdefined\pnasformat
\begin{figure*}[t]
\else
\begin{figure}[t]
\fi
\centering
\includegraphics[width=\textwidth]{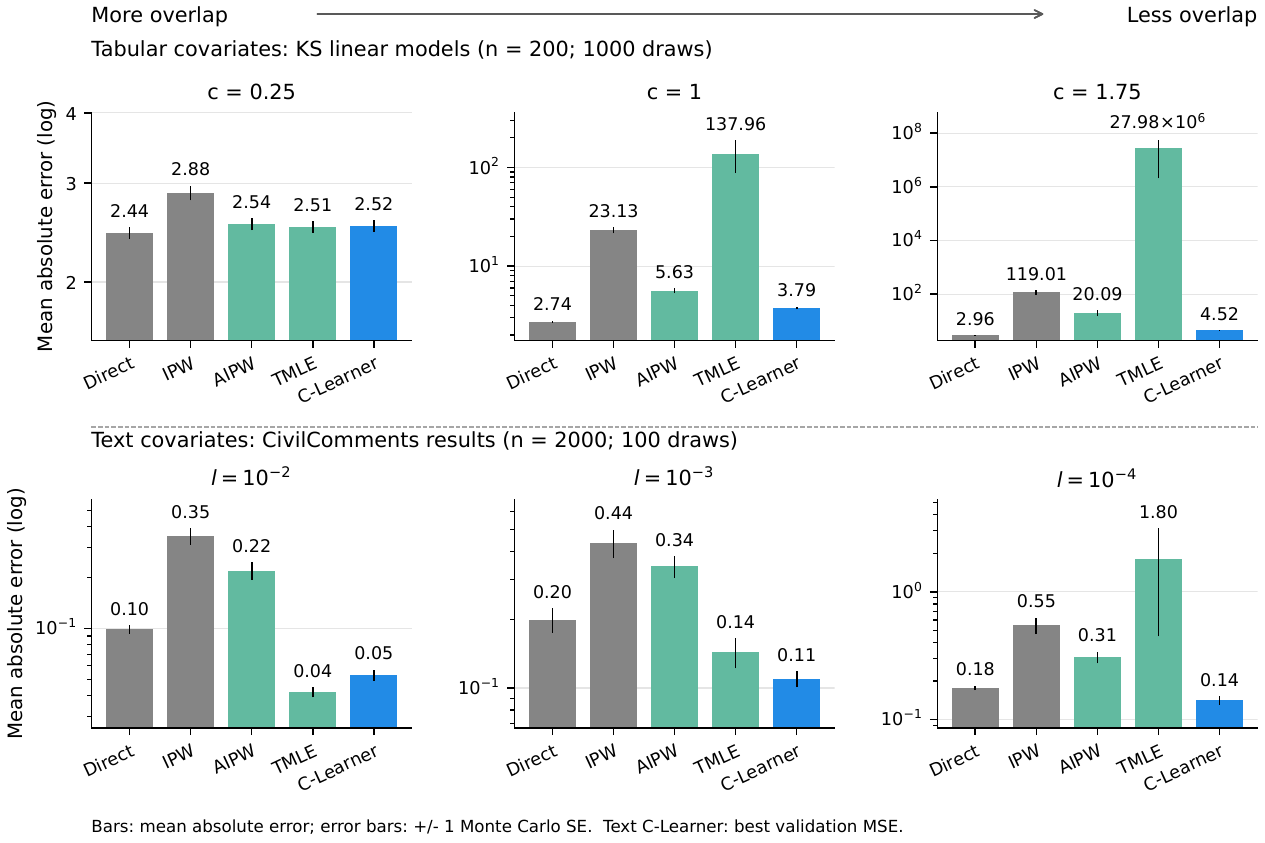}
\caption{\begin{yfrevision}\yf{Mean absolute error for the tabular linear
comparison (top: 1000 datasets of size 200 per $c$) and the text comparison
(bottom: 100 datasets of size 2000 per setting). Error bars show $\pm1$ Monte Carlo
standard error of MAE. The top panels use the KS setting in
\Cref{section:ks}; the bottom panels use CivilComments and the C-Learner
selected by validation MSE, as described in
\Cref{section:civilcomments}.}\end{yfrevision}}
\label{fig:intro_bar_comparison}
\ifdefined\pnasformat
\end{figure*}
\else
\end{figure}
\fi

%\fi
%In both settings, we observe that C-Learner outperforms one-step estimation and targeting, especially in settings with low overlap.

\ifdefined\msom\else
\begin{yfrevision}
\yfdelete{Lastly, on the Infant Health and Development Program dataset, C-Learner with gradient boosted regression trees matches the performance of one-step estimation and targeting.}
\yf{Lastly, in \Cref{section:ihdp_experiment}, we study the Infant Health and
Development Program dataset~\cite{brooks1992effects}. We instantiate the C-Learner with a projection-based
gradient-boosted-tree implementation, and we show that the C-Learner has the lowest mean absolute error among
the five estimators we compare.}
\end{yfrevision}
\fi

\subsection{Tabular Dataset from \citet{KangSc07}}
\label{section:ks}

% who compared a number of estimators for the population mean from incomplete data, which is functionally equivalent to estimating the ATE when assuming $Y(0):=0$, along with SUTVA, overlap, and ignorability (\Cref{section:ate}).
%We are interested in the problem of estimating the populational average when we have outcome labels missing at random.

We start with the  synthetic tabular setting constructed by \citet{KangSc07}, who
demonstrate empirically that the direct method (a naive plug-in estimator with a linear outcome model, labeled as ``OLS'' in \citep{KangSc07}) achieves better performance than
asymptotically optimal methods in their setting.
\citet{robins2007comment} note that
settings in which some subpopulations have a much higher
probability of being treated than others,
%\citep{robins1995analysis,robins2000inference} \tc{what is this citation? how are they relevant? who put this here?},
like
%including the one proposed by
in \citet{KangSc07},
are very challenging for existing asymptotically optimal methods.
%In response,
%\citet{robins2007comment} note that
%settings in which some subpopulations have a much higher
%probability of being treated than others,
%\citep{robins1995analysis,robins2000inference} \tc{what is this citation? how are they relevant? who put this here?},
%including the one proposed by \citet{KangSc07},
%can be very challenging for existing asymptotically optimal methods, as such methods are sensitive to propensity weights ($1/\what\pi(X)$) that can take extreme values in these settings.
Thus, we %investigating the performance of
evaluate the C-Learner on this well-known and challenging setting. %, which we describe below. % setting allows us to assess its empirical stability on a well-known dataset, which we describe below.

\begin{yfrevision}\yf{We estimate the population mean
$\psi(P):=P[Y(1)]=210$ from observations $(X,A,AY)$, using the
convention $Y(0):=0$.}\end{yfrevision}
\ifdefined\pnas
The covariates $Z$ that generate the data are low-dimensional, and
the covariates $X$ that are observed are a misspecified version of $Z$. See \Cref{section:ks_dgp} for more details.
\else
The \textit{true} outcome and treatment  mechanisms depend on covariates $\xi \sim N(0, I) \in \mathbb R^4$ and $\varepsilon\sim N(0,1)$ via
%\begin{align*}
$Y(1) = 210 + 27.4 \xi_{1} + 13.7 \xi_{2} + 13.7 \xi_{3} + 13.7 \xi_{4} + \varepsilon$ and $
\pi\opt(\xi)  = \frac{\exp(-\xi_{1} + 0.5 \xi_{2} - 0.25 \xi_{3} - 0.1 \xi_{4})}{1+\exp(-\xi_{1} + 0.5 \xi_{2} - 0.25 \xi_{3} - 0.1 \xi_{4})}.$
%\end{align*}
The target is $P[Y(1)]=210$, whereas $P[Y(1)\mid A=1]\approx200$, so a naive average of the treated outcomes has bias approximately $-10$. For a random sample of 100 data points, the true propensity score can be as low as 1\% and as high as 95\%. We focus on the misspecified setting, where instead of observing $\xi$, the modeler observes
$
X_{1} = \exp(\xi_{1}/2),\;X_{2} = \xi_{2}/(1+\exp(\xi_{1})) + 10,\;X_{3} = (\xi_{1}\xi_{3}/25 + 0.6)^3,\;X_{4} = (\xi_{2} + \xi_{4} + 20)^2.
$
\fi

% Generated by doc/experiments-revised/build_audited_tables.py
\begin{yfrevision}
\begin{table}[t]
\centering
\footnotesize
\setlength{\tabcolsep}{5pt}
\begin{tabular}{lrrrrrrrr}
\toprule
& \multicolumn{4}{c}{$n=200$} & \multicolumn{4}{c}{$n=1000$} \\
\cmidrule(lr){2-5}\cmidrule(lr){6-9}
Method & \multicolumn{2}{c}{Bias} & \multicolumn{2}{c}{Mean Abs Err} & \multicolumn{2}{c}{Bias} & \multicolumn{2}{c}{Mean Abs Err} \\
\cmidrule(lr){1-1}\cmidrule(lr){2-3}\cmidrule(lr){4-5}\cmidrule(lr){6-7}\cmidrule(lr){8-9}
\yf{Direct} & -0.590 & (0.107) & 2.741 & (0.065) & -0.991 & (0.046) & 1.409 & (0.033) \\
\yf{IPW} & 17.267 & (1.708) & 23.129 & (1.637) & 66.243 & (26.092) & 66.619 & (26.091) \\
\yf{IPW-SN} & 3.045 & (0.270) & 5.253 & (0.233) & 6.723 & (0.334) & 6.968 & (0.329) \\
\greymidrule
\yf{AIPW} & -4.453 & (0.324) & 5.634 & (0.305) & -26.567 & (12.917) & 26.609 & (12.917) \\
\yf{AIPW-SN} & -3.561 & (0.195) & 4.737 & (0.168) & -7.830 & (0.427) & 7.864 & (0.426) \\
\yf{TMLE} & -135.691 & (50.381) & 137.957 & (50.374) & -723.680 & (710.192) & 723.689 & (710.192) \\
\yf{C-Learner} & \textbf{-2.611} & (0.127) & \textbf{3.789} & (0.093) & \textbf{-4.358} & (0.083) & \textbf{4.387} & (0.082) \\
\greydashedmidrule
\yf{TMLE-L} & -1.953 & (0.118) & 3.349 & (0.081) & \textbf{\textit{-1.821}} & (0.056) & \textbf{\textit{2.099}} & (0.045) \\
\yf{C-Learner-L} & \textbf{\textit{-1.559}} & (0.119) & \textbf{\textit{3.194}} & (0.080) & -2.383 & (0.075) & 2.650 & (0.066) \\
\bottomrule
\end{tabular}
\caption{\yf{Comparison with linear outcome models over 1000 KS simulations per sample size, without propensity truncation. 
Initial outcome and propensity fits omit intercepts; fitting and evaluation use the same sample. 
Asymptotically optimal methods are below the solid divider; their logistic-link variants are below the dashed divider. 
Bold marks the lowest absolute bias and MAE among the squared-loss methods below the solid divider; bold-italic marks the best including the logistic-link methods. 
Standard errors are in parentheses to the right of the point estimates.}}
\label{tb:linear_simulation}
\end{table}
\end{yfrevision}

\subsubsection{Linear Outcome Models}
\label{section:ks_linear}
%\kh{Check the paragraph titles in this subsection.} \yf{What needs to be fixed?} \kh{We have three sections: linear outcome models, prop score sensitivity, and xgboost. Shall we put sensitivity after discussing both the linear models and the tree? Also we shall make upper/lower case in the initials consistent.} Following \citet{KangSc07}, we consider learning outcome models that are linear in $X$ and propensity (treatment) models that are linear logistic regressions of the treatment on $X$. Both models are misspecified, as is common in real world applications. We use squared loss to train the outcome model on treated units ($A=1$) and we use the logistic loss to learn the propensity model.
%\hn{We don't do sample splitting / cross-fitting?} \tc{No, ours is similar to K\&S, which doesn't do cross-fitting, either. We also do a similarly strange data split on IHDP as Riesznet does...}
%\kh{The beginning of this section is slightly abrupt. We can start with something like ``
Linear outcome models are appealing due to their simplicity and interpretability, and are a fundamental tool for both theorists and practitioners.
%In estimating average treatment effects, linear outcome models are appealing due to their simplicity and interpretability, making them a fundamental tool for both theorists and practitioners.
\looseness=-1 Here,
we fit a linear model $\what\mu(X)$ on covariates $X$ to predict the potential outcome $Y$. %, and assess the performance of different ATE estimators in \Cref{tb:linear_simulation}.
\begin{yfrevision}
\yf{The centered observed-sample $R^2$ averages approximately 0.83 for both samples, and the logistic propensity fits have mean AUC of aproximatelly 0.74.}
\end{yfrevision}
%This leads to outcome models with $R^2$ close to 0.99 and propensity models with ROC around 0.75.
Following the original study by \citet{KangSc07},
we do not distinguish between data splits; there is only one split (%we do not distinguish between
$\what{P}_{\rm train}=\what{P}_{\rm val}=\what{P}_{\rm eval}$ in
\Cref{section:data_splitting}). %\Cref{section:methodology}).
%\tc{Yuri: can you confirm the above?}\yf{check!}\tc{thanks!}

%First, we consider the direct method, which is the plug-in of the outcome model $\what \mu$ trained under the usual objective (``OLS''
%in \cite{KangSc07}), i.e. %$P_n[\what\mu(X)]$.
% we explain direct method at beginning of section.

%; we call this the ``direct'' method. We also compute the IPW and AIPW (one-step estimation, \eqref{eqn:debiased}) \citep{glynn2010introduction}  and their corresponding normalized versions (``IPW-SN'' and ``AIPW-SN'' \eqref{eq:normalized_aipw} in \cite{KangSc07}), as well as the TMLE (targeting, \eqref{eqn:tmle}) \citep{VanDerLaanRo11} and our proposed method (\Cref{section:methods}). As we consider linear functions of $X$, the C-Learner $\what \mu_C$ is computed in closed form, as in \Cref{section:methods_linear}. See \Cref{section:ks_experiment_details} for additional details.

We present a comparison of mean estimators that use $\what \mu, \what \pi$ as described above in \Cref{tb:linear_simulation}.
In this table, ``Direct'' refers to the plug-in with the outcome model trained as usual (``OLS'' in \citep{KangSc07}), ``IPW-SN'' \citep{KangSc07} refers to self-normalized IPW (a.k.a. Hajek estimator~\citep{ding2023course}), and ``AIPW-SN'' refers to self-normalized AIPW~\eqref{eq:normalized_aipw}.
%Observe that the C-Learner outperforms other asymptotically optimal methods in terms of mean absolute error (MAE). %\footnote{The median absolute error is reported in \cite{KangSc07}. Qualitative results are the same under this additional metric, which we omit for easier exposition.}
%In \Cref{section:ks_more_results}, we show similar results for other metrics such as RMSE and median absolute error.
\begin{yfrevision}\yf{C-Learner performs best out of asymptotically optimal methods
(AIPW, AIPW-SN, TMLE, C-Learner) in this setting, demonstrating strong numerical stability
despite extreme inverse propensity weights. In some cases, the C-Learner improves upon AIPW
and TMLE by orders of magnitude.}\end{yfrevision}\ifdefined\msom\else \footnote{The median absolute error is reported in \cite{KangSc07}. Qualitative results are the same under this additional metric, which we omit for easier exposition. In \Cref{section:ks_more_results}, we show similar results for other metrics such as RMSE and median absolute error.} \fi\ifdefined\msom
In \Cref{section:ks_more_results}, we show similar results for other metrics such as RMSE and median absolute error.\fi

\begin{yfrevision}
\yf{Even when outcomes are not truly bounded (here they are Gaussian), it is common to implement
the TMLE via a logistic link for improved estimator stability. Therefore, in addition to TMLE with the squared loss,
we also compare logistic-link TMLE-L. To use the logistic link, one needs to normalize the values of Y to be
between zero and one, which requires an upper and lower bound for Y . Even when Y is not actually
bounded, such lower and upper bounds can be roughly estimated from data. We can also formulate
the C-Learner using the logistic link, which we call “C-Learner-L”. We retain the same-sample
KS protocol and report the untruncated results below the dashed divider in
\Cref{tb:linear_simulation}. Both logistic-link methods improve on squared-loss
C-Learner. C-Learner-L has lower MAE at $n=200$ (3.194 versus 3.349), whereas
TMLE-L has lower MAE at $n=1000$ (2.099 versus 2.650). The direct method
still has the lowest MAE overall. Implementation details and sensitivities
to truncation and targeting are in \Cref{section:ks-logistic-matched}.}
\end{yfrevision}
\paragraph{Numerical Instability and Common Ways To Mitigate}

%For instance, the self-normalized AIPW~\eqref{eq:normalized_aipw} %and TMLE~\eqref{eqn:tmle}
%can be viewed as normalizing inverse propensity weights $1/\what\pi$, and we find it is more stable than the standard AIPW in practice. %This normalization is done either by using $\sum_{i=1}^n \frac{A_i}{\what\pi(X_i)}$, or by dividing this factor by $\sum_{i=1}^n \frac{A_i}{\what{\pi}(X_i)^2}$. The normalized inverse propensity weights are less extreme than the original weights, so the estimators are more stable in practice. \\
% \tc{TMLE is not more stable in practice despite being reweighted}
% \kh{In addition to normalizing propensity scores, outcome models can also be normalized. For instance, Chapter 21 of \cite{VanDerLaanRo11} suggests treating outcomes as being bounded and then normalizing the outcomes so that resulting observations belong to the $[0,1]$ interval. (\textbf{Keep the rest.})}
% \tc{I don't think the logistic TMLE is more stable just because it normalizes the outcome model; normalizing the outcome model this way would not help for e.g. AIPW.}
% \kh{I see.}

%As methods such as IPW and AIPW are known to be unstable if the propensity score $\what\pi(X)$ is close to 0, a common heuristic
ATE estimation methods involving inverse propensity weights are known to be numerically unstable if the propensity score $\what\pi(X)$ is close to 0. One way to address this instability
is through normalizing these inverse propensity score weights. We find that self-normalized versions of IPW and AIPW are more stable and have better performance than their original versions, e.g. in \Cref{tb:linear_simulation}.

Another common heuristic to handle extreme estimated propensity scores is to truncate $\what\pi(X)$ for a chosen small $\eta>0$ so that if $\what\pi(X)<\eta$, we redefine $\what \pi(X) := \eta$ \citep{ding2023course,cole2008constructing}. %,kish1992weighting,potter1990study}.
%\tc{I have not been able to verify that the others are relevant}
%\tcdebug{check cites?} %\tc{should we cite something else, and not just Peng Ding's book?}. \kh{We have already cited some literature on weight truncation or weight truncation in the next few sentences.} \tc{one example you cite is \cite{cole2008constructing} but it is not the canonical reference (maybe ok)}
As dealing with extreme estimated propensity scores $\what \pi(X)$ can be challenging, another option is to avoid these difficult $X$ with small $\what \pi(X)$ entirely by adjusting the estimand to exclude or down-weight values of $X$ for which $\what \pi(X)$ takes extreme values~\citep{CrumpHoImMi06,li18}.
%\tc{\cite{li18} is not meant to discuss explicit prescriptions of $\eta$; it discusses different ways to avoid these difficult $X$ with small $\what\pi(X)$, e.g. changing the estimand. Is the phrasing confusing?}
% \kh{I don't think \cite{li18} involves any explicit prescriptions of $\eta$. There are two ways of weight truncation (involving $\eta$) in general:
% \begin{itemize}
%     \item ``Exclude": One way is to remove samples at the tails of the prop score distribution, or to discard units when they fail to find a match in the prop scores. One can either use rule-of-thumb to truncate the range to be $[\eta,1-\eta]$ (or $[1/\eta,1-1/\eta]$) \cite{CrumpHoImMi06}, or have some complications in deciding the range of weights \cite{sturmer2010treatment}.
%     \item Another way (which is what we are using now) is to informally reset the values at the tails with the cutoff values, e.g., in \cite{cole2008constructing} ``One simple way to explore this bias-variance tradeoff is to progressively truncate (38) the weights as shown in table 4. Specifically, the weights are progressively truncated by resetting the value of weights greater (lower) than percentile p (100 – p) to the value of percentiles p (100 – p)." Some standard references of this second way of weight truncation include \cite{kish1992weighting,potter1990study}
% \end{itemize}
% Apart from the above, another way is ``overlap weighting" \cite{li18,li2018balancing}.}
We show in \Cref{section:ks_more_results} that while AIPW and TMLE with squared-loss perform poorly using raw estimated propensity scores $\what\pi$, both perform better after truncation.

\begin{figure}[t]
    \centering
    \includegraphics[width=\linewidth]{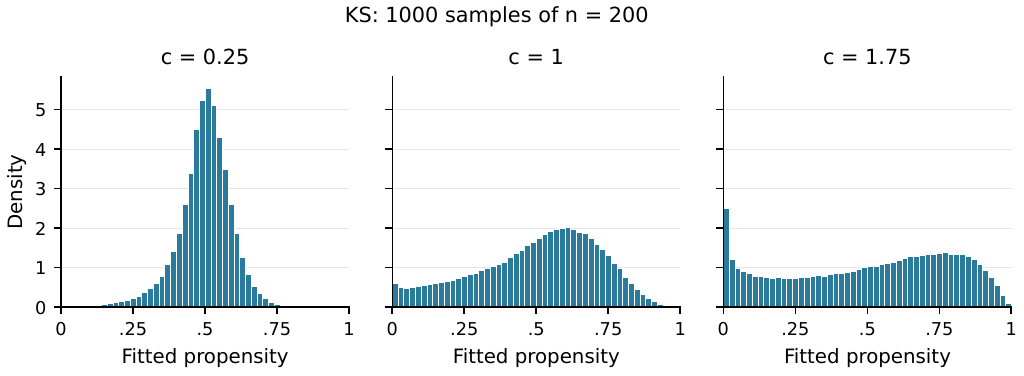}
    %\caption{Kernel density estimation of fitted propensity
    \caption{Empirical density of fitted propensity
    scores $\what\pi$, for \citet{KangSc07} settings modified with parameter $c\in\{0.25,1,1.75\}$. \begin{yfrevision}\yf{Each panel pools 1000 datasets of size 200 into a 50-bin histogram, with common density axes.}\end{yfrevision}}
    \label{fig:density_propensity}
\end{figure}

\looseness=-1 \paragraph{Estimator Performance When Varying Overlap Between Treatment and Control} In order to understand how quickly estimator performance deteriorates in settings with low overlap, we compare how estimators perform in different data settings that vary by the degree of overlap or variation in propensity scores by modifying the treatment assignment. Specifically, we scale the logit of the treatment mechanism by a parameter $c$ such that if $c = 0$, every observation has an equal chance of being observed, while as $c$ increases, treatment probabilities become more extreme (i.e. closer to 0 or 1).
See \Cref{section:ks_sensitivity} for more details.
\begin{yfrevision}
\yf{\Cref{fig:density_propensity} extends the analysis of \Cref{fig:intro_bar_comparison} and shows fitted propensities for
$c\in\{0.25,1,1.75\}$; $c=1$ is the original KS setting.
\Cref{fig:mae_vs_prop} displays MAE on a log scale. As $c$ grows
from 0.25 to 1.75, C-Learner's MAE rises from 2.517 to 4.519, whereas
AIPW's rises from 2.536 to 20.091. Squared-loss TMLE has much larger
tail-driven errors at the largest values of $c$. Although some methods deteriorate faster than others, the C-Learner uniformly beats all other debiased methods for all the ranges of overlap tested. \Cref{tb:prop_stats_ks,tb:1_prop_stats_ks}
in \Cref{section:ks_sensitivity}
report medians of the per-dataset propensity summaries.}
\end{yfrevision}

%In a similar vein,
\citet{robins2007comment} noted that methods like the AIPW, for example, perform poorly in the setting proposed by \citet{KangSc07}
due to high variability in propensity scores. They thus suggest considering the flipped version of the task, in which we want to estimate $\wt\psi(P):= P[Y(0)]$ instead; there, propensity scores are not extreme, and asymptotically optimal methods such as AIPW work well. \citet{robins2007comment} thus ask the following question (rewritten to match our notation): ``Can we find doubly robust estimators that, under the authors' chosen joint distribution for $(X,A,(1-A)Y)$, both perform almost as well as the direct method for $P[Y(1)]$
and yet perform better than the direct method for $P[Y(0)]$?''
We observe that the C-Learner is such an estimator, as it performs well for estimating both $P[Y(0)]$ and $P[Y(1)]$. See \Cref{tb:flipped_ks} in \Cref{section:ks_more_results} for
%results on
$P[Y(0)]$.

\ifdefined\pnas
\else
\begin{figure}[t]
    \centering
    \includegraphics[width=\linewidth]{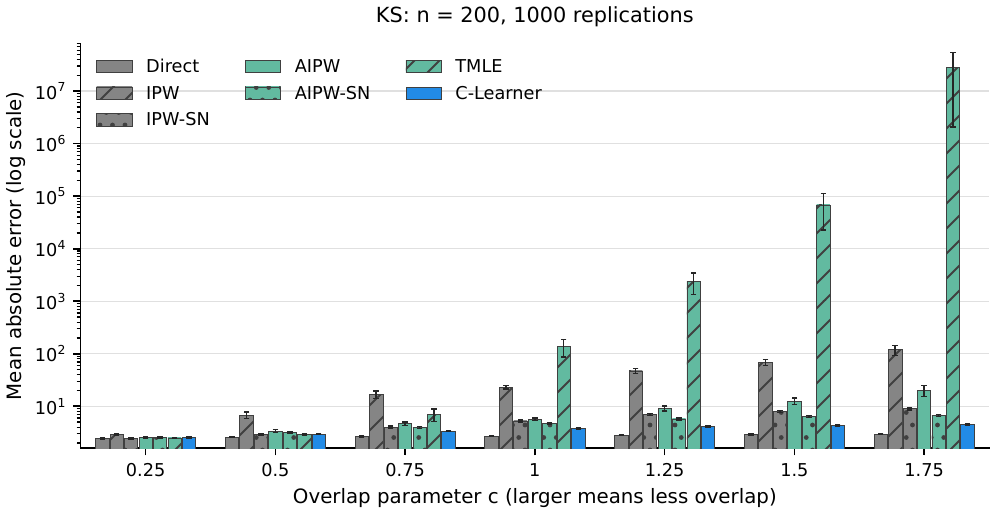}
    \caption{\begin{yfrevision}\yf{Mean absolute error for linear
    outcome models in the KS setting, as the
    observation logit is scaled by $c$. Each bar uses 1000 datasets of
    size 200; error bars are $\pm1$ Monte Carlo standard error of MAE.
    The vertical axis is logarithmic.}\end{yfrevision}}
    \label{fig:mae_vs_prop}
\end{figure}

%%% Local Variables:
%%% mode: latex
%%% TeX-master: "../main"
%%% End:

\fi

\ifdefined\response\newpage\newpage\clearpage\pagebreak\else\fi
\Needspace{6\baselineskip}
\subsubsection{The Dual C-Learner and Covariate Balancing.}
\label{section:ks_ipw}
\begin{yfrevision}
\yf{We also evaluate the propensity-side constraint in \Cref{section:balance}.
Holding the initial linear outcome fit fixed, the dual C-Learner fits a
logistic propensity model subject to its empirical balance constraint and
uses the result for inverse-propensity weighting. \Cref{tb:linear_ipw}
compares IPW, over-identified CBPS, and the dual C-Learner
solution, each with and without self-normalization. All use the same
1000 KS datasets per sample size.
The dual C-Learner has MAE 3.445 and 3.593 at $n=200$ and 1000;
its self-normalized version has MAE 2.609 and 1.330, the lowest values in
this table. Thus self-normalization, not the unnormalized dual estimator,
drives the strongest comparison here.
\Cref{section:dual_details} gives implementation details;
\Cref{tb:linear_cbps} in \Cref{section:linear_cbps_ks}
reports additional estimators with the same
over-identified CBPS propensities.}
\end{yfrevision}
% Generated by doc/experiments-revised/build_audited_tables.py
\begin{yfrevision}
\begin{table}[t]
\centering
\footnotesize
\setlength{\tabcolsep}{5pt}
\begin{tabular}{lrrrrrrrr}
\toprule
& \multicolumn{4}{c}{$n=200$} & \multicolumn{4}{c}{$n=1000$} \\
\cmidrule(lr){2-5}\cmidrule(lr){6-9}
Method & \multicolumn{2}{c}{Bias} & \multicolumn{2}{c}{Mean Abs Err} & \multicolumn{2}{c}{Bias} & \multicolumn{2}{c}{Mean Abs Err} \\
\cmidrule(lr){1-1}\cmidrule(lr){2-3}\cmidrule(lr){4-5}\cmidrule(lr){6-7}\cmidrule(lr){8-9}
\yf{IPW} & 17.267 & (1.708) & 23.129 & (1.637) & 66.243 & (26.092) & 66.619 & (26.091) \\
\yf{IPW-SN} & 3.045 & (0.270) & 5.253 & (0.233) & 6.723 & (0.334) & 6.968 & (0.329) \\
\yf{CBPS} & -6.439 & (0.260) & 8.351 & (0.199) & 1.154 & (0.178) & 3.995 & (0.131) \\
\yf{CBPS-SN} & -2.869 & (0.104) & 3.611 & (0.078) & -1.245 & (0.061) & 1.764 & (0.046) \\
\yf{C-Learner-Dual} & -2.260 & (0.116) & 3.445 & (0.082) & -3.559 & (0.067) & 3.593 & (0.065) \\
\yf{C-Learner-Dual-SN} & 0.001 & (0.103) & 2.609 & (0.062) & 0.575 & (0.050) & 1.330 & (0.032) \\
\bottomrule
\end{tabular}
\caption{\yf{Comparison of propensity-weighted estimators. 1000 simulations from the KS setting in \Cref{section:ks}, with population mean 210. Linear outcome and logistic propensity models are fitted without intercepts on the same sample. CBPS uses the over-identified model with an intercept. Implementation details are in \Cref{section:dual_details}. Standard errors are in parentheses.}}
\label{tb:linear_ipw}
\ifyfreview\par\smallskip
{\footnotesize\yfdelete{Previous table: Param-Fluc and Param-Fluc-SN rows, and earlier C-Learner-Dual and C-Learner-Dual-SN results.}\par}\fi
\end{table}
\end{yfrevision}

\label{section:ks_ipw-end}
\ifdefined\response\newpage\newpage\clearpage\pagebreak\else\fi

\subsubsection{Gradient Boosted Regression Tree Outcome Models.}
\label{section:ks_xgb}
\leavevmode
\begin{yfrevision}

We now demonstrate the flexibility of the C-Learner by using gradient boosted regression trees as outlined in \Cref{alg:boosting} in \Cref{section:methods_boosting}. 
The propensity model $\what\pi$ is fit as an L1-regularized logistic regression on covariates $X$.
Since sample splitting is more appropriate for this flexible model class, 
we use cross-fitting with $K=2$ folds as
we describe in \Cref{section:data_splitting}, and treat more formally in \Cref{section:theory_text}.  
\yf{Additional implementation details, including the hyperparameter grid
and coverage results in \Cref{tb:coverage_ks_xgb}, are in
\Cref{section:ks-xgb-projection}.}

% Generated by doc/experiments-revised/build_audited_tables.py
\begin{yfrevision}
\begin{table}[t]
\centering
\footnotesize
\setlength{\tabcolsep}{5pt}
\begin{tabular}{lrrrrrrrr}
\toprule
& \multicolumn{4}{c}{$n=200$} & \multicolumn{4}{c}{$n=1000$} \\
\cmidrule(lr){2-5}\cmidrule(lr){6-9}
Method & \multicolumn{2}{c}{Bias} & \multicolumn{2}{c}{Mean Abs Err} & \multicolumn{2}{c}{Bias} & \multicolumn{2}{c}{Mean Abs Err} \\
\cmidrule(lr){1-1}\cmidrule(lr){2-3}\cmidrule(lr){4-5}\cmidrule(lr){6-7}\cmidrule(lr){8-9}
\yf{Direct} & -4.482 & (0.101) & 4.724 & (0.090) & -1.652 & (0.040) & 1.777 & (0.034) \\
\yf{IPW} & 79.715 & (50.931) & 95.200 & (50.905) & 93.393 & (49.379) & 100.574 & (49.365) \\
\yf{IPW-SN} & -3.524 & (0.228) & 6.388 & (0.154) & 0.348 & (0.208) & 3.886 & (0.168) \\
\greymidrule
\yf{AIPW} & 19.651 & (17.664) & 25.052 & (17.657) & 11.762 & (5.384) & 13.159 & (5.380) \\
\yf{AIPW-SN} & \textbf{-1.521} & (0.155) & 3.859 & (0.106) & \textbf{0.467} & (0.112) & 1.905 & (0.095) \\
\yf{TMLE} & -115.162 & (273.634) & 412.505 & (273.347) & 8.461 & (7.653) & 10.436 & (7.651) \\
\yf{C-Learner} & -2.856 & (0.104) & \textbf{3.531} & (0.080) & -0.883 & (0.041) & \textbf{1.263} & (0.029) \\
\bottomrule
\end{tabular}
\caption{\yf{Comparison with gradient-boosted regression trees and two-fold cross-fitting. Each sample size uses 1000 KS simulations. All methods share the same initial outcome and propensity estimates. There is no additional propensity truncation beyond the common $10^{-8}$ numerical guard. TMLE uses squared loss; C-Learner uses the projection procedure. Asymptotically optimal methods are below the solid divider, with the lowest absolute bias and MAE among these methods in bold. Standard errors are in parentheses.}}
\label{tb:linear_simulation_xgb}
\end{table}
\end{yfrevision}

\yf{\Cref{tb:linear_simulation_xgb} compares all estimators with the same
nuisance predictions and splits over 1000 datasets at each sample size.
The data follow the same KS construction used for the linear models.}
\yf{Without truncation, C-Learner has the lowest observed MAE:
3.531 and 1.263 at $n=200$ and 1000, compared with 4.724 and 1.777
for the direct method. As we demonstrate in
\Cref{tb:ks-xgb-projection-truncation} in \Cref{section:ks-xgb-projection},
propensity truncation changes the ranking and
TMLE has lower MAE than C-Learner at high clipping thresholds. Finally,
\Cref{tb:coverage_ks_xgb} in \Cref{section:ks-xgb-projection} reports
confidence interval coverage separately from point-estimation error. Nominal 95\%
intervals for C-Learner have coverage 0.821 and 0.887, so the improvement
in point estimation does not ensure nominal coverage.}
\end{yfrevision}

\subsection{CivilComments Dataset and Neural Network Language Models}
\label{section:civilcomments}
% While BERT has been referred to as a large language model, more recently, colloquially, LLMs tend to refer to models larger than the ones we use
% I'm on the fence on what to call these
Neural networks can learn good feature representations for image and text data.
%Despite the recent interest in using  neural networks for fitting nuisance parameters~\cite{chernozhukov2022riesznet,shi2019adapting}, prior related work do not evaluate their methods outside of tabular settings.
%To fill this gap,
% sadly, other papers have done this in the years in between....
We construct a new semisynthetic causal inference dataset using text covariates.

\paragraph{Setting}
Content moderation is a fundamental problem
for maintaining the integrity of social media platforms.
We consider a setting in which we wish to
measure the average level of toxicity across all user
comments.
It is infeasible to have human experts label all comments for
toxicity, and the comments that
get flagged for human labeling may be
a biased sample.
This is a mean missing outcome problem (\Cref{section:ate}).
%We pose this as a missing data problem, which is equivalent to our usual goal of estimating the ATE~\eqref{eqn:ate} as in \Cref{section:ate}.

We use the CivilComments dataset \citep{civilcomments}, which contains
real-world online comments and corresponding human-labeled toxicity scores.
The dataset contains toxicity labels $Y(1)$ for all
comments $X$ (which provides a ground truth to compare to),
and we construct the labeling (treatment) mechanism $A \in \{0, 1\}$
to induce selection bias.
Specifically,
whether the toxicity label for a comment can be observed is drawn as
$    A\sim \textrm{Bernoulli}( g(X) )$ where
$g(X)=\textrm{clip}(b(X), l, u)$
and $\textrm{clip}(y,l,u):=\min(u,\max(l,y))$.
% whether the toxicity label for a comment can be observed is drawn as
% \begin{equation*}
%     A\sim \textrm{Bernoulli}( g(X) )
%     ~~~\mbox{where}~~~
%     g(X)=\textrm{clip}(b(X), l, u),
%     \quad
%     \textrm{clip}(y,l,u)=\begin{cases}
%         \max(l,y) & \textrm{if}\;  y \leq  l \\
%         \min(u,y) & \textrm{if}\; y > u
%     \end{cases}.
% \end{equation*}
%In practice, w
We set $u=0.9$ and $l=10^{-4}$.
A lower $l$ implies more extreme $\pi\opt(X)$ (and thus also $\what \pi(X)$), i.e. less overlap. %and more extreme propensity weights.
Here, $b(X)\in[0,1]$ is a continuous measure of whether comment $X$ relates to the demographic identity ``Black'',
 from the dataset.
The labeled data suffers from the following selection bias: within this dataset,
comments mentioning the demographic identity ``Black'' tend to be labeled as more toxic, compared to ones that don't, so a naive average of toxicity over labeled units would overestimate the overall toxicity. From a causal perspective, we have induced confounding, which can be handled by ATE estimators. %which ATE estimators can handle. %, which we can handle by using estimators for the ATE.

\paragraph{Procedure}
\ifdefined\pnas
    We demonstrate how C-Learner can be instantiated with neural networks. To learn $\what{\pi}$, $\what{\mu}$, and $\what{\mu}^C$,  we fine-tune a pre-trained DistilBERT model \citep{distilbert} with a linear head using stochastic gradient descent.

    We fit nuisance estimators on 100 re-drawn datasets of size 2000 each, from the full dataset of size 405,130.
    %\kh{what is this 405,130 size} \tc{the full set of comments from civilcomments. maybe we could put in appendix instead?} \kh{Oh I see. I thought something was 405 by 130}
    For C-Learner, we consider two ways of picking hyperparameters ($\lambda$ and learning rate), as described in \Cref{section:methods_nn}: one in which we choose hyperparameters for the best outcome model loss on validation data, and one in which we choose hyperparameters that minimize the size of the bias shift.
    Additional experiment details are in \Cref{section:nlp_details}.
\else
    We demonstrate C-Learner with neural networks. To learn $\what{\pi}$, $\what{\mu}$, and $\what{\mu}_C$, we fine-tune a pre-trained DistilBERT model \citep{distilbert} with a linear head using stochastic gradient descent, with $\what\mu^C$ learned using the procedure described in \Cref{section:methods_nn}.
    We train $\what \mu$
    using squared loss on labeled units, and propensity models $\what \pi$ using the logistic loss.
    We fit nuisance estimators on 100 re-drawn datasets of size 2000 each, from the full dataset of size 405,130. On each dataset draw, we use cross-fitting, as described in \Cref{section:theory_text}, with $K=2$ folds, for all estimators. We consider two ways of picking hyperparameters ($\lambda$ and learning rate), as described in \Cref{section:methods_nn}: one in which we choose hyperparameters for the best mean squared error on validation data, and one in which we choose hyperparameters that minimize the size of the bias shift.
    Additional  details are in \Cref{section:nlp_details}.
\fi

%On each split, $\what \mu$ was only trained to predict toxicity ($Y$) on comments $X$ with labels ($A=1$), while $\what \pi$ was trained to predict labeling $A$ over all $X$.
%$\what\mu$ is trained using MSE loss on labeled units, and $\what\pi$ is trained using logistic loss on all units.
%and evaluate them to

\paragraph{Results}

\begin{table}[t]
\centering
\footnotesize
\begin{tabular}{lrrrr}
\toprule
Method & \multicolumn{2}{c}{Bias} & \multicolumn{2}{c}{Mean Abs Err} \\
\cmidrule(lr){1-1}
\cmidrule(lr){2-3}
\cmidrule(lr){4-5}
% Direct & 0.147 & (0.028) & 0.200 & (0.025) \\
% IPW & 0.417 & (0.064) & 0.437 & (0.063) \\
% IPW-SN & 0.079 & (0.015) & 0.120 & (0.013) \\
% \greymidrule
% AIPW & \textbf{-0.056} & (0.052) & 0.344 & (0.039) \\
% AIPW-SN & 0.089 & (0.020) & 0.134 & (0.018) \\
% TMLE & 0.074 & (0.025) & 0.144 & (0.022) \\
% C-Learner (best val MSE) & 0.067 & (0.012) & 0.110 & (0.009) \\
% C-Learner (min bias shift) & 0.060 & (0.011) & \textbf{0.099} & (0.008) \\
Direct & 0.173 & (0.008) & 0.177 & (0.007) \\
IPW & 0.504 & (0.084) & 0.546 & (0.081) \\
IPW-SN & 0.114 & (0.017) & 0.153 & (0.014) \\
\greymidrule
AIPW & 0.084 & (0.043) & 0.307 & (0.032) \\
AIPW-SN & 0.116 & (0.018) & 0.161 & (0.014) \\
TMLE & -1.264 & (1.361) & 1.802 & (1.355) \\
C-Learner (best val MSE) & 0.103 & (0.015) & 0.141 & (0.011) \\
C-Learner (smallest bias shift) & \textbf{0.075} & (0.012) & \textbf{0.115} & (0.008) \\
\bottomrule
\end{tabular}

% \begin{tabular}{lrrrr}
% \toprule
% Method & \multicolumn{2}{c}{Bias} & \multicolumn{2}{c}{Mean Abs Err} \\
% \cmidrule(lr){1-1}
% \cmidrule(lr){2-3}
% \cmidrule(lr){4-5}
% Direct & 0.139 & (0.007) & 0.140 & (0.006)\\
% IPW & 0.504 & (0.084) & 0.546 & (0.081)\\
% IPW-SN & 0.114 & (0.017) & 0.153 & (0.014)\\
% \greymidrule
% AIPW & 0.124 & (0.047) & 0.316 & (0.037)\\
% AIPW-SN & 0.110 & (0.017) & 0.147 & (0.014)\\
% TMLE & 0.328 & (0.275) & 0.637 & (0.270)\\
% C-Learner & \textbf{0.075} & (0.012) & \textbf{0.115} & (0.008)\\
% Direct    &  0.110 &   (0.004) & 0.110 &  (0.004) \\
% IPW       &  0.131 &   (0.019) & 0.140 &  (0.019) \\
% IPW-SN   &  0.006 &   (0.004) & 0.030 &  (0.003) \\
% \greymidrule
% AIPW      & -0.103 &   (0.018) & 0.120 &  (0.017) \\
% AIPW-SN  &  \textbf{0.002} &   (0.005) & 0.035 &  (0.003) \\
% TMLE      &  0.030 &   (0.004) & 0.038 &  (0.003) \\
% C-Learner &  0.008 &   (0.005) & \textbf{0.033} &  (0.003) \\
% \bottomrule
% \vspace{0.05em}
% \end{tabular}
\caption{Comparison of estimators in the CivilComments~\citep{civilcomments} semi-synthetic dataset (\Cref{section:civilcomments}) over 100 re-drawn datasets, with $l=10^{-4}$. Asymptotically optimal methods are listed beneath the horizontal divider. We highlight the best-performing method in \textbf{bold}. Estimators (besides the ``smallest bias shift'' C-Learner) are analogous to those in \Cref{section:ks}. Standard errors are displayed within parentheses to the right of the point estimate.}
    \label{tab:LLM_results}
\end{table}
In \Cref{tab:LLM_results} (a low overlap setting with $l=10^{-4})$, both variants of C-Learner for hyperparameter selection perform the \emph{best out of all methods}.
%In \Cref{tab:LLM_results}, we find that C-Learner with both variants for hyperparameter section has the best mean absolute error compared to other methods. In this setting, C-Learner beats both not only other asymptotically optimal methods, but all methods.
In \Cref{section:more_nlp_results} we also investigate settings with both low and high overlap ($l=10^{-2},10^{-3},10^{-4}$). We observe that C-Learner performs better compared to other asymptotically optimal methods in datasets with low overlap, and more comparably to other asymptotically optimal methods with higher overlap.
We also display a subset of these results in the second row of \Cref{fig:intro_bar_comparison}, with
%; there, C-Learner results are displayed where
$\lambda$ chosen to have the best validation MSE.

%For intuition, we also include
%histograms of errors from different bootstrap resamples of the dataset for C-Learner and other asymptotically optimal methods in \Cref{fig:clearner_hist}.
%In \Cref{fig:clearner_hist}, we demonstrate how C-Learner estimates are more stable compared to other asymptotically optimal methods.

\looseness=-1 These results with text covariates echo our tabular results (\Cref{section:ks}), %. %results in our tabular experiments.
suggesting C-Learner's strong performance in low-overlap extends to more complex data settings.
%, suggesting that C-Learner's reliable performance extends to settings with more complex covariates and where nuisance parameters are modeled as neural networks trained with standard stochastic optimization methods.
%As is typical for neural networks, we observe
%nontrivial variation in estimation performance across hyperparameter configurations.
%Model selection for causal inference is an active area of research~\cite{ChernozhukovDeDuFe18, Rothenhausler20, YadlowskyFlShBrWa21}.
%An in-depth empirical study on how to tune neural networks for causal estimands is a fruitful direction of future study.

%%% Local Variables:
%%% mode: latex
%%% TeX-master: "../main"
%%% End:
                 % Section 5
\section{Discussion}
\label{section:discussion}
\looseness=-1 We introduce a constrained learning framework for first-order debiasing in causal estimation and semiparametric inference.
We pose asymptotically optimal plug-in estimators as those whose nuisance parameters are solutions to an optimization problem, under
the constraint that the first-order error of the plug-in estimator with respect to the nuisance parameter estimate is zero.
This perspective encompasses versions of one-step estimation and targeting. %, two existing approaches. % to achieving asymptotically optimal estimators.

The resulting method, the C-Learner, solves this constrained problem over the
full chosen model class. Our high-dimensional analysis shows how this choice
improves on AIPW and TMLE, even though all three methods can attain the same
classical efficiency bound. With a known propensity and a correctly specified
linear outcome model, the two-fold cross-fitted C-Learner has strictly smaller
asymptotic variance under sufficient overlap for least squares, minimum-norm
interpolation, and ridge regression (\Cref{section:high_dim}). These regimes
clarify the role of the training objective. Below interpolation, satisfying
the constraint changes the fitted training predictions; above interpolation,
the model can satisfy the constraint while preserving them. In each case,
the C-Learner scales the usual first-order correction by a
problem-dependent factor that accounts for overlap, model size, and
regularization. This factor determines how much of the correction enters
the estimated population mean.

As overlap deteriorates, this adjustment automatically reduces how much
noise from the first-order correction enters the estimated mean. Even when some covariate values
are rarely observed, other observed outcomes can continue to identify the
fixed-dimensional coefficients that determine the mean of the shared outcome
model. Under our identification and moment conditions, the C-Learner then
retains the direct method's root-$n$ limit. AIPW and TMLE remain consistent
at a slower rate, but their errors grow relative to the root-$n$ scale.

We also instantiate the C-Learner with linear models, gradient boosted trees,
and neural networks on datasets with tabular and text covariates.
Empirically, its improvements over basic one-step and targeting methods are
most pronounced with limited overlap, and it performs similarly otherwise.
%. On datasets with tabular and text covariates, we observe the C-Learner performs better than other asymptotically optimal methods in settings with low overlap, and similarly in other settings.

%Building a theoretical understanding of precisely when and why C-Learner works better than one-step estimation and targeting can be a fruitful direction of future work.
\looseness=-1 The high-dimensional analysis in \Cref{section:high_dim}
assumes a correctly specified linear outcome model, a known propensity, and
continued identification of the fixed-dimensional outcome coefficients.
Extending the exact comparison to estimated propensities and misspecified
outcome models is a natural next step, especially when information cannot be
shared reliably with regions of low overlap. Another direction is to
investigate which model classes and optimization methods satisfy the
assumptions of the classical analysis (e.g. \Cref{ass:empirical_proc_mean}).
Finally,
%although we have demonstrated the C-Learner's versatility by instantiating it for diverse model classes, including linear models, gradient-boosted trees, and neural networks,
we hope this work spurs further investigation of constrained optimization for
robust estimation, including better optimization procedures, additional model
classes, and broader semiparametric estimands.
%We hope this work spurs further investigation of how constrained optimization algorithms~\citep{nocedal1999numerical} can provide more robust causal estimators. Alternative ways to instantiate C-Learner using gradient boosted regression trees include, for example, alternating direction method of multipliers (ADMM) \citep{boyd2011distributed}. However, in practice, we found that implementing ADMM in conjunction with XGBoost was enormously slower than \Cref{alg:boosting}.

%\yf{maybe we could soft a bit the when and why statement? It feels that we have no clue at all of what is happening. Maybe something like deepening the understanding of why the constrained learning framework improves numerical stability in finite sample could be a fruitful direction for future research.} \tc{the phrasing is not that C-Learner is better with lower overlap, only that we observe it to be better in settings with lower overlap. I think this might be okay.
%One way to discuss is that propensity weights can be extreme, but we incorporate these weights only as a constraint, unlike AIPW and TMLE.
%It might help to include your expressions for AIPW, TMLE, and C-Learner in the closed-form linear setting}

%%% Local Variables:
%%% mode: latex
%%% TeX-master: "../main"
%%% End:
                  % Section 6

%% ========================== Bibliography ============================ %%

\ifdefined\ORStyle
  % Print endnotes before the references (template idiom)
  \begingroup \parindent 0pt \parskip 0.0ex
  \def\enotesize{\normalsize} \theendnotes \endgroup
  \bibliographystyle{or-style/informs2014}
  \setlength{\bibsep}{.0em}
\else
  \newpage
  \bibliographystyle{abbrvnat}
  \setlength{\bibsep}{.7em}
\fi
\bibliography{bib,hongbib}

@Manual{nloptr2023,
  title = {nloptr: R Interface to NLopt Optimization Library},
  author = {Johan {R.} {K.} {M.} {D.} {C.} {E.} {H.} {Y.} {W.} {L.} {S.} {V.} {B.} {R.} {B.} {S.} {T.}},
  year = {2023},
  note = {R package version 2.0.3},
  url = {https://CRAN.R-project.org/package=nloptr}
}

@misc{shi2019adapting,
      title={Adapting Neural Networks for the Estimation of Treatment Effects}, 
      author={Claudia Shi and David M. Blei and Victor Veitch},
      year={2019},
      eprint={1906.02120},
      archivePrefix={arXiv},
      primaryClass={stat.ML}
}

@article{Bang2005DoublyRE,
  title={Doubly Robust Estimation in Missing Data and Causal Inference Models},
  author={Heejung Bang and James M. Robins},
  journal={Biometrics},
  year={2005},
  volume={61},
  url={https://api.semanticscholar.org/CorpusID:14135922}
}

@misc{ding2023course,
      title={A First Course in Causal Inference}, 
      author={Peng Ding},
      year={2023},
      eprint={2305.18793},
      archivePrefix={arXiv},
      primaryClass={stat.ME}
}

@article{distilbert,
  author       = {Victor Sanh and
                  Lysandre Debut and
                  Julien Chaumond and
                  Thomas Wolf},
  title        = {DistilBERT, a distilled version of {BERT:} smaller, faster, cheaper
                  and lighter},
  journal      = {CoRR},
  volume       = {abs/1910.01108},
  year         = {2019},
  eprinttype    = {arXiv},
  eprint       = {1910.01108},
}

@inproceedings{chernozhukov2022riesznet,
  title={Riesznet and forestriesz: Automatic debiased machine learning with neural nets and random forests},
  author={Chernozhukov, Victor and Newey, Whitney and Quintas-Mart{\'i}nez, V{\'i}ctor M and Syrgkanis, Vasilis},
  booktitle={International Conference on Machine Learning},
  pages={3901--3914},
  year={2022},
  organization={PMLR}
}

@misc{civilcomments,
    author = {cjadams AND Daniel Borkan AND inversion AND Jeffrey Sorensen AND Lucas Dixon AND Lucy Vasserman AND nithum},
    title = {Jigsaw Unintended Bias in Toxicity Classification},
    publisher = {Kaggle},
    year = {2019},
}

@article{Nabi2024StatisticalLF,
  title={Statistical learning for constrained functional parameters in infinite-dimensional models with applications in fair machine learning},
  author={Razieh Nabi and Nima S. Hejazi and Mark J. van der Laan and David C. Benkeser},
  journal={ArXiv},
  year={2024},
  volume={abs/2404.09847},
  url={https://api.semanticscholar.org/CorpusID:269149113}
}

@article{stone1982optimal,
  title={Optimal global rates of convergence for nonparametric regression},
  author={Stone, Charles J},
  journal={The annals of statistics},
  pages={1040--1053},
  year={1982},
  publisher={JSTOR}
}

@article{robins2007comment,
  title={Comment: Performance of double-robust estimators when" inverse probability" weights are highly variable},
  author={Robins, James and Sued, Mariela and Lei-Gomez, Quanhong and Rotnitzky, Andrea},
  journal={Statistical Science},
  volume={22},
  number={4},
  pages={544--559},
  year={2007},
  publisher={JSTOR}
}

@article{van2006targeted,
  title={Targeted maximum likelihood learning},
  author={van der Laan, Mark and Rubin, Daniel},
  journal={The international journal of biostatistics},
  volume={2},
  number={1},
  year={2006},
  publisher={De Gruyter}
}

@article{athey2018approximate,
  title={Approximate residual balancing: debiased inference of average treatment effects in high dimensions},
  author={Athey, Susan and Imbens, Guido W and Wager, Stefan},
  journal={Journal of the Royal Statistical Society Series B: Statistical Methodology},
  volume={80},
  number={4},
  pages={597--623},
  year={2018},
  publisher={Oxford University Press}
}

@inproceedings{oprescu2019orthogonal,
  title={Orthogonal random forest for causal inference},
  author={Oprescu, Miruna and Syrgkanis, Vasilis and Wu, Zhiwei Steven},
  booktitle={International Conference on Machine Learning},
  pages={4932--4941},
  year={2019},
  organization={PMLR}
}

@article{athey2016recursive,
  title={Recursive partitioning for heterogeneous causal effects},
  author={Athey, Susan and Imbens, Guido},
  journal={Proceedings of the National Academy of Sciences},
  volume={113},
  number={27},
  pages={7353--7360},
  year={2016},
  publisher={National Acad Sciences}
}

@article{wager2018estimation,
  title={Estimation and inference of heterogeneous treatment effects using random forests},
  author={Wager, Stefan and Athey, Susan},
  journal={Journal of the American Statistical Association},
  volume={113},
  number={523},
  pages={1228--1242},
  year={2018},
  publisher={Taylor \& Francis}
}

@article{athey2019generalized,
  title={Generalized random forests},
  author={Athey, Susan and Tibshirani, Julie and Wager, Stefan},
  year={2019}
}

@article{hill2011bayesian,
  title={Bayesian nonparametric modeling for causal inference},
  author={Hill, Jennifer L},
  journal={Journal of Computational and Graphical Statistics},
  volume={20},
  number={1},
  pages={217--240},
  year={2011},
  publisher={Taylor \& Francis}
}

@article{cole2008constructing,
  title={Constructing inverse probability weights for marginal structural models},
  author={Cole, Stephen R and Hern{\'a}n, Miguel A},
  journal={American journal of epidemiology},
  volume={168},
  number={6},
  pages={656--664},
  year={2008},
  publisher={Oxford University Press}
}

@article{friedman2001greedy,
  title={Greedy function approximation: a gradient boosting machine},
  author={Friedman, Jerome H},
  journal={Annals of statistics},
  pages={1189--1232},
  year={2001},
  publisher={JSTOR}
}

@article{brooks1992effects,
  title={Effects of early intervention on cognitive function of low birth weight preterm infants},
  author={Brooks-Gunn, Jeanne and Liaw, Fong-ruey and Klebanov, Pamela Kato},
  journal={The Journal of pediatrics},
  volume={120},
  number={3},
  pages={350--359},
  year={1992},
  publisher={Elsevier}
}

@inproceedings{xgb,
 author = {Chen, Tianqi and Guestrin, Carlos},
 title = {{XGBoost}: A Scalable Tree Boosting System},
 booktitle = {Proceedings of the 22nd ACM SIGKDD International Conference on Knowledge Discovery and Data Mining},
 series = {KDD '16},
 year = {2016},
 pages = {785--794},
 numpages = {10},
 publisher = {ACM},
}

@misc{hahn2019bayesian,
      title={Bayesian regression tree models for causal inference: regularization, confounding, and heterogeneous effects}, 
      author={P. Richard Hahn and Jared S. Murray and Carlos Carvalho},
      year={2019},
      eprint={1706.09523},
      archivePrefix={arXiv},
      primaryClass={stat.ME}
}

@article{setodji2017right,
  title={The right tool for the job: choosing between covariate-balancing and generalized boosted model propensity scores},
  author={Setodji, Claude M and McCaffrey, Daniel F and Burgette, Lane F and Almirall, Daniel and Griffin, Beth Ann},
  journal={Epidemiology},
  volume={28},
  number={6},
  pages={802--811},
  year={2017},
  publisher={LWW}
}

@article{li18,
    author = {Li, Fan and Thomas, Laine E and Li, Fan},
    title = "{Addressing Extreme Propensity Scores via the Overlap Weights}",
    journal = {American Journal of Epidemiology},
    volume = {188},
    number = {1},
    pages = {250-257},
    year = {2018},
    month = {09},
}

@article{rolling2013,
    author = {Rolling, Craig A. and Yang, Yuhong},
    title = "{Model Selection for Estimating Treatment Effects}",
    journal = {Journal of the Royal Statistical Society Series B: Statistical Methodology},
    volume = {76},
    number = {4},
    pages = {749-769},
    year = {2013},
    month = {11},
}

@article{pfanzagl1985contributions,
  title={Contributions to a general asymptotic statistical theory},
  author={Pfanzagl, Johann and Wefelmeyer, Wolfgang},
  journal={Statistics \& Risk Modeling},
  volume={3},
  number={3-4},
  pages={379--388},
  year={1985},
  publisher={OLDENBOURG WISSENSCHAFTSVERLAG}
}

@article{gruber2012tmle,
  title={tmle: an R package for targeted maximum likelihood estimation},
  author={Gruber, Susan and van der Laan, Mark},
  journal={Journal of Statistical Software},
  volume={51},
  pages={1--35},
  year={2012}
}

@article{van2011propensity,
  title={Propensity-score-based estimators and C-TMLE},
  author={van der Laan, Mark and Rose, Sherri and Sekhon, Jasjeet S and Gruber, Susan and Porter, Kristin E and van der Laan, Mark J},
  journal={Targeted Learning: Causal Inference for Observational and Experimental Data},
  pages={343--364},
  year={2011},
  publisher={Springer}
}

@article{imai2014covariate,
  title={Covariate balancing propensity score},
  author={Imai, Kosuke and Ratkovic, Marc},
  journal={Journal of the Royal Statistical Society Series B: Statistical Methodology},
  volume={76},
  number={1},
  pages={243--263},
  year={2014},
  publisher={Oxford University Press}
}

@article{zhao2019covariate,
  title={Covariate balancing propensity score by tailored loss functions},
  author={Zhao, Qingyuan},
  year={2019}
}

@article{tan2010bounded,
  title={Bounded, efficient and doubly robust estimation with inverse weighting},
  author={Tan, Zhiqiang},
  journal={Biometrika},
  volume={97},
  number={3},
  pages={661--682},
  year={2010},
  publisher={Oxford University Press}
}

@book{baisilverstein2010,
  title={Spectral Analysis of Large Dimensional Random Matrices},
  author={Bai, Zhidong D. and Silverstein, Jack W.},
  edition={Second},
  publisher={Springer},
  year={2010}
}

@article{belkinetal2020,
  title={Two models of double descent for weak features},
  author={Belkin, Mikhail and Hsu, Daniel and Xu, Ji},
  journal={SIAM Journal on Mathematics of Data Science},
  volume={2},
  number={4},
  pages={1167--1180},
  year={2020}
}

@article{breimanfreedman1983,
  title={How many variables should be entered in a regression equation?},
  author={Breiman, Leo and Freedman, David},
  journal={Journal of the American Statistical Association},
  volume={78},
  number={381},
  pages={131--136},
  year={1983}
}

@article{celentanowainwright2023,
  title={Challenges of the inconsistency regime: Novel debiasing methods for missing data models},
  author={Celentano, Michael and Wainwright, Martin J.},
  journal={arXiv:2309.01362 [math.ST]},
  year={2023}
}

@article{chenetal2024,
  title={Method-of-Moments Inference for {GLMs} and Doubly Robust Functionals under Proportional Asymptotics},
  author={Chen, Xingyu and Liu, Lin and Mukherjee, Rajarshi},
  journal={arXiv:2408.06103 [math.ST]},
  year={2024}
}

@article{dicker2016,
  title={Ridge regression and asymptotic minimax estimation over spheres of growing dimension},
  author={Dicker, Lee H.},
  journal={Bernoulli},
  volume={22},
  number={1},
  pages={1--37},
  year={2016}
}

@article{dobribanwager2018,
  title={High-dimensional asymptotics of prediction: Ridge regression and classification},
  author={Dobriban, Edgar and Wager, Stefan},
  journal={Annals of Statistics},
  volume={46},
  number={1},
  pages={247--279},
  year={2018}
}

@article{hastieetal2022,
  title={Surprises in high-dimensional ridgeless least squares interpolation},
  author={Hastie, Trevor and Montanari, Andrea and Rosset, Saharon and Tibshirani, Ryan J.},
  journal={Annals of Statistics},
  volume={50},
  number={2},
  pages={949--986},
  year={2022}
}

@article{jiangetal2025,
  title={A new central limit theorem for the augmented {IPW} estimator: Variance inflation, cross-fit covariance and beyond},
  author={Jiang, Kuanhao and Mukherjee, Rajarshi and Sen, Subhabrata and Sur, Pragya},
  journal={Annals of Statistics},
  volume={53},
  number={2},
  pages={647--675},
  year={2025}
}

@inproceedings{mcgrathetal2025,
  title={Optimal Nuisance Function Tuning for Estimating a Doubly Robust Functional under Proportional Asymptotics},
  author={McGrath, Sean and Mukherjee, Debarghya and Mukherjee, Rajarshi and Wang, Zixiao Jolene},
  booktitle={Advances in Neural Information Processing Systems},
  volume={38},
  pages={90233--90305},
  year={2025}
}

@book{muirhead1982,
  title={Aspects of Multivariate Statistical Theory},
  author={Muirhead, Robb J.},
  publisher={Wiley},
  year={1982}
}

@article{rubiomestre2011,
  title={Spectral convergence for a general class of random matrices},
  author={Rubio, Francisco and Mestre, Xavier},
  journal={Statistics \& Probability Letters},
  volume={81},
  number={5},
  pages={592--602},
  year={2011}
}

@article{yadlowsky2022,
  title={Explaining practical differences between treatment effect estimators with high-dimensional asymptotics},
  author={Yadlowsky, Steve},
  journal={arXiv:2203.12538 [math.ST]},
  year={2022}
}

@article{mawang2020,
  title={Robust Inference Using Inverse Probability Weighting},
  author={Ma, Xinwei and Wang, Jingshen},
  journal={Journal of the American Statistical Association},
  volume={115},
  number={532},
  pages={1851--1860},
  year={2020},
  doi={10.1080/01621459.2019.1660173}
}

@article{heilerkazak2021,
  title={Valid Inference for Treatment Effect Parameters under Irregular Identification and Many Extreme Propensity Scores},
  author={Heiler, Phillip and Kazak, Ekaterina},
  journal={Journal of Econometrics},
  volume={222},
  number={2},
  pages={1083--1108},
  year={2021},
  doi={10.1016/j.jeconom.2020.03.025}
}

@book{meckes2019classical,
  title = {The Random Matrix Theory of the Classical Compact Groups},
  author = {Meckes, Elizabeth S.},
  year = {2019},
  publisher = {Cambridge University Press},
  doi = {10.1017/9781108303453}
}

@string{acm = {Association for Computing Machinery}}

@string{wiley=	{John Wiley \& Sons}}

@book{BickelKlRiWe98,
author = {Peter Bickel and C. A. J. Klaassen and Yaacov Ritov and Jon Wellner},
title = {Efficient and Adaptive Estimation for Semiparametric Models},
year = 1998,
publisher = {Springer Verlag},
}

@article{ChernozhukovChDeDuHaNeRo18,
  title={Double/debiased machine learning for treatment and structural parameters},
  author={Chernozhukov, Victor and Chetverikov, Denis and Demirer, Mert and Duflo, Esther and Hansen, Christian and Newey, Whitney and Robins, James},
  journal={The Econometrics Journal},
  volume=21,
  number=1,
  pages={C1--C68},
  year=2018,
  publisher={Wiley Online Library}
}

@techreport{CrumpHoImMi06,
  title={Moving the goalposts: Addressing limited overlap in the estimation of average treatment effects by changing the estimand},
  author={Crump, Richard K and Hotz, V Joseph and Imbens, Guido W and Mitnik, Oscar A},
  year=2006,
  institution={National Bureau of Economic Research}
}

@article{Hahn98,
  title={On the role of the propensity score in efficient semiparametric estimation of average treatment effects},
  author={Hahn, Jinyong},
  journal={Econometrica},
  pages={315--331},
  year=1998,
  publisher={JSTOR}
}

@article{KangSc07,
author = {Joseph D. Y. Kang and Joseph L. Schafer},
title = {{Demystifying Double Robustness: A Comparison of Alternative Strategies for Estimating a Population Mean from Incomplete Data}},
volume = {22},
journal = {Statistical Science},
number = {4},
publisher = {Institute of Mathematical Statistics},
pages = {523 -- 539}, 
year = {2007}, 
}

@article{Kennedy22,
  title={Semiparametric doubly robust targeted double machine learning: a review},
  author={Kennedy, Edward H},
  journal={arXiv:2203.06469 [stat.ME]},
  year=2022
}

@article{Newey94,
  title={The asymptotic variance of semiparametric estimators},
  author={Newey, Whitney K},
  journal={Econometrica},
  pages={1349--1382},
  year=1994,
  publisher={JSTOR}
}

@article{Newey97,
  title={Convergence rates and asymptotic normality for series estimators},
  author={Newey, Whitney K},
  journal={Journal of Econometrics},
  volume=79,
  number=1,
  pages={147--168},
  year=1997,
  publisher={Elsevier}
}

@book{Tsiatis07,
  title={Semiparametric theory and missing data},
  author={Tsiatis, Anastasios},
  year=2007,
  publisher={Springer Science \& Business Media}
}

@book{VanDerLaanRo11,
  title={Targeted learning: causal inference for observational and experimental data},
  author={Van der Laan, Mark J and Rose, Sherri},
  volume=10,
  year=2011,
  publisher={Springer}
}

%% =========================== Appendix =============================== %%

\ifdefined\ORStyle
  \ECSwitch
  % ECSwitch resets counters without updating hyperref's destination names.
  % Give the e-companion its own namespace so links cannot hit main-text items.
  \renewcommand{\theHsection}{EC.\arabic{section}}
  \renewcommand{\theHsubsection}{\theHsection.\arabic{subsection}}
  \renewcommand{\theHsubsubsection}{\theHsubsection.\arabic{subsubsection}}
  \newcommand{\ECuniqueanchor}[1]{%
    \expandafter\def\csname theH#1\endcsname{\theHsection.\arabic{#1}}}
  \forcsvlist{\ECuniqueanchor}{equation,theorem,lemma,proposition,corollary,claim,
    conjecture,hypothesis,assumption,remark,example,problem,definition,
    question,answer,exercise,table,figure}
  %\ECDisclaimer
  \ECHead{Appendix}
\else
  \newpage
  \appendix
\fi

% Read the automatically generated headings, but only after the appendix marker.
% Section numbers, titles, page numbers, and hyperlinks follow both style modes.
\addtocontents{toc}{\protect\appendixcontentsstart}
\pdfbookmark[0]{Appendix Contents}{appendix-contents}
\section*{Appendix Contents}
\par\medskip
\begingroup
\hypersetup{pdfborder={0 0 0}}
\makeatletter
\setcounter{tocdepth}{3}
\small
\baselineskip=12pt
\parskip=0pt
\def\@pnumwidth{3em}
\def\@tocrmarg{4em}
\let\appendixcontentsline\contentsline
\long\def\contentsline#1#2#3#4{}%
\renewcommand*\l@section[2]{%
  \addvspace{4pt}%
  {\bfseries\@dottedtocline{1}{0em}{3.4em}{#1}{#2}}}
\renewcommand*\l@subsection{\@dottedtocline{2}{1.2em}{4.4em}}
\renewcommand*\l@subsubsection{\@dottedtocline{3}{2.4em}{5.4em}}
\@starttoc{toc}
\makeatother
\endgroup
\clearpage

%% Appendix sections in main-body order: Section 3 material (other
%% approaches, implementations, classical asymptotics), then the proof of
%% Section 4, the finite-sample low-overlap comparison, and experiment details.
\section{Instantiating the C-Learner}
\label{section:methodology}

The constrained learning framework can be instantiated in many ways depending
on the function class $\mc{F}$ for outcome models.  We present approximate
solution methods to the constrained optimization problem~\eqref{eq:c-learner}
for linear models, gradient boosted regression trees, and deep neural
networks. We empirically demonstrate these instantiations in
\Cref{section:experiments}.\footnote{Other
  frameworks~\citep{Nabi2024StatisticalLF} exist for constructing constrained
  outcome models, but do not restrict outcome models to $\mathcal F$; we
  exclude these from our analysis.}

First, we give the explicit C-Learner for linear models.  Let
$b(X):=(1,X^\top)^\top$, let $\Lambda\succeq0$ denote a quadratic penalty, and
define $\what H:=\what P_{\rm train}[Ab(X)b(X)^\top]+\Lambda$,
$c:=\what P_{\rm eval}[Ab(X)/\what\pi(X)]$, and
$\what\theta:=\what H^{-1}\what P_{\rm train}[Ab(X)Y]$.  Here, $\Lambda=0$ gives OLS,
while $\Lambda=\operatorname{diag}(0,\lambda\Id_d)$ gives ridge regression
with an unpenalized intercept. As we derive in
\Cref{section:proof-of-reduction}, if $\what H$ is positive definite and $c\neq0$,
\begin{equation*}
\what\theta^C
=\what\theta+\what H^{-1}c\,
\frac{
\what P_{\rm eval}\left[
\dfrac{A}{\what\pi(X)}(Y-b(X)^\top\what\theta)
\right]
}{c^\top \what H^{-1}c}.
\end{equation*}

\begin{yfrevision}\yf{Outside of linear functions, we use approximate
optimization methods for flexible ML models that we describe next.}\end{yfrevision}

\subsection{Gradient Boosted Regression Trees}
\label{section:methods_boosting}
\begin{yfrevision}
\yf{We use gradient boosted regression trees~\citep{friedman2001greedy,xgb}
to fit an initial outcome model and an approximate projection of the
inverse propensity, then calibrate the outcome model along the fitted
direction. This projected interpretation of the C-Learner is an heuristic to implement the constrained
optimization problem and can be solved in closed form for linear models as presented above. As before, $\mu(X)$ denotes $\mu(1,X)$ and $\mu(0,X)=0$.
On each training fold, fit $\what\mu$ by squared loss on observed outcomes
and fit $\what\pi$ using $(X,A)$. All tuning of these initial nuisance
models takes place within that training fold.}

\yf{On the evaluation fold, set $\what h(X)=1/\what\pi(X)$, using the
same prespecified propensity floor as the other estimators, if any.
A separate squared-loss XGBoost regression approximates}
\begin{equation}
\yf{\what q\ \approx\
\argmin_{q\in\mathcal F_{\rm XGB}}
\what P_{\rm eval}[A\{\what h(X)-q(X)\}^2].}
\label{eq:boosting-projection}
\end{equation}
\yf{This fit uses evaluation covariates and observation indicators, but
not evaluation outcomes. Its trees are fitted separately from the
initial outcome ensemble. Projection hyperparameters and the tree count are selected by
inverse-propensity prediction error on an internal validation split. 
The selected model is refitted on all its observed
rows, then evaluated at every covariate in the fold.}

\yf{Using the observed evaluation outcomes, set}
\begin{equation}
\yf{\what\epsilon
 =\frac{\what P_{\rm eval}[A\what h(X)\{Y-\what\mu(X)\}]}
        {\what P_{\rm eval}[A\what h(X)\what q(X)]},
 \qquad
 \what\mu^C(X)=\what\mu(X)+\what\epsilon\what q(X).}
\label{eq:boosting-projection-calibration}
\end{equation}
\yf{When the denominator is finite and nonzero, this update enforces
$\what P_{\rm eval}[A\what h(X)\{Y-\what\mu^C(X)\}]=0$.
The estimate is the plug-in $\what P_{\rm eval}[\what\mu^C(X)]$;
there is no additional one-step correction. This is an interesting distinction from squared-loss TMLE, which instead
uses $\what h$ itself as the direction.}

\begin{algorithm}[t]
\caption{\yf{C-Learner with Gradient Boosted Regression Trees}}
\label{alg:boosting}
\yf{\small
\begin{algorithmic}[1]
\State \textbf{Input:} $(X_i,A_i,A_iY_i)_{i=1}^n$ and two folds $I_1,I_2$.
\For{$k=1,2$}
 \State Let $\what P_{\rm train}$ average over $I_{-k}$ and
        $\what P_{\rm eval}$ over $I_k$.
 \State Tune $\what\pi$ on training-fold $(X,A)$; refit on the full training fold.
 \State Tune the squared-loss outcome regression on training-fold
        observed outcomes; refit on all observed training rows to obtain $\what\mu$.
 \State Apply any common propensity floor and compute $\what h=1/\what\pi$
        on the evaluation fold.
 \State Tune the regression of $\what h(X)$ on $X$ using evaluation-fold
        rows with $A=1$, without using $Y$; refit on all these rows to obtain $\what q$.
 \State Evaluate $\what\mu$ and $\what q$ on every $X$ in $I_k$.
 \State Compute $\what\epsilon$ and $\what\mu^C$ by
        \eqref{eq:boosting-projection-calibration}.
 \State Set $\what\psi_k=\what P_{\rm eval}[\what\mu^C(X)]$.
\EndFor
\State \textbf{Return} $\displaystyle\sum_{k=1}^2\frac{|I_k|}{n}\what\psi_k$.
\end{algorithmic}}
\end{algorithm}

\yf{\Cref{alg:boosting} cross-fits the initial outcome and propensity
estimates; the projection, fluctuation, and constraint use the evaluation
fold. It enforces feasibility along a learned direction, and is an heuristic for solving the optimization problem over boosted-tree ensembles.}

\yf{For a full ATE, fit separate treated and control outcome models,
apply the same projection and calibration steps within each arm using
inverse weights $1/\what\pi$ and $1/(1-\what\pi)$, respectively, and
subtract the control plug-in from the treated plug-in before averaging
folds. The KS settings are given in \Cref{section:ks-xgb-projection};
\Cref{section:ihdp_xgb} gives the armwise IHDP implementation.}
\end{yfrevision}
\FloatBarrier

% \tc{Constrained outcome models implemented using gradient boosting have also been proposed by
% \citet{Schuler2017SynthValidationST}; there, however, these outcome models are used to construct semi-synthetic datasets with a fixed (constrained) average treatment effect, where the semi-sythethetic datasets also resemble an existing dataset, in order to evaluate various causal estimators on that existing dataset. In contrast, C-Learner uses constrained learning to ensure asymptotic properties in our estimator.}

% \begin{algorithm}[t]
% \caption{C-Learner with neural networks}
% \begin{algorithmic}[1]
% \State Set learning rate $\eta$, regularization parameter $\lambda$, number of epochs $N$
% \State Initialize neural network weights $\theta$
% \For{epoch in $1, 2, \dots, N$}
%     \For{each batch of training examples $\what{P}_{\rm train}^{\rm batch}\subset \what{P}_{\rm train}$}
%         \State Compute gradient of combined loss with respect to weights:
%         $$\nabla_\theta L_{\rm total}^{\rm batch}(\theta):=\nabla_\theta \left(\what{P}_{\rm train}^{\rm batch}[\ell_{\rm MSE}(\theta; X,A,Y)] + L_{\rm regularizer}(\theta)\right)$$
%         \State Update weights using SGD: $\theta \gets \theta - \eta \nabla_\theta L_{\rm total}^{\rm batch}(\theta)$
%         \State \quad\quad (or use another optimizer, e.g. Adam)
%     \EndFor
%     \State Adjust the constant bias term $\theta_{\rm bias}$ to satisfy constraint exactly, using Equation~\eqref{eq:const_shift}
% \EndFor
% \end{algorithmic}
% \label{alg:nn}
% \end{algorithm}

\subsection{Neural Networks}
\label{section:methods_nn}
When outcome models are
neural networks $\what\mu_\theta(x)$
with weights $\theta$,
we consider the usual MSE loss with a Lagrangian regularizer for the constraint~\eqref{eq:c-learner}
\begin{align*}
    L(\theta) = \what{P}_{\rm train}\left[A(Y-\what\mu_\theta(X))^2\right] +
\lambda\cdot \what{P}_{\rm eval}\left[\frac{A}{\what\pi(X)}(Y-\what\mu_\theta(X))\right]^2.
\end{align*}
\looseness=-1 We optimize the objective using stochastic gradient methods, where we take mini-batches of the training set to approximate the gradient of the first term and take a full-batch gradient on the evaluation set for the second term. At the end of every training epoch, the constraint on $\what{P}_{\rm eval}$ is enforced \emph{exactly} by
adjusting the constant bias term $\theta_{\rm bias}$ in the neural network:
\begin{align}
\label{eq:const_shift}
    \theta_{\rm bias} \gets \theta_{\rm bias} +
    \left( \what{P}_{\rm eval}\left[\frac{A}{\what\pi(X)}\right] \right)^{-1} \what{P}_{\rm eval}\left[\frac{A}{\what\pi(X)}(Y-\what\mu_\theta(X))\right].
\end{align}
\looseness=-1 We choose to stop training at the epoch that minimizes $\what{P}_{\rm val}  \left[ A(Y-\what\mu_\theta(X))^2\right]$, the MSE loss on the validation split. We consider two options for choosing  hyperparameters (e.g. learning rate, $\lambda$).
The first option is to minimize MSE on $\what{P}_{\rm val}$. Second,
 since a small bias shift indicates a regularizer that is successful,
 we choose
  among hyperparameters with reasonable MSE loss on $\what{P}_{\rm val}$ the one that minimizes the size of the bias shift \eqref{eq:const_shift} in the first epoch.
While the first method is more standard, the second encourages satisfying the constraint over the course of training, to avoid big jumps in optimization.
Model selection for nuisance parameters is an active area of research \citep{hahn2019bayesian,rolling2013,setodji2017right}; we explore these in \Cref{section:civilcomments}.
%Additionally, we note that the methods for linear models can often be used for neural networks as there is often a last linear layer over a feature embedding; during training,

\subsection{A ``Dual'' C-Learner and Connections to Covariate Balancing}
\label{section:balance}
\Cref{eq:c-learner} starts by fitting and fixing a propensity score $\what\pi$, and then finding the best fitting outcome model $\what \mu^C$  subject to a constraint that depends on $\what\pi$ to ensure efficiency of the direct plug-in estimator that uses $\what \mu^C$. One could alternatively first fit an outcome model $\what \mu$, and then find the best fitting propensity model $\what \pi^C$ subject to a constraint to ensure efficiency of the resulting estimator. We call this the ``dual'' C-Learner:
\begin{align}
\label{eq:c_learner_dual}
    &\what \pi^C\in \argmax_{\pi\in \mathcal F_\pi} \left\{
\what{P}_{\rm train}[ A\log \pi(X) +(1-A)\log (1-\pi(X))] :
\frac{1}{n}\sum_{i=1}^n \left(1-\frac{A_i}{\pi(X_i)}\right) \what\mu(X_i)=0
    \right\},\nonumber
    \\
    &\what\psi^\text{C-Learner-Dual}:=\frac{1}{n}\sum_{i=1}^n \frac{A_i}{\what \pi^C(X_i)}Y_i
.\end{align}
%where the constraint is chosen analogously to the constraint for the usual C-Learner \Cref{eq:c-learner};
While this constraint is algebraically distinct from the constraint for the usual C-Learner~\eqref{eq:c-learner}, both are chosen by the same principle: enforce a condition that eliminates the first-order bias term. In the usual C-Learner, the constraint sets the projected canonical gradient $P[\varphi_X(Z;\hat P)]$ to zero. In the dual C-Learner, we instead impose a balancing condition that plays an analogous role, which sets $P[\varphi(Z;\hat P)]$ to zero without relying on the projected canonical gradient $\varphi_X$.

% While this constraint is algebraically distinct from the constraint for the usual C-Learner \eqref{eq:c-learner}, the constraints for the dual and usual C-Learners are chosen analogously as they both lead to semiparametric efficiency. More specifically, while the usual C-Learner attempts to set $P[\varphi(Z;\what P)]=0$ by setting
% $P[\varphi_X(Z;\what P)]=0$, the dual C-Learner approaches $P[\varphi(Z;\what P)]=0$ without using $\varphi_X$. See details in
%\Cref{section:asymptotics_dual}. %\textcolor{blue}{Maybe we could include the self-normalized estimator too, or simply mention it.}

\looseness=-1 Other parts of this paper focus on the C-Learner in
\Cref{eq:c-learner}; we include the dual C-Learner here to relate to existing
literature, and because we evaluate it empirically in \Cref{section:ks_ipw}.
%, but include results for this ``dual'' C-Learner as well.
The dual C-Learner formulation is connected to covariate balancing propensity scores (CBPS)~\cite{imai2014covariate}. Specifically,
%in the following way: %recall th
%The constraint in \Cref{eq:c-learner} can be thought of as a balancing condition.
%Many readers may be familiar with covariate balancing propensity scores (CBPS)~\cite{imai2014covariate,zhao2019covariate}. While the ``balancing conditions'' in C-Learner and CBPS are different in formulation, a specific covariate balancing condition results in ``dual'' version of a C-Learner, that also achieves semiparametric efficiency in largely the same way.
%\citep{zhao2019covariate}, and (2) covariate balancing can lead to semiparametric efficiency, though under different assumptions and using different arguments from ours \citep{tan2010bounded}.
%We begin by comparing the balancing conditions between CBPS and C-Learner. While the covariate balancing condition only balances covariates (and functions of the covariates) between treatment and control, the C-Learner constraint balances \emph{modeled outcomes versus observed outcomes}. We provide the covariate balancing condition and compare with the C-Learner constraint for reference.
the covariate balancing condition from \citep{imai2014covariate} for estimating the (full) ATE can be written as
\begin{align}
\label{eq:imai_cbps}
\E\left[\frac{A f(X)}{ \pi\opt\left(X\right)}-\frac{\left(1-A\right) f(X)}{1-\pi\opt\left(X\right)}\right]=0
%\vspace{-0.9em}
\end{align}
where $f(X)$ can be any function of $X$. For example, one practical choice could be $f(X)=X$ so that ensuring the condition in \eqref{eq:imai_cbps} results in balancing each covariate dimension. %In theory, the condition \eqref{eq:imai_cbps} should hold for any choice of $f$; in practice, one practical choice could be $f(X)=X$ so that ensuring the condition in \eqref{eq:imai_cbps} results in balancing each covariate dimension.

As our paper focuses on the mean missing outcome setting (equivalent to the ATE with $Y(0):=0$ for cleaner notation),
we focus on a corresponding balancing condition, i.e. \citep{tan2010bounded}
\begin{align}
\label{eq:cbps_mmo}
\E\left[\left( 1 - \frac{A}{ \pi\opt(X)}\right) f(X) \right]=0,
\end{align}
again for any choice of $f$; note that if \eqref{eq:cbps_mmo} holds for both $A$ and $1-A$, i.e. that
\begin{align}
\E[f(X)]=\E\left[\frac{A}{ \pi\opt(X)} f(X) \right]
\quad\text{ and }\quad
\E[f(X)]=\E\left[\frac{1-A}{ 1-\pi\opt(X)} f(X) \right]
\end{align}
then \eqref{eq:imai_cbps} must also hold. A finite-sample counterpart of  \eqref{eq:cbps_mmo} can also be written \citep{tan2010bounded} as
\begin{align}
\label{eq:cbps_mmo_finite}
\frac{1}{n}\sum_{i=1}^n \left(1-\frac{A_i}{\what\pi(X_i)}\right) f(X) = 0
.\end{align}
%Contrast this with the C-Learner constraint in \eqref{eq:c-learner}, which we repeat here:
%$$
%\frac{1}{n} \sum_{i=1}^n
%\frac{A_i}{\what{\pi}(X_i)}
%(Y_i-\what\mu(X_i))=0;
%$$
%Observe that the covariate balancing condition~\eqref{eq:cbps_mmo_finite} only ensures a balance between inverse propensity weighted functions of covariates, versus unweighted functions of covariates. In contrast, the C-Learner constraint ensures a balance between observed outcomes and modeled outcomes, when both are weighted by inverse propensity scores.
Notably, this condition is exactly the constraint in \Cref{eq:c_learner_dual} when we take $f$ to be $\what\mu$.
CBPS can then be used to enforce this constraint if, for example, $\mu\opt$ is linear in some basis of functions $\{f_1,\ldots,f_J\}$, and balance is enforced for all elements in this basis. In this way, covariate balancing can result in semiparametric efficiency.

%Notably, the estimand in the balancing condition \eqref{eq:cbps_mmo} when $f$ is taken to be $\mu$ is the efficient influence function projected onto the tangent space of $\pi$. A ``dual'' version of C-Learner could be implemented in which we fix $\what \mu$ and learn $\what \pi$ under the constraint that the sample estimate of this projected efficient influence function is 0.

\paragraph{Related works}
\citet{tan2010bounded} produces semiparametrically efficient IPW-based estimators, similar to a dual C-Learner, but their methods are analogous to targeting, as they learn an extended propensity score that is a parametric fluctuation from an initial propensity score model, where the magnitude of the fluctuation is determined by maximizing likelihood. %, and where solutions remove the finite-sample estimate of the first-order error term, very similar to a ``dual'' C-Learner (discussed in more detail in \Cref{section:asymptotics_dual}).
Unlike targeting, their estimators have additional constraints to ensure boundedness of their estimator, improved local efficiency (better variance for a given fixed estimated outcome model; this is outside of the scope of our work), and double robustness. In contrast to our setting, their theoretical results focus on parametric nuisance models and notions of efficiency when certain nuisance parameters are well-specified or not. Similar to our proposed method, their estimators also achieve the semiparametric efficiency bound.

\citet{zhao2019covariate} proposes a loss function to minimize for learning a nonparametric propensity model, where the first-order conditions of the loss are the balancing conditions in CBPS. % (and relaxations thereof in overdetermined settings where there are more balancing constraints than model parameters).
They also show their corresponding inverse propensity weighted estimators are efficient, but under different assumptions:
%while we assume standard non-parametric rates of convergence for nuisance parameters, %\tc{todo what are reasonable cites} % in order to show that terms in the first-order distributional Taylor expansion are negligible (\Cref{section:theory_text}),
\citet{zhao2019covariate} shows efficiency using a sieve approach with a growing basis of functions %feature set (e.g. $\{f_1,\ldots, f_J\}$ from earlier)
to model the propensity score through a logistic link,
while we assume standard non-parametric rates of convergence for nuisance parameters. %, which can be more directly applicable to general machine learning methods used in practice.
%While sieve approaches are often used to model nonparametric nuisance parameters,
%we assume standard non-parametric rates of convergence for nuisance parameters, which applies more directly to the more general machine learning methods used in practice.
%While sieve approaches are often used to model nonparametric nuisance parameters, the rate of growth of the number of features is assumed to be quite slow in the number of data points $n$, where the number of features grows more slowly than $n^{1/9}$ \citep{zhao2019covariate}. \tc{it's possible this proof is weaker than it has to be, but the proof is very similar to the Hirano/Imbens/Ridder proof where they show efficiency for IPW with an estimated propensity score -- no extra structure needed. Zhao basically proposes different loss functions for different estimands (ATE, ATT, etc) and replicates the HIR proof but for the other loss functions.}

%Additionally, unlike our work, their proposed propensity model satisfies a version of a covariate balancing condition in Equation (5), if you take $\what h$ to be $f$ in~\eqref{eq:cbps_mmo}.
%One consequence of the difference in formulation between our work and \citep{tan2010bounded} is that their proposed estimators are quite different from what would be obtained in our setting.

%\paragraph{Dual formulation}
%\cite{tan2010bounded}
\paragraph{Low-overlap settings} In \Cref{section:theory_simple}, AIPW and
squared-loss TMLE have infinite mean squared error even with known
propensities. The C-Learner has finite mean squared error under a
square-integrable envelope, with a bounded extension when no constrained
solution is selected. This illustrates a possible benefit of controlling the fitted
outcome model. The result does not establish the behavior of covariate
balancing methods or the dual C-Learner.
%By construction, the dual C-Learner and covariate balancing approaches yield estimators that weight observed outcomes by the inverse propensity score, making them potentially vulnerable to the same instability, although we do not formally establish their variance behavior here. This illustrates how the usual C-Learner can be more stable in low-overlap regimes.
%the dual C-Learner and other covariate balancing methods.

\label{section:balance-end}

%%% Local Variables:
%%% mode: latex
%%% TeX-master: "../main"
%%% End:
             % old Sections 4.2--4.4
\section{Proof of \Cref{prop:classical-guarantees}}
\label{section:theory_appendix}
\label{section:theory_text}

This appendix gives sufficient conditions and proofs for the cross-fitted
C-Learner to be semiparametrically efficient and doubly robust. Let
$\pi\opt(x):=P(A=1\mid X=x)$,
$\mu\opt(x):=P(Y\mid A=1,X=x)$, and
$\psi(P):=P[\mu\opt(X)]$. Its efficient influence function is
$\varphi(Z;P)=\frac{A}{\pi\opt(X)}\{Y-\mu\opt(X)\}+\mu\opt(X)-\psi(P)$
\citep{Hahn98}.

\subsection{Cross-fitted C-Learner}
\label{section:data_splitting}
Fix $K$ and split the sample of size $n$ into $K$ folds of equal size. Let $\what{P}_{k,n}$
denote the empirical measure on fold $k$, and let $\what{P}_{-k,n}$ denote the
empirical measure on the remaining folds. Let $\what{P}_n$ denote the
full-sample empirical measure. We fit
$\what\pi_{-k,n}$ on $\what{P}_{-k,n}$ and define
\begin{align}
\label{eq:constraint}
\what\mu^C_{-k,n}
&\in\argmin_{\mu\in\mathcal F}
\left\{\what{P}_{-k,n}[A\{Y-\mu(X)\}^2]:
\what{P}_{k,n}\left[\frac{A}{\what\pi_{-k,n}(X)}\{Y-\mu(X)\}\right]=0
\right\},\\
\what\psi_n^C
&:=\frac{1}{K}\sum_{k=1}^K
\what{P}_{k,n}[\what\mu^C_{-k,n}(X)].
\label{eq:c_learner_def}
\end{align}
We assume that a measurable solution exists with probability tending to one
and set $\what\mu^C_{-k,n}=0$ otherwise. The constraint-based decomposition
below is used on the event that a solution exists.
The propensity model is fit without fold $k$, but the constraint uses fold
$k$. Thus, $\what\mu^C_{-k,n}$ is generally not independent of
$\what{P}_{k,n}$.

We impose the following conditions for each $k\in\{1,\ldots,K\}$.

\begin{assumption}[Overlap]
\label{ass:overlap}
For some $\eta>0$,
\[
\eta\leq\pi\opt(X)\leq1-\eta,
\qquad
\eta\leq\what\pi_{-k,n}(X)\leq1-\eta
\quad\text{a.s.}
\]
\end{assumption}

\begin{assumption}[Nuisance rates]
\label{ass:consistency}
The nuisance estimates satisfy
\begin{align*}
\|\what\pi_{-k,n}-\pi\opt\|_{L_2(P)}&=o_P(1),
&
\|\what\mu^C_{-k,n}-\mu\opt\|_{L_2(P)}&=o_P(1),\\
\|\what\pi_{-k,n}-\pi\opt\|_{L_2(P)}
\|\what\mu^C_{-k,n}-\mu\opt\|_{L_2(P)}
&=o_P(n^{-1/2}).
\end{align*}
\end{assumption}

\begin{assumption}[Outcome moments]
\label{ass:bounded_outcomes}
We have $P[\mu\opt(X)^2]<\infty$, and, for some $B<\infty$,
\[
P[\{Y-\mu\opt(X)\}^2\mid A=1,X=x]\leq B^2
\quad\text{for $P$-almost every $x$.}
\]
\end{assumption}

\begin{assumption}[Evaluation-fold stability]
\label{ass:empirical_proc_mean}
\[
(\what{P}_{k,n}-P)\left[
\left\{1-\frac{A}{\what\pi_{-k,n}(X)}\right\}
\{\what\mu^C_{-k,n}(X)-\mu\opt(X)\}
\right]
=o_P(n^{-1/2}).
\]
\end{assumption}

\Cref{ass:empirical_proc_mean} is needed because the constraint uses
fold $k$, so ordinary cross-fitting does not make the displayed term
conditionally centered. If the outcome fit were also learned without fold
$k$, the overlap and consistency conditions would imply this bound by
standard cross-fitting (e.g., \citet[Lemma~1]{Kennedy22}).
The one-dimensional classes that recover
self-normalized AIPW and squared-loss TMLE in \Cref{section:methods},
$\{\what\mu+\epsilon:\epsilon\in\mathbb R\}$ and
$\{\what\mu+\epsilon/\what\pi:\epsilon\in\mathbb R\}$, satisfy this
condition under the overlap and outcome-moment assumptions when the initial
fits are learned without fold $k$ and are $L_2(P)$-consistent. We give two examples below where familiar rate conditions that satisfy
the usual AIPW product-rate condition also imply
\Cref{ass:empirical_proc_mean}.
In both examples, we clip $\what\pi_{-k,n}$ to $[\eta,1-\eta]$ to satisfy
\Cref{ass:overlap}.

\begin{example}[Growing linear models]
Consider linear outcome and propensity models
\[
\mu\opt(x)=b_{\mu,n}(x)^\top\beta\opt_n,
\qquad
\operatorname{logit}\{\pi\opt(x)\}
=b_{\pi,n}(x)^\top\gamma\opt_n,
\]
where both bases contain an intercept and have dimension $d_n$. Suppose the
designs are uniformly sub-Gaussian and their population and empirical Gram
matrices are well conditioned. The training-fold least-squares outcome ERM
$\what{\mu}_{-k,n}$ and logistic propensity ERM $\what{\pi}_{-k,n}$ then
satisfy
\[
r_{\mu,n}
:=\|\what{\mu}_{-k,n}-\mu\opt\|_{L_2(P)}
=O_P\left(\sqrt{\frac{d_n}{n}}\right),
\qquad
r_{\pi,n}
:=\|\what{\pi}_{-k,n}-\pi\opt\|_{L_2(P)}
=O_P\left(\sqrt{\frac{d_n}{n}}\right).
\]
The usual AIPW product-rate condition is
$r_{\pi,n}r_{\mu,n}=o_P(n^{-1/2})$, so it is enough that
\begin{equation}
d_n=o(\sqrt n).
\label{eq:linear-aipw-rates}
\end{equation}

We verify \Cref{ass:empirical_proc_mean} under the same restriction.
Set $b_n=b_{\mu,n}$ and consider the full outcome class
\[
\mathcal F_n=\{x\mapsto b_n(x)^\top\beta:\beta\in\mathbb R^{d_n}\}.
\]
Write
\[
\what{\mu}_{-k,n}(x):=b_n(x)^\top\what{\beta}_{-k,n},
\qquad
\what{H}_{-k,n}:=\what{P}_{-k,n}[A b_n(X)b_n(X)^\top].
\]
By overlap, the intercept, and the empirical Gram condition,
$\what{H}_{-k,n}$ is invertible and the denominator below is positive with
probability tending to one; we condition on this event.
Since $\what{\mu}_{-k,n}$ is fitted without fold $k$, standard cross-fitting
and $r_{\mu,n}=o_P(1)$ give
\[
(\what{P}_{k,n}-P)\left[
\left\{1-\frac{A}{\what{\pi}_{-k,n}(X)}\right\}
\{\what{\mu}_{-k,n}(X)-\mu\opt(X)\}\right]
=o_P(n^{-1/2}).
\]
It remains to control the change induced by the constraint. The KKT and
primal-feasibility calculations used to prove \Cref{lem:reduction} in
\Cref{section:proof-of-reduction} give
\begin{equation}
\what{\mu}^C_{-k,n}(x)
=\what{\mu}_{-k,n}(x)
+\what{P}_{k,n}\left[
\frac{A}{\what{\pi}_{-k,n}(X)}
\{Y-\what{\mu}_{-k,n}(X)\}\right]
\what{h}_{k,n}(x),
\label{eq:linear-c-update}
\end{equation}
where
\[
\what{h}_{k,n}(x)
:=
b_n(x)^\top
\frac{
\what{H}_{-k,n}^{-1}
\what{P}_{k,n}\left[\dfrac{A}{\what{\pi}_{-k,n}(X)}b_n(X)\right]
}{
\what{P}_{k,n}\left[\dfrac{A}{\what{\pi}_{-k,n}(X)}b_n(X)\right]^\top
\what{H}_{-k,n}^{-1}
\what{P}_{k,n}\left[\dfrac{A}{\what{\pi}_{-k,n}(X)}b_n(X)\right]
}.
\]
Under the stated overlap and Gram conditions, conditional concentration gives
the first bound below. For the second, write
$\what P_{k,n}=(\what P_{k,n}-P)+P$. Conditional on the training fold,
the centered evaluation term is $O_P(n^{-1/2})$. The population term is
controlled by a standard linear-functional bound for correctly specified
least squares and the usual product-error argument. Together these give
\begin{align*}
&(\what{P}_{k,n}-P)\left[
\left\{1-\frac{A}{\what{\pi}_{-k,n}(X)}\right\}
\what{h}_{k,n}(X)\right]
=O_P\left(n^{-1/2}+\frac{d_n}{n}\right),\\
&\what{P}_{k,n}\left[
\frac{A}{\what{\pi}_{-k,n}(X)}
\{Y-\what{\mu}_{-k,n}(X)\}\right]
=O_P\{n^{-1/2}+r_{\pi,n}r_{\mu,n}\}
=O_P(n^{-1/2}).
\end{align*}
Combining these bounds with \eqref{eq:linear-c-update},
\begin{align*}
&(\what{P}_{k,n}-P)\left[
\left\{1-\frac{A}{\what{\pi}_{-k,n}(X)}\right\}
\{\what{\mu}^C_{-k,n}(X)-\what{\mu}_{-k,n}(X)\}\right] \\
&\qquad=O_P(n^{-1/2})
O_P\left(n^{-1/2}+\frac{d_n}{n}\right)
=O_P(n^{-1})
=o_P(n^{-1/2}),
\end{align*}
where the penultimate equality uses $d_n=o(\sqrt n)$. This proves
\Cref{ass:empirical_proc_mean}. The constraint direction has bounded
$L_2(P)$ norm, so the constrained fit retains the rate $r_{\mu,n}$.
\end{example}

\begin{example}[Spline and wavelet sieves]
Suppose $X\in[0,1]^p$, and both $\mu\opt$ and the logit of $\pi\opt$ are
$s$-H\"older smooth. Let $\what{\mu}_{-k,n}$ and $\what{\pi}_{-k,n}$ be the
training-fold least-squares and logistic ERMs over $d_n$-term spline or
wavelet sieves that contain an intercept. Under the usual approximation
conditions, both nuisance errors have rate \citep{stone1982optimal,Newey97}
\[
r_n=O_P\left(d_n^{-s/p}+\sqrt{\frac{d_n}{n}}\right).
\]
Taking $d_n\asymp n^{p/(2s+p)}$ gives
$r_n=O_P\{n^{-s/(2s+p)}\}$. Thus, the usual AIPW product-rate condition
$r_n^2=o_P(n^{-1/2})$ holds when $s>p/2$.

The same condition implies \Cref{ass:empirical_proc_mean}. Suppose
the splines or wavelets are compactly supported, the design density is bounded
above and away from zero, and, for $Q_n:=P[b_n(X)b_n(X)^\top]$,
$\sup_x b_n(x)^\top Q_n^{-1}b_n(x)=O(d_n)$. If $s>p/2$, then
\[
r_n=o_P(1),
\qquad
\frac{d_n\log d_n}{n}\to0,
\qquad
\frac{d_n r_n}{\sqrt n}=o_P(1).
\]
With $b_n$ denoting the sieve basis, the KKT representation
\eqref{eq:linear-c-update} and the same conditional concentration argument
give
\begin{align*}
&(\what{P}_{k,n}-P)\left[
\left\{1-\frac{A}{\what{\pi}_{-k,n}(X)}\right\}
\{\what{\mu}^C_{-k,n}(X)-\what{\mu}_{-k,n}(X)\}\right] \\
&\qquad=O_P(r_n+n^{-1/2})
O_P\left(n^{-1/2}+\frac{d_n}{n}\right)
=o_P(n^{-1/2}).
\end{align*}
The constraint direction has bounded $L_2(P)$ norm, so the constrained fit
retains the rate $r_n$.

For an additive model with a fixed number of covariates and $s$-smooth
univariate components, the same argument applies with $s>1/2$.
\end{example}

\begin{theorem}[Asymptotic efficiency of the C-Learner]
\label{thm:asymp}
Under \Cref{ass:overlap,ass:consistency,ass:bounded_outcomes,ass:empirical_proc_mean},
\[
\sqrt n\{\what\psi_n^C-\psi(P)\}
\cd N(0,\sigma^2),
\qquad
\sigma^2:=\var_P\{\varphi(Z;P)\}.
\]
Consequently, $\what\psi_n^C$ is semiparametrically efficient.
\end{theorem}

\subsection{Proof of \Cref{thm:asymp}}
\label{section:asymptotics}

Fix a fold $k$ and abbreviate
$\what\pi=\what\pi_{-k,n}$,
$\what\mu^C=\what\mu^C_{-k,n}$, and
$\what\psi_k=\what{P}_{k,n}[\what\mu^C(X)]$. The constraint gives
\[
\what\psi_k
=\what{P}_{k,n}\left[
\what\mu^C(X)+\frac{A}{\what\pi(X)}\{Y-\what\mu^C(X)\}
\right].
\]
Adding and subtracting the true nuisance functions yields the exact
decomposition
\begin{equation}
\label{eq:asymp_direct_decomposition}
\what\psi_k-\psi(P)
=(\what{P}_{k,n}-P)\varphi(Z;P)+E_{1k}+E_{2k}+R_k,
\end{equation}
where
\begin{align*}
E_{1k}
&:=(\what{P}_{k,n}-P)\left[
A\left\{\frac{1}{\what\pi(X)}-\frac{1}{\pi\opt(X)}\right\}
\{Y-\mu\opt(X)\}
\right],\\
E_{2k}
&:=(\what{P}_{k,n}-P)\left[
\left\{1-\frac{A}{\what\pi(X)}\right\}
\{\what\mu^C(X)-\mu\opt(X)\}
\right],\\
R_k
&:=P\left[
\left\{1-\frac{\pi\opt(X)}{\what\pi(X)}\right\}
\{\what\mu^C(X)-\mu\opt(X)\}
\right].
\end{align*}

Conditional on $\what{P}_{-k,n}$, the function defining $E_{1k}$ is fixed,
centered under $P$, and evaluated on an independent fold of size $n/K$.
\Cref{ass:overlap,ass:bounded_outcomes} give
\[
\var(E_{1k}\mid\what P_{-k,n})
=\frac{K}{n}
\left\|
A\left(\frac{1}{\what\pi}-\frac{1}{\pi\opt}\right)
(Y-\mu\opt)
\right\|_{L_2(P)}^2
\leq
\frac{KB^2}{n\eta^3}
\|\what\pi-\pi\opt\|_{L_2(P)}^2
=o_P(n^{-1}).
\]
Since $K$ is fixed and $\|\what\pi-\pi\opt\|_{L_2(P)}\leq1$, conditional
Chebyshev's inequality and bounded convergence give
$E_{1k}=o_P(n^{-1/2})$. \Cref{ass:empirical_proc_mean} gives
$E_{2k}=o_P(n^{-1/2})$. Finally, the Cauchy-Schwarz inequality and overlap give
\[
|R_k|
\leq\frac{1}{\eta}
\|\what\pi-\pi\opt\|_{L_2(P)}
\|\what\mu^C-\mu\opt\|_{L_2(P)}
=o_P(n^{-1/2})
\]
by \Cref{ass:consistency}.

Averaging \eqref{eq:asymp_direct_decomposition} over the fixed number of
folds therefore gives
\[
\what\psi_n^C-\psi(P)
=(\what{P}_n-P)\varphi(Z;P)+o_P(n^{-1/2}).
\]
The central limit theorem proves the result. Since $\varphi$ is the efficient
influence function, the limiting variance is the semiparametric efficiency
bound~\citep{Tsiatis07,Kennedy22}.

\subsection{Double robustness}
\label{section:double-robustness}

\begin{proposition}[Double robustness]
\label{prop:double-robustness}
Suppose \Cref{ass:overlap,ass:bounded_outcomes} hold and, for every fold $k$,
\[
\|\what\mu^C_{-k,n}-\mu\opt\|_{L_2(P)}=O_P(1),
\qquad E_{2k}=o_P(1),
\]
where $E_{2k}$ is defined in the decomposition
\eqref{eq:asymp_direct_decomposition}. If, for each fold, either
$\|\what\pi_{-k,n}-\pi\opt\|_{L_2(P)}=o_P(1)$ or
$\|\what\mu^C_{-k,n}-\mu\opt\|_{L_2(P)}=o_P(1)$, then
$\what\psi_n^C\cp\psi(P)$. The condition on $E_{2k}$ is the
evaluation-fold stability requirement for consistency, which is weaker than
the condition used for efficiency.
\end{proposition}

\begin{proof}
Fix a fold $k$ and use the abbreviations from the preceding proof. The same
conditional variance bound gives
\[
\var(E_{1k}\mid\what P_{-k,n})
\leq\frac{KB^2}{n\eta^3}\|\what\pi-\pi\opt\|_{L_2(P)}^2
\leq\frac{KB^2}{n\eta^3}.
\]
Thus $E_{1k}=O_P(n^{-1/2})$ even without propensity consistency. Also,
\[
|R_k|\leq\frac{1}{\eta}
\|\what\pi-\pi\opt\|_{L_2(P)}
\|\what\mu^C-\mu\opt\|_{L_2(P)}=o_P(1):
\]
if the propensity is consistent, its error is multiplied by an $O_P(1)$
outcome error; if the outcome fit is consistent, its error is multiplied by
a propensity error bounded by one. The empirical influence-function term
is $O_P(n^{-1/2})$ by the overlap and outcome-moment assumptions, and
$E_{2k}=o_P(1)$ by hypothesis. On the event that the constraints are feasible,
\eqref{eq:asymp_direct_decomposition} therefore gives consistency after
averaging over the fixed number of folds. Since this event has probability
tending to one, the estimator extended as above is also consistent.
\end{proof}

%%% Local Variables:
%%% mode: latex
%%% TeX-master: "../main"
%%% End:
       % old Section 7 + its proofs (sanity check)
\section{Proof of \Cref{thm:comparison}}
\label{section:high_dim_proofs}
\label{section:high_dim_reduction_proofs}

We prove \Cref{thm:comparison} by expressing each estimator as the direct
method plus a multiple of the same first-order correction. We first
find the limiting distribution when this multiplier is fixed, then
calculate the multiplier chosen by each estimator.

\Cref{section:proof-of-reduction} proves \Cref{lem:reduction}, which gives
this representation for the C-Learner and TMLE. To compute the limiting
variance, we must account for the errors in \(\what\beta_1,\what\beta_2\)
and the dependence created by exchanging the training and evaluation
folds. \Cref{section:coefficient-limits} collects the needed Gaussian
regression limits in \Cref{lem:coefficient-limits} and proves them using
standard results for least squares, minimum-norm interpolation, and ridge.
\Cref{section:proof-of-clt} then uses these limits to prove
\Cref{thm:family-clt}, the central limit theorem for the estimator family
in~\eqref{eq:family}. It first bounds the first-order correction, then
uses \Cref{lem:estimator-expansion} to compute the joint limit of the
terms in the estimator expansion;
\Cref{section:expansion-proof} gives the algebra and remainder bounds
behind that expansion.

\Cref{section:proof-of-main} completes \Cref{thm:comparison} by calculating
the C-Learner and TMLE multipliers. Under sufficient overlap, their limits
determine the asymptotic variances, proving part~\ref{item:sufficient-overlap}
in \Cref{section:proof-of-regular}. Under deteriorating overlap, the
C-Learner multiplier tends to zero fast enough that its correction is
negligible at the root-\(n\) scale. The AIPW and TMLE corrections instead
determine their slower convergence rate. This proves
part~\ref{item:deteriorating-overlap} in \Cref{section:proof-of-degenerate}.

\subsection{Proof of \Cref{lem:reduction}}
\label{section:proof-of-reduction}

We will show that
\[
  \what\psi_{1,\CL}-\what\psi_{1,\DM}
    =\what s_{1,\CL}\Peval\left[\frac A{\pi_n\opt}(Y-\what\mu_1(X))\right].
\]
Write \(\mathcal I_1:=\Ptrain[AXX^\top]+\diag(0_q,\lambda\Id_{d_n})\).
When \(\mathcal I_1\) is invertible, we solve
for the coefficient change using the constrained normal equation. Under
interpolation, we minimize \(\norm\beta\) subject to the training equations
and the evaluation constraint.
For the TMLE, the fluctuation score equation determines the coefficient
of \(\Peval[(A/\pi_n\opt)(Y-\what\mu_1(X))]\) directly.

We use \(\Ptrain\) for the training fold and \(\Peval\) for its evaluation
fold, as in \Cref{section:high_dim}. Write
\[
  b_1:=\Peval[(A/\pi_n\opt)X],\qquad c_1:=\Peval X,\qquad
  \vartheta:=(\alpha^\top,\beta^\top)^\top,
\]
and let \(\what\vartheta_1\) and \(\what\vartheta_1^C\) denote the
coefficient vectors in~\eqref{eq:ridge-fit} and~\eqref{eq:c-program}, respectively.

\paragraph{Positive-definite training Hessian.}
When \(\mathcal I_1\) is positive definite, the Lagrangian of~\eqref{eq:c-program} is
\[
  \Ptrain[A(Y-X^\top\vartheta)^2]+\lambda\norm\beta^2
  -2t\big(b_1^\top\vartheta-\Peval[(A/\pi_n\opt)Y]\big).
\]
The unconstrained normal equation is
\(\mathcal I_1\what\vartheta_1=\Ptrain[AXY]\). Subtracting it from the
stationarity condition of the constrained problem gives
\begin{equation}
  \what\vartheta_1^C-\what\vartheta_1=t\mathcal I_1^{-1}b_1.
  \label{eq:kkt-stationarity}
\end{equation}
Thus \(\what\mu_1^C(x)-\what\mu_1(x)=t\ell_1(x)\), where
\(\ell_1(x)=x^\top\mathcal I_1^{-1}b_1\).
Substituting the coefficient change into the constraint gives \(\Peval[(A/\pi_n\opt)(Y-\what\mu_1(X))]=t b_1^\top\mathcal I_1^{-1}b_1\).
The denominator is positive whenever the evaluation fold has an observed
outcome, since the intercept coordinate of \(b_1\) is then positive.
Multiplying~\eqref{eq:kkt-stationarity} on the left by \(c_1^\top\) gives
\[
  \what\psi_{1,\CL}-\what\psi_{1,\DM}
  =\frac{c_1^\top\mathcal I_1^{-1}b_1}
         {b_1^\top\mathcal I_1^{-1}b_1}\Peval[(A/\pi_n\opt)(Y-\what\mu_1(X))].
\]

\paragraph{Minimum-norm interpolation.}

When the training design has a nullspace, we can change the evaluation
predictions while preserving every observed training prediction.
Let \(N_1=n\Ptrain A\), and let \(\mathbb X_{1,0}\) and \(\mathbb X_{1,n}\)
be the matrices of observed training rows of \(X_0\) and \(X_n\), respectively.
Define
\[
  \what\Sigma_{1,0}:=\Ptrain[AX_0(X_0)^\top],\qquad
  C_1:=\Ptrain[AX_n(X_0)^\top],\qquad
  \Pi_1^0:=\mathbb X_{1,0}
  ((\mathbb X_{1,0})^\top\mathbb X_{1,0})^{-1}(\mathbb X_{1,0})^\top.
\]
Suppose \(\mathbb X_{1,0}\) has rank \(q\) and
\((\Id-\Pi_1^0)\mathbb X_{1,n}\) has rank \(N_1-q<d_n\).
Let \(\Pi_1\) project onto the nullspace of~\((\Id-\Pi_1^0)\mathbb X_{1,n}\). Solving the training normal equation
for \(\alpha\) shows that every interpolating coefficient vector has
\(\beta=\what\beta_1+g\) and
\(\alpha=\what\alpha_1-\what\Sigma_{1,0}^{-1}C_1^\top g\),
where \(g\in\operatorname{range}(\Pi_1)\).
Thus \(x_0^\top\alpha+x_n^\top\beta-\what\mu_1(x)\) equals
\[
  \big(x_n-C_1\what\Sigma_{1,0}^{-1}x_0\big)^\top g.
\]
Partition \(b_1=(b_{1,0},b_{1,n})\) and \(c_1=(c_{1,0},c_{1,n})\) into the two
covariate blocks, and define
\begin{equation}
  u_1:=c_{1,n}-C_1\what\Sigma_{1,0}^{-1}c_{1,0},\qquad
  v_1:=b_{1,n}-C_1\what\Sigma_{1,0}^{-1}b_{1,0}.
  \label{eq:uv}
\end{equation}
The first-order constraint becomes \(v_1^\top g=\Peval[(A/\pi_n\opt)(Y-\what\mu_1(X))]\). Since the
minimum-norm solution satisfies \(\what\beta_1^\top g=0\), we have
\(\norm{\what\beta_1+g}^2=\norm{\what\beta_1}^2+\norm g^2\).
The constrained problem therefore reduces to
\[
  \minimize_{g\in\operatorname{range}(\Pi_1)}\norm g^2
  \quad\text{subject to}\quad v_1^\top g=\Peval[(A/\pi_n\opt)(Y-\what\mu_1(X))].
\]
The minimizing change is
\[
  g=\frac{\Pi_1v_1}{v_1^\top\Pi_1v_1}
    \Peval\left[\frac A{\pi_n\opt}(Y-\what\mu_1(X))\right].
\]
Averaging the change in predictions gives
\begin{equation}
\begin{gathered}
  \ell_1(x)=\big(x_n-C_1\what\Sigma_{1,0}^{-1}x_0\big)^\top\Pi_1v_1,\qquad
  \what s_{1,\CL}=\frac{u_1^\top\Pi_1v_1}{v_1^\top\Pi_1v_1},\\
  \what\psi_{1,\CL}-\what\psi_{1,\DM}
    =\what s_{1,\CL}\Peval\left[\frac A{\pi_n\opt}(Y-\what\mu_1(X))\right].
\end{gathered}
  \label{eq:s1-minnorm}
\end{equation}
Conditionally on the training fold and the \((X_0,\pi_n\opt,A)\) values in both
folds, \(b_{1,n}\) is Gaussian with covariance
\(\Peval[A/(\pi_n\opt)^2]\Id_{d_n}/n\). This covariance is positive
when the evaluation fold has an observed outcome. Since \(\Pi_1\) has
positive rank, \(v_1^\top\Pi_1v_1>0\) almost surely. This proves the
C-Learner reduction in all three regimes.

\paragraph{TMLE.}
For the TMLE, the score equation of its quadratic fluctuation is
\[
  \Peval\left[\frac A{\pi_n\opt}
  \left(Y-\what\mu_1(X)-\frac{\what\zeta_1}{\pi_n\opt}\right)\right]=0.
\]
Hence
\[
  \what\zeta_1=\frac{\Peval[(A/\pi_n\opt)(Y-\what\mu_1(X))]}
    {\Peval[A/(\pi_n\opt)^2]}.
\]
Substituting into
\(\what\psi_{1,\TMLE}=\Peval[\what\mu_1+\what\zeta_1/\pi_n\opt]\)
gives~\eqref{eq:tmle-family}. This calculation does
not require the training Hessian to be invertible.

\subsection{Preliminary Gaussian regression limits}
\label{section:coefficient-limits}

The limits below are standard consequences of high-dimensional least
squares and ridge theory
\citep{breimanfreedman1983,belkinetal2020,hastieetal2022,dicker2016,
  dobribanwager2018,rubiomestre2011}, Gaussian projection
\citep[Section~3.1]{muirhead1982}, and the Marchenko--Pastur and Bai--Yin
laws \citep[Theorems~3.6 and~5.11]{baisilverstein2010}.
We adapt them to the unpenalized covariates \(X_0\) and the random number
of observed outcomes in each training fold.

The propensity restriction preserves the Gaussian design among observed
outcomes. Conditional on the values of \(X_0,\pi_n\opt,A\) in both folds,
the observed \(X_n\) rows remain independent standard Gaussian vectors,
and their errors remain independent Gaussians. Eliminating the unpenalized
coefficient of \(X_0\) projects these rows and errors onto the orthogonal
complement of the observed \(X_0\) columns. \Cref{lem:centering} shows that,
in orthonormal coordinates, this gives a regression with independent
Gaussian design entries and independent Gaussian errors. It also makes
the projected observations independent of their components along the
\(X_0\) columns, which we use later to compute conditional variances
in the estimator expansion.

The cited regression results
give the limits of \(\norm{\what\beta_1}^2\) and
\(\norm{\what\beta_1-\beta_n\opt}^2\). Gaussian covariance identities
and rotational invariance give the limit of
\(\what\beta_1^\top\what\beta_2\), where \(\what\beta_1\) and
\(\what\beta_2\) are estimated on independent folds.
The next lemma collects the coefficient limits and the variance constants
used in the proof of \Cref{thm:family-clt} and the estimator comparison.

\begin{lemma}[Gaussian regression limits]
\label{lem:coefficient-limits}
Under~\eqref{eq:model}--\eqref{eq:gamma} and
\Cref{ass:overlap-geometry}, in each regime of \Cref{tab:regime-constants},
the coefficients \(\what\beta_1,\what\beta_2\) defined by~\eqref{eq:ridge-fit}
satisfy
\begin{equation}
  \norm{\what\beta_1}^2\cp\sigma_1^2,\qquad
  \what\beta_1^\top\what\beta_2\cp\sigma_{12},\qquad
  (\beta_n\opt)^\top\what\beta_1\cp\sigma_{1\star}.
  \label{eq:three-limits}
\end{equation}
The same limits hold after exchanging 1 and 2.
The constants are as follows, with \(t=\lim\lambda/P[\pi_n\opt]\)
in the ridge regime.
\begin{enumerate}[label=\textup{(\alph*)},leftmargin=2.2em,itemsep=3pt]
\item\label{item:coef-ols} \textup{Least squares, \(\lambda=0\), \(\gamma<1\):}
\[
  \sigma_1^2=\sigma_\star^2+\frac{\sigma^2\gamma}{1-\gamma},\qquad
  \sigma_{12}=\sigma_{1\star}=\sigma_\star^2.
\]
\item\label{item:coef-minnorm} \textup{Minimum-norm interpolation,
  \(\lambda=0\), \(\gamma>1\):}
\[
  \sigma_1^2=\frac{\sigma_\star^2}\gamma+\frac{\sigma^2}{\gamma-1},\qquad
  \sigma_{12}=\frac{\sigma_\star^2}{\gamma^2},\qquad
  \sigma_{1\star}=\frac{\sigma_\star^2}\gamma.
\]
\item\label{item:coef-ridge} \textup{Ridge regression, \(\lambda>0\):}
\[
\begin{aligned}
  \sigma_1^2&=\sigma_\star^2\big(1-2t m_\gamma(t)-t^2m_\gamma'(t)\big)
     +\sigma^2\gamma\big(m_\gamma(t)+t m_\gamma'(t)\big),\\
  \sigma_{12}&=\sigma_\star^2(1-tm_\gamma(t))^2,\qquad
  \sigma_{1\star}=\sigma_\star^2(1-tm_\gamma(t)).
\end{aligned}
\]
\end{enumerate}
In all three regimes,
\begin{equation}
  \sigma_{\mathrm{fit}}^2(\gamma,\lambda)=\sigma_1^2-\sigma_{12},\qquad
  \sigma_{\mathrm{res}}^2(\gamma,\lambda)=\sigma^2+\sigma_1^2-2\sigma_{1\star}
                             +\sigma_\star^2.
  \label{eq:derived-constants}
\end{equation}
Here \(\sigma_{\mathrm{fit}}^2\) and \(\sigma_{\mathrm{res}}^2\) have the
values in \Cref{tab:regime-constants}, and \(0<\sigma_{\mathrm{fit}}^2(\gamma,\lambda)
<\sigma_{\mathrm{res}}^2(\gamma,\lambda)\).
\end{lemma}

We will establish \Cref{lem:coefficient-limits} in the rest of this section.
We first eliminate \(\alpha\) from~\eqref{eq:ridge-fit} to express~\(\what\beta_1\) in terms of a matrix with independent Gaussian entries.
The three subsubsections then establish parts~\ref{item:coef-ols},
\ref{item:coef-minnorm}, and~\ref{item:coef-ridge}, respectively.
In each case, we apply the cited regression results to obtain~\(\norm{\what\beta_1}^2\) and \((\beta_n\opt)^\top\what\beta_1\),
then compute \(\what\beta_1^\top\what\beta_2\) using independence of the folds.

\paragraph{Gaussian reduction.}
Let \(\mathcal A\) be the sigma-field generated by
\((X_0,\pi_n\opt,A)\) in both folds. Recall that \(\mathbb X_{1,0}\)
and \(\mathbb X_{1,n}\) contain the observed training rows, and
\(\Pi_1^0\) projects onto the columns of \(\mathbb X_{1,0}\).
Write \(\boldsymbol\eps_1\) for the errors on the observed training rows.

The empirical matrix \(\what\Sigma_{1,0}=\Ptrain[AX_0X_0^\top]\)
converges to \(\Sigma\succ0\), since its entries have variances bounded
by \(P[\norm{X_0}^4]/n\). Also \(N_1/(nP[\pi_n\opt])\cp1\).
With probability tending to one, \(\mathbb X_{1,0}\) has rank~\(q\)
and \((\Id-\Pi_1^0)\mathbb X_{1,n}\) has rank
\(\min(N_1-q,d_n)\), because \(d_n/(N_1-q)\cp\gamma\) and Gaussian
matrices have full rank almost surely. We restrict the calculations to
these full-rank designs; assigning fixed values to undefined estimators
on the complementary event does not change their weak limits.

\begin{lemma}[{Gaussian conditioning; \citet[Section~3.1]{muirhead1982}}]
\label{lem:centering}
Conditional on \(\mathcal A\),
\begin{equation*}
  (\Pi_1^0\mathbb X_{1,n},\Pi_1^0\boldsymbol\eps_1)
  \ \perp\!\!\!\perp\
  ((\Id-\Pi_1^0)\mathbb X_{1,n},(\Id-\Pi_1^0)\boldsymbol\eps_1).
\end{equation*}
Express \((\Id-\Pi_1^0)\mathbb X_{1,n}\) and
\((\Id-\Pi_1^0)\boldsymbol\eps_1\) in an \(\mathcal A\)-measurable
orthonormal basis for \(\operatorname{range}(\Id-\Pi_1^0)\), and denote
their coordinates by \(G_1,e_1\). Then \(G_1\) has \(m_1=N_1-q\)
rows with independent standard Gaussian entries, and
\(e_1\sim\Normal(0,\sigma^2\Id_{m_1})\) is independent of \(G_1\).
\end{lemma}
\noindent Solving the normal equation for \(\alpha\) and substituting into the
normal equation for \(\beta\) gives
\[
  \big(\mathbb X_{1,n}^\top(\Id-\Pi_1^0)\mathbb X_{1,n}
       +n\lambda\Id\big)\what\beta_1
  =\mathbb X_{1,n}^\top(\Id-\Pi_1^0)\mathbb X_{1,n}\beta_n\opt
       +\mathbb X_{1,n}^\top(\Id-\Pi_1^0)\boldsymbol\eps_1.
\]
Apply \Cref{lem:centering} to the two projected quantities in this equation.
Their coordinates \(G_1,e_1\) satisfy
\(\mathbb X_{1,n}^\top(\Id-\Pi_1^0)\mathbb X_{1,n}=G_1^\top G_1\)
and \(\mathbb X_{1,n}^\top(\Id-\Pi_1^0)\boldsymbol\eps_1=G_1^\top e_1\),
so
\begin{equation}
  \what\beta_1=(G_1^\top G_1+n\lambda\Id)^{-1}
       (G_1^\top G_1\beta_n\opt+G_1^\top e_1),
  \label{eq:beta-formula}
\end{equation}
with the Moore--Penrose inverse in the minimum-norm regime.
By \Cref{lem:centering}, conditional on \(\mathcal A\), this is a regression
with independent standard Gaussian design entries and independent Gaussian
errors. Moreover, \(\what\beta_1\) depends only on \(G_1,e_1\), so it is
conditionally independent of
\((\Pi_1^0\mathbb X_{1,n},\Pi_1^0\boldsymbol\eps_1)\).

Exchanging the training and evaluation folds defines \(G_2,e_2,m_2\).

Equation~\eqref{eq:beta-formula} lets us apply the standard Gaussian
regression results cited in each case. Given~\(\mathcal A\), the design \(G_j\) has
\(m_j=N_j-q\) independent rows and
\[
  \frac{m_j}{nP[\pi_n\opt]}\cp1,\qquad
  \frac{d_n}{m_j}\cp\gamma,\qquad
  \frac{n\lambda}{m_j}\cp\lim_{n\to\infty}\frac\lambda{P[\pi_n\opt]}.
\]
Conditional on \(m_j\), each cited regression theorem applies to a
Gaussian design of fixed dimensions. Passing to an almost surely convergent
subsequence of \(m_j/n\) gives the stated limits when the number of
observed rows is random.

\subsubsection{Part~\ref{item:coef-ols}: least squares below interpolation}
\label{section:ols}

\paragraph{Limits of \(\norm{\what\beta_1}^2\) and \((\beta_n\opt)^\top\what\beta_1\).}
\begin{lemma}[{\citet{breimanfreedman1983};
  \citet[Thm.~1, Prop.~2]{belkinetal2020}}]
\label{lem:ols-slope}
For \(\what\beta_j\) in~\eqref{eq:beta-formula}, if \(\lambda=0\)
and \(d_n/m_j\cp\gamma<1\), then
  \[
    \norm{\what\beta_j-\beta_n\opt}^2
       \cp\frac{\sigma^2\gamma}{1-\gamma},\qquad
    \norm{\what\beta_j}^2
       \cp\sigma_\star^2+\frac{\sigma^2\gamma}{1-\gamma},\qquad
    (\beta_n\opt)^\top\what\beta_j\cp\sigma_\star^2.
  \]
\end{lemma}

\paragraph{Limit of \(\what\beta_1^\top\what\beta_2\).}
It remains to compute \(\what\beta_1^\top\what\beta_2\), which enters
the covariance between the two evaluation estimates.
Given both designs, \(\what\beta_1-\beta_n\opt\) and \(\what\beta_2-\beta_n\opt\) are independent centered
Gaussians with covariance matrices \(\sigma^2(G_j^\top G_j)^{-1}\).
The Marchenko--Pastur and Bai--Yin laws
\citep[Theorems~3.6 and~5.11]{baisilverstein2010} imply
\[
  \opnorm{(G_j^\top G_j)^{-1}}=\Op(n^{-1}),\qquad
  \tr(G_j^\top G_j)^{-1}\cp\frac\gamma{1-\gamma}.
\]
Consequently
\[
\begin{aligned}
  \E[((\what\beta_1-\beta_n\opt)^\top(\what\beta_2-\beta_n\opt))^2
    \mid\mathcal A,G_1,G_2]
  &=\sigma^4\tr\big((G_1^\top G_1)^{-1}(G_2^\top G_2)^{-1}\big)\\
  &=\Op(n^{-1}).
\end{aligned}
\]
Thus \((\what\beta_1-\beta_n\opt)^\top(\what\beta_2-\beta_n\opt)\cp0\). Expanding
\(\what\beta_1^\top\what\beta_2\) gives
\(\sigma_{12}=\sigma_{1\star}=\sigma_\star^2\), and hence
\[
  \sigma_{\mathrm{fit}}^2(\gamma,0)=\frac{\sigma^2\gamma}{1-\gamma},\qquad
  \sigma_{\mathrm{res}}^2(\gamma,0)=\frac{\sigma^2}{1-\gamma}.
\]
These are the least-squares constants in \Cref{tab:regime-constants},
which proves \Cref{lem:coefficient-limits}\ref{item:coef-ols}.

\subsubsection{Part~\ref{item:coef-minnorm}: minimum-norm interpolation}
\label{section:minnorm}

\paragraph{Limits of \(\norm{\what\beta_1}^2\) and \((\beta_n\opt)^\top\what\beta_1\).}
For \(\lambda=0\) and \(\gamma>1\), equation~\eqref{eq:beta-formula} gives
\[
  \what\beta_j=(\Id-\Pi_j)\beta_n\opt+\omega_j,
  \qquad
  \omega_j\mid\mathcal A,G_j
     \sim\Normal\big(0,\sigma^2(G_j^\top G_j)^\dagger\big).
\]
\begin{lemma}[{\citet[Thms.~1--2]{belkinetal2020};
  \citet[Thm.~1, Cor.~3]{hastieetal2022}}]
\label{lem:minnorm-slope}
For the minimum-norm solution in~\eqref{eq:beta-formula}, if
\(\lambda=0\) and \(d_n/m_j\cp\gamma>1\), then
  \[
    \norm{\what\beta_j}^2
       \cp\frac{\sigma_\star^2}\gamma+\frac{\sigma^2}{\gamma-1},\qquad
    (\beta_n\opt)^\top\what\beta_j\cp\frac{\sigma_\star^2}\gamma.
  \]
\end{lemma}

\paragraph{Limit of \(\what\beta_1^\top\what\beta_2\).}
By the Bai--Yin and Marchenko--Pastur laws
\citep[Theorems~5.11 and~3.6]{baisilverstein2010}, the nonzero eigenvalues of
\((G_j^\top G_j)^\dagger\) are \(\Op(n^{-1})\) and its trace is \(\Op(1)\).
Consequently, the conditional variances of
\(\omega_1^\top\omega_2\) and
\((\beta_n\opt)^\top(\Id-\Pi_1)\omega_2\) are \(\Op(n^{-1})\);
the latter bound also holds with the folds exchanged. After removing
these noise terms, \(\what\beta_1^\top\what\beta_2\) equals
\((\beta_n\opt)^\top(\Id-\Pi_1)(\Id-\Pi_2)\beta_n\opt+\op(1)\).
Condition on \(\mathcal A,G_2\). This fixes
\((\Id-\Pi_2)\beta_n\opt\), while \(\Id-\Pi_1\) remains the projection
onto a uniformly random \(m_1\)-dimensional subspace of \(\R^{d_n}\).
Its conditional mean is \((m_1/d_n)\Id\): all directions have the same
expected projection, and the trace is \(m_1\). Thus
\[
  \E[(\beta_n\opt)^\top(\Id-\Pi_1)(\Id-\Pi_2)\beta_n\opt
       \mid\mathcal A,G_2]
  =\frac{m_1}{d_n}(\beta_n\opt)^\top(\Id-\Pi_2)\beta_n\opt.
\]
The Haar moment formula \citep[Lemma~2.22]{meckes2019classical} gives
\[
  \var((\beta_n\opt)^\top(\Id-\Pi_1)(\Id-\Pi_2)\beta_n\opt
       \mid\mathcal A,G_2)
  \le\frac{2}{d_n}\norm{\beta_n\opt}^2
           \norm{(\Id-\Pi_2)\beta_n\opt}^2
  \le\frac{2}{d_n}\norm{\beta_n\opt}^4=o(1).
\]
Chebyshev's inequality therefore lets us replace \(\Id-\Pi_1\) by
its conditional mean:
\[
  (\beta_n\opt)^\top(\Id-\Pi_1)(\Id-\Pi_2)\beta_n\opt
  -\frac{m_1}{d_n}(\beta_n\opt)^\top(\Id-\Pi_2)\beta_n\opt\cp0.
\]
Since \(m_1/d_n\cp1/\gamma\) and
\((\beta_n\opt)^\top(\Id-\Pi_2)\beta_n\opt\cp\sigma_\star^2/\gamma\),
we obtain \(\sigma_{12}=\sigma_\star^2/\gamma^2\).
Substitution into~\eqref{eq:derived-constants} yields
\[
  \sigma_{\mathrm{fit}}^2(\gamma,0)
  =\frac{\sigma_\star^2(\gamma-1)}{\gamma^2}
       +\frac{\sigma^2}{\gamma-1},\qquad
  \sigma_{\mathrm{res}}^2(\gamma,0)=\gamma\sigma_{\mathrm{fit}}^2(\gamma,0).
\]
These are the minimum-norm constants in \Cref{tab:regime-constants},
which proves \Cref{lem:coefficient-limits}\ref{item:coef-minnorm}.

\subsubsection{Part~\ref{item:coef-ridge}: ridge regression}
\label{section:ridge}

\paragraph{Limit of \(\norm{\what\beta_1-\beta_n\opt}^2\).}
Let \(t=\lim_{n\to\infty}\lambda/P[\pi_n\opt]\), as in
\Cref{section:high_dim}, and put \(\what\Sigma_{j,n}=G_j^\top G_j/m_j\).
The effective penalty \(t_j=n\lambda/m_j\) converges to \(t>0\).
The resolvent identity gives
\[
  \opnorm{(\what\Sigma_{j,n}+t_j\Id)^{-1}-(\what\Sigma_{j,n}+t\Id)^{-1}}
  \le\frac{|t_j-t|}{t_jt}=\op(1).
\]
The resolvent norm difference tends to zero, so replacing \(t_j\) by
\(t\) does not change the limits of normalized traces. We can therefore
apply the Marchenko--Pastur law
\citep[Theorem~3.6]{baisilverstein2010} at the fixed penalty \(t\).
\begin{lemma}[{\citet[Corollary~6]{hastieetal2022};
  \citet{dicker2016,dobribanwager2018}}]
\label{lem:ridge-slope}
For \(\what\beta_j\) in~\eqref{eq:beta-formula}, if
\(d_n/m_j\cp\gamma\in(0,\infty)\) and \(n\lambda/m_j\cp t>0\), then
\begin{equation}
    \E[\norm{\what\beta_j-\beta_n\opt}^2\mid\mathcal A,G_j]
       \cp-\sigma_\star^2t^2m_\gamma'(t)
       +\sigma^2\gamma\big(m_\gamma(t)+t m_\gamma'(t)\big).
  \label{eq:ridge-risk}
\end{equation}
\end{lemma}

\paragraph{Limits of \(\norm{\what\beta_1}^2\) and \((\beta_n\opt)^\top\what\beta_1\).}
To compute \((\beta_n\opt)^\top\what\beta_j\), we use the following
specialization to \(\what\Sigma_{j,n}=G_j^\top G_j/m_j\).
\begin{lemma}[{\citet[Theorem~1]{rubiomestre2011}}]
\label{lem:resolvent-equivalent}
Suppose \(G_j\) has independent standard Gaussian entries and
\(d_n/m_j\cp\gamma\in(0,\infty)\). For fixed \(t>0\) and deterministic
\(\beta_n\opt\) with \(\norm{\beta_n\opt}^2\to\sigma_\star^2\),
\begin{equation*}
  (\beta_n\opt)^\top(\what\Sigma_{j,n}+t\Id)^{-1}\beta_n\opt
    \cp\sigma_\star^2m_\gamma(t).
\end{equation*}
\end{lemma}
\noindent By~\eqref{eq:beta-formula},
\[
\begin{gathered}
  \what\beta_j=(\what\Sigma_{j,n}+t_j\Id)^{-1}\what\Sigma_{j,n}\beta_n\opt+\xi_j,\\
  \xi_j\mid\mathcal A,G_j\sim\Normal(0,\Xi_j),\qquad
  \Xi_j:=\frac{\sigma^2}{m_j}\what\Sigma_{j,n}(\what\Sigma_{j,n}+t_j\Id)^{-2}.
\end{gathered}
\]
The matrix \(\Xi_j\) satisfies \(\opnorm{\Xi_j}=\Op(n^{-1})\)
and \(\tr\Xi_j=\Op(1)\). The conditional bias
\(\what\beta_j-\beta_n\opt-\xi_j=-t_j(\what\Sigma_{j,n}+t_j\Id)^{-1}\beta_n\opt\) has bounded norm.
The Gaussian variance formula therefore gives
\[
\begin{aligned}
  \var(\norm{\what\beta_j-\beta_n\opt}^2\mid\mathcal A,G_j)
   &=2\tr(\Xi_j^2)
       +4(\what\beta_j-\beta_n\opt-\xi_j)^\top\Xi_j(\what\beta_j-\beta_n\opt-\xi_j)\\
   &\le2\opnorm{\Xi_j}\tr\Xi_j
       +4\opnorm{\Xi_j}\norm{\what\beta_j-\beta_n\opt-\xi_j}^2=\Op(n^{-1}).
\end{aligned}
\]
The conditional variance bound implies that
\(\norm{\what\beta_j-\beta_n\opt}^2-\E[\norm{\what\beta_j-\beta_n\opt}^2\mid\mathcal A,G_j]\cp0\).
Hence \Cref{lem:ridge-slope} also gives the limit of
\(\norm{\what\beta_j-\beta_n\opt}^2\).
\Cref{lem:resolvent-equivalent} and
\((\beta_n\opt)^\top\xi_j=\op(1)\) give
\(\sigma_{1\star}=\sigma_\star^2(1-t m_\gamma(t))\). Combining this with
\eqref{eq:ridge-risk} gives
\[
  \sigma_1^2=\sigma_\star^2\big(1-2t m_\gamma(t)-t^2m_\gamma'(t)\big)
             +\sigma^2\gamma\big(m_\gamma(t)+t m_\gamma'(t)\big).
\]
\paragraph{Limit of \(\what\beta_1^\top\what\beta_2\) and the variance inequalities.}
The bounds \(\opnorm{\Xi_j}=\Op(n^{-1})\) and
\(\tr\Xi_j=\Op(1)\) make every term involving \(\xi_1\) or
\(\xi_2\) in \(\what\beta_1^\top\what\beta_2\) converge to zero. Put \(B_j=(\what\Sigma_{j,n}+t_j\Id)^{-1}\what\Sigma_{j,n}\), whose eigenvalues lie
in \([0,1]\). Conditional on \(\mathcal A,G_2\) and the eigenvalues of
\(B_1\), its eigenvectors are uniformly distributed. As for the projection
in \Cref{section:minnorm}, its conditional mean is
\((\tr B_1/d_n)\Id\), and the conditional variance of
\((\beta_n\opt)^\top B_1B_2\beta_n\opt\) is at most
\(2\norm{\beta_n\opt}^2\norm{B_2\beta_n\opt}^2/d_n=o(1)\).
Chebyshev's inequality gives
\[
  (\beta_n\opt)^\top B_1B_2\beta_n\opt
  -\frac{\tr B_1}{d_n}(\beta_n\opt)^\top B_2\beta_n\opt\cp0.
\]
Since \(\tr B_1/d_n\cp1-tm_\gamma(t)\) and
\((\beta_n\opt)^\top B_2\beta_n\opt\cp\sigma_\star^2(1-tm_\gamma(t))\),
\(\sigma_{12}=\sigma_\star^2(1-t m_\gamma(t))^2\).
Substitution of \(\sigma_1^2,\sigma_{1\star},\sigma_{12}\) into
\eqref{eq:derived-constants} gives the ridge formulas in
\Cref{tab:regime-constants}, proving part~\ref{item:coef-ridge}.
To finish \Cref{lem:coefficient-limits},
it remains to verify
\begin{equation}
  \sigma_{\mathrm{fit}}^2(\gamma,\lambda)>0,
  \qquad \sigma_{\mathrm{res}}^2(\gamma,\lambda)>\sigma_{\mathrm{fit}}^2(\gamma,\lambda).
  \label{eq:fit-positivity}
\end{equation}
For ridge, \(\sigma_{\mathrm{fit}}^2(\gamma,\lambda)>0\) follows from \(\sigma^2>0\) and
\[
  -m_\gamma'(t)-m_\gamma(t)^2\ge0,\qquad
  m_\gamma(t)+t m_\gamma'(t)=\int\frac{x}{(x+t)^2}\,dF_\gamma(x)>0.
\]
The difference \(\sigma_{\mathrm{res}}^2-\sigma_{\mathrm{fit}}^2\) is
\(\sigma^2+\sigma_\star^2t^2m_\gamma(t)^2>0\). The least-squares formulas give both
inequalities directly.
This completes the proof of \Cref{lem:coefficient-limits}. In particular,
\(\norm{\what\beta_1-\beta_n\opt}=\Op(1)\) and \(\norm{\what\beta_2-\beta_n\opt}=\Op(1)\), which we
will use to control the variance of
\(\Peval[(A/\pi_n\opt)(\eps-X_n^\top(\what\beta_1-\beta_n\opt))]\).

\subsection{Proof of \Cref{thm:family-clt}}
\label{section:proof-of-clt}

We first bound the first-order correction, proving the first assertion of
\Cref{thm:family-clt}. We then use \Cref{lem:estimator-expansion} to
compute the limiting distribution of \(\what\psi(s)\).

\paragraph{The size of the first-order correction.}
For the outcome model \(\what\mu_1\) trained on \(\Ptrain\), the
first-order correction on \(\Peval\) is
\begin{equation}
\begin{aligned}
  \Peval\left[\frac A{\pi_n\opt}(Y-\what\mu_1(X))\right]
  &=\Peval\left[\frac A{\pi_n\opt}
       \big(\eps-X_n^\top(\what\beta_1-\beta_n\opt)\big)\right]
    -\Peval\left[\frac A{\pi_n\opt}X_0\right]^\top
       (\what\alpha_1-\alpha\opt).
\end{aligned}
\label{eq:residual-exact}
\end{equation}

The weighted mean \(\Peval[(A/\pi_n\opt)X_0]\) multiplies
\(\what\alpha_1-\alpha\opt\) in~\eqref{eq:residual-exact}. We first show
\(\what\alpha_1-\alpha\opt=\Op(n^{-1/2})\). The coefficients of \(X_0\)
are unpenalized, so their training normal equation gives
\begin{equation}
  \what\alpha_1-\alpha\opt
  =\what\Sigma_{1,0}^{-1}\Ptrain
       [AX_0(\eps-X_n^\top(\what\beta_1-\beta_n\opt))].
  \label{eq:alpha-identity}
\end{equation}
Here \(\what\Sigma_{1,0}=\Ptrain[AX_0X_0^\top]\), as in
\Cref{section:coefficient-limits}. Given \(\mathcal A\),
\(\what\beta_1\) is a function of \((G_1,e_1)\) by~\eqref{eq:beta-formula}.
The quantities \(\Ptrain[AX_0\eps]\) and \(\Ptrain[AX_0X_n^\top]\)
are functions of \(\Pi_1^0\boldsymbol\eps_1\) and
\(\Pi_1^0\mathbb X_{1,n}\), respectively, so they are independent of
\((G_1,e_1)\) by \Cref{lem:centering}. Conditional on
\(\mathcal A,G_1,e_1\), the coefficient error
\(\what\alpha_1-\alpha\opt\) in~\eqref{eq:alpha-identity}
is therefore centered Gaussian with covariance
\[
  \frac{\sigma^2+\norm{\what\beta_1-\beta_n\opt}^2}{n}\what\Sigma_{1,0}^{-1}.
\]
By \Cref{lem:coefficient-limits}, \(\norm{\what\beta_1-\beta_n\opt}=\Op(1)\), and
\(\what\Sigma_{1,0}\cp\Sigma\succ0\) was established in
\Cref{section:coefficient-limits}. Thus the covariance is \(\Op(n^{-1})\),
which proves \(\what\alpha_1-\alpha\opt=\Op(n^{-1/2})\).

Next we bound
\(\Peval[(A/\pi_n\opt)(\eps-X_n^\top(\what\beta_1-\beta_n\opt))]\),
the first term on the right-hand side of~\eqref{eq:residual-exact}.
Its conditional variance depends on \(\Peval[A/(\pi_n\opt)^2]\).
We first record the bounds on this and the other weight averages used below,
retaining their dependence on
\(P[(\pi_n\opt)^{-1}]\) so the same bounds apply under deteriorating overlap:
\begin{equation}
\begin{gathered}
  \Peval[A/\pi_n\opt]-1
    =\Op\left(\sqrt{\frac{P[(\pi_n\opt)^{-1}]}n}\right),\qquad
  \Peval[(A/\pi_n\opt)X_0]-P[X_0]
    =\Op\left(\sqrt{\frac{P[(\pi_n\opt)^{-1}]}n}\right),\\
  \frac{\Peval[A/(\pi_n\opt)^2]}{P[(\pi_n\opt)^{-1}]}\cp1,
  \qquad
  \frac{\Peval[(\pi_n\opt)^{-1}]}{P[(\pi_n\opt)^{-1}]}\cp1.
\end{gathered}
  \label{eq:weighted-llns}
\end{equation}
Indeed, \(\Peval[A/\pi_n\opt]\) has mean \(1\) and variance at most
\(P[(\pi_n\opt)^{-1}]/n\). The vector
\(\Peval[(A/\pi_n\opt)X_0]\) has mean \(P[X_0]\) and covariance
trace at most \(P[\norm{X_0}^2/\pi_n\opt]/n\). For the squared weights,
\[
  \var\left(\frac{\Peval[A/(\pi_n\opt)^2]}{P[(\pi_n\opt)^{-1}]}\right)
  \le\frac{P[(\pi_n\opt)^{-3}]}{nP[(\pi_n\opt)^{-1}]^2}\longrightarrow0.
\]
For \(\Peval[(\pi_n\opt)^{-1}]\), the normalized variance is bounded
by \(P[(\pi_n\opt)^{-2}]/(nP[(\pi_n\opt)^{-1}]^2)\), which tends
to zero because \(P[(\pi_n\opt)^{-2}]\le P[(\pi_n\opt)^{-3}]\).
All four claims also hold with \(\Ptrain\) in place of \(\Peval\).

Conditional on \(\mathcal A\) and the training observations,
\(\Peval[(A/\pi_n\opt)(\eps-X_n^\top(\what\beta_1-\beta_n\opt))]\) is centered Gaussian
with variance
\[
  \frac{\Peval[A/(\pi_n\opt)^2]}n(\sigma^2+\norm{\what\beta_1-\beta_n\opt}^2)
    =\Op\left(\frac{P[(\pi_n\opt)^{-1}]}n\right).
\]
The product \(\Peval[(A/\pi_n\opt)X_0]^\top
(\what\alpha_1-\alpha\opt)\) is \(\Op(n^{-1/2})\), since the weighted
mean converges to \(P[X_0]\) and the coefficient error is
\(\Op(n^{-1/2})\). Adding the two residual terms gives
\begin{equation}
  \Peval[(A/\pi_n\opt)(Y-\what\mu_1(X))]=\Op\left(\sqrt{\frac{P[(\pi_n\opt)^{-1}]}n}\right),
  \label{eq:residual-rate}
\end{equation}
and the same bound holds after exchanging the folds.
The bound~\eqref{eq:residual-rate} proves the first assertion of
\Cref{thm:family-clt} when \(P[(\pi_n\opt)^{-1}]\) converges.

\paragraph{The estimator expansion.}
The following lemma separates \(\what\psi(s)-\psi\) into the ordinary
sample mean of \((\alpha\opt)^\top X_0\), two terms \(Z_1(s),Z_2(s)\)
that are Gaussian conditional on \(\mathcal A,G_1,e_1,G_2,e_2\), and two
terms \(E_1,E_2\) from the inverse-propensity weights orthogonal to
the columns of \(\mathbb X_{1,0},\mathbb X_{2,0}\).

\begin{lemma}[Estimator expansion]
\label{lem:estimator-expansion}
Under the assumptions of \Cref{lem:coefficient-limits}, let
\(\boldsymbol\pi_1^{-1}\in\R^{N_1}\) contain the inverse propensities
on the observed training rows, and let
\(\boldsymbol r_1=(\Id-\Pi_1^0)\boldsymbol\pi_1^{-1}\).
Define
\[
\begin{aligned}
  E_1&=\frac1n\boldsymbol r_1^\top
       (\boldsymbol\eps_1-\mathbb X_{1,n}(\what\beta_2-\beta_n\opt)),\\
  Z_1(s)&=\what\beta_2^\top\Ptrain X_n
    +P[X_0]^\top\what\Sigma_{1,0}^{-1}
       \Ptrain\left[AX_0\big(\eps-X_n^\top((1-s)\what\beta_1+s\what\beta_2-\beta_n\opt)\big)\right],
\end{aligned}
\]
and define \(Z_2(s),E_2\) by exchanging 1 and 2 and \(\Ptrain,\Peval\).
For every fixed \(s\),
\begin{equation}
\begin{aligned}
  \what\psi(s)-\psi
  &=(\alpha\opt)^\top\left(\frac{\Ptrain+\Peval}{2}-P\right)X_0
       +\frac{Z_1(s)+Z_2(s)}2+s\frac{E_1+E_2}2\\
  &\quad+\Op(n^{-1})
       +\Op\left(\frac{|s|\sqrt{P[(\pi_n\opt)^{-1}]}}n\right).
\end{aligned}
  \label{eq:two-fold}
\end{equation}
\end{lemma}
\noindent We prove the lemma in \Cref{section:expansion-proof}.

It remains to obtain the joint limit of the terms in~\eqref{eq:two-fold}.
The ordinary central limit theorem handles the sample mean of
\((\alpha\opt)^\top X_0\). For \(Z_1(s),Z_2(s)\),
\Cref{lem:centering} gives the conditional Gaussian distribution, so
we compute their conditional variances first.
For \(E_1,E_2\), the difficulty is that the Gaussian observations in
\(E_1\) also enter \(\what\beta_1\), which is used in \(E_2\).
We choose orthonormal coordinates with the first axis along
\(\boldsymbol r_1\), so \(E_1\) uses the covariates and noise from just
the first row. Recomputing~\(\what\beta_1\) without this row removes
those observations from the coefficient used in \(E_2\).
We show that this changes \(E_2\) by a negligible amount.
Doing the same in the other fold gives two terms that are independent
Gaussian variables conditional on the remaining rows.
Their variance limits, together with those of \(Z_1(s),Z_2(s)\) and
the sample mean of \((\alpha\opt)^\top X_0\), give \(V(s)\).

\paragraph{The variance of \(Z_1(s)\).}
Let \(\mc F\) be generated by \(\mathcal A\) and
\((G_1,e_1),(G_2,e_2)\). Conditional on \(\mc F\),
\(\what\beta_1,\what\beta_2\) are fixed and \(Z_1(s)\) is linear in
\(\Pi_1^0\mathbb X_{1,n}\), \(\Pi_1^0\boldsymbol\eps_1\), and the
\(X_n\) values on training rows with \(A=0\). In particular,
\(\mathbf1^\top(\Id-\Pi_1^0)=0\) because \(X_0\) includes an intercept,
so \(\mathbf1^\top\mathbb X_{1,n}=\mathbf1^\top\Pi_1^0\mathbb X_{1,n}\).
\Cref{lem:centering} makes \(Z_1(s),Z_2(s)\) independent
centered Gaussians given \(\mc F\). Their variance follows from
\[
\begin{gathered}
  n\cov(\Ptrain X_n\mid\mc F)=\Id,\qquad
  n\cov\left(\Ptrain X_n,\,
    \Ptrain[AX_nX_0^\top]\what\Sigma_{1,0}^{-1}P[X_0]
    \mid\mc F\right)=\Id,
\end{gathered}
\]
where the second equality uses the intercept coordinate of \(X_0\).
The covariance of \(\Ptrain X_n\) and
\(\Ptrain[AX_nX_0^\top]\what\Sigma_{1,0}^{-1}P[X_0]\) gives the
cross term in \(\var(Z_1(s)\mid\mc F)\).
Combining the displayed covariance identities with
\(n\cov(\Ptrain[AX_0\eps]\mid\mc F)=\sigma^2\what\Sigma_{1,0}\) gives
\begin{equation}
\begin{aligned}
  n\var(Z_1(s)\mid\mc F)
  &=\left(\sigma^2+\norm{(1-s)\what\beta_1+s\what\beta_2-\beta_n\opt}^2\right)
       P[X_0]^\top\what\Sigma_{1,0}^{-1}P[X_0]\\
  &\quad+\norm{\what\beta_2}^2
       -2\left((1-s)\what\beta_1+s\what\beta_2-\beta_n\opt\right)^\top\what\beta_2\\
  &\cp V_\DM-(\alpha\opt)^\top\cov(X_0)\alpha\opt
       -2s(1+(1-s)\sigma_X^2)\sigma_{\mathrm{fit}}^2(\gamma,\lambda)
       \qquad\mbox{by \Cref{lem:coefficient-limits}}.
\end{aligned}
  \label{eq:condvar}
\end{equation}
Exchanging the folds gives the identical variance limit for \(Z_2(s)\).

\paragraph{The variance of \(E_1\) and independence across folds.}
The term \(E_1=n^{-1}\boldsymbol r_1^\top
(\boldsymbol\eps_1-\mathbb X_{1,n}(\what\beta_2-\beta_n\opt))\) depends on
\(\boldsymbol r_1\), whose squared norm satisfies
\begin{equation}
\begin{aligned}
  \frac{\norm{\boldsymbol r_1}^2}{n}
  &=\Ptrain[A/(\pi_n\opt)^2]
    -\Ptrain[(A/\pi_n\opt)X_0]^\top\what\Sigma_{1,0}^{-1}
       \Ptrain[(A/\pi_n\opt)X_0]\\
  &=P[(\pi_n\opt)^{-1}]-\sigma_X^2
       +\op\big(P[(\pi_n\opt)^{-1}]\big),
\end{aligned}
  \label{eq:wperp}
\end{equation}
The subtraction removes the squared norm of
\(\Pi_1^0\boldsymbol\pi_1^{-1}\).
By~\eqref{eq:weighted-llns},
\(\Ptrain[A/(\pi_n\opt)^2]/P[(\pi_n\opt)^{-1}]\cp1\) and
\(\Ptrain[(A/\pi_n\opt)X_0]\cp P[X_0]\), so the subtracted quadratic
form converges to \(\sigma_X^2\). If
\(\boldsymbol r_1=\boldsymbol r_2=0\), then \(E_1=E_2=0\) and only
the variances of \(Z_1(s),Z_2(s)\) and
\((\alpha\opt)^\top((\Ptrain+\Peval)/2-P)X_0\) remain to be computed. In general, \(E_1\) uses
observations that also determine \(\what\beta_1\), while \(E_2\) uses
\(\what\beta_1\) as its coefficient. We therefore cannot treat
\(E_1,E_2\) as independent after conditioning on both full regressions.

We remove this dependence by isolating the single Gaussian row used in
\(E_1\). Choose the orthonormal coordinates for \(G_1,e_1\) so their
first coordinate is along \(\boldsymbol r_1\). This rotation depends
only on \(\mathcal A\) and leaves \(\what\beta_1\) unchanged. If
\(\boldsymbol r_1=0\), any first coordinate suffices.
Write the first row as~\((g_1^\top,\eps_1^\circ)\). Then
\[
  E_1=\frac{\norm{\boldsymbol r_1}}n
       (\eps_1^\circ-g_1^\top(\what\beta_2-\beta_n\opt)).
\]
Let \(G_1^-,e_1^-\) denote the remaining rows and \(\what\beta_1^-\)
the coefficient vector from~\eqref{eq:beta-formula} with
\((G_1^-,e_1^-)\) in place of \((G_1,e_1)\), retaining \(n\lambda\).
Make the same construction in fold 2.
\begin{lemma}[Deleting one Gaussian row]
  \label{lem:loo}
  In each of the three regimes,
  \(\norm{\what\beta_1-\what\beta_1^-}=\Op(n^{-1/2})\).
  Thus deleting one row preserves the coefficient limits.
\end{lemma}
\noindent The proof is given in \Cref{section:one-row-proof}.
Replace \(\what\beta_2\) in \(E_1\) by \(\what\beta_2^-\), writing
\[
  E_1^-:=\frac{\norm{\boldsymbol r_1}}n
       \big(\eps_1^\circ-g_1^\top(\what\beta_2^--\beta_n\opt)\big).
\]
Conditional on $\mathcal A,G_2,e_2$, the vector $\what\beta_2-\what\beta_2^-$ is fixed and $g_1\sim\Normal(0,\Id_{d_n})$, independently of these quantities. Consequently, $g_1^\top(\what\beta_2-\what\beta_2^-)\mid \mathcal A,G_2,e_2 \sim\Normal(0,\norm{\what\beta_2-\what\beta_2^-}^2)$.
By \Cref{lem:loo}, $\norm{\what\beta_2-\what\beta_2^-}= \Op(n^{-1/2})$, and hence $g_1^\top(\what\beta_2-\what\beta_2^-)=\Op(n^{-1/2})$.
Also, \eqref{eq:wperp} gives
$\norm{\boldsymbol r_1}=\Op\left(\sqrt{nP[(\pi_n\opt)^{-1}]}\right).$ It follows that
\begin{equation}
  E_1-E_1^- =-\frac{\norm{\boldsymbol r_1}}n\,g_1^\top(\what\beta_2-\what\beta_2^-)
=\Op\left(\frac{\sqrt{P[(\pi_n\opt)^{-1}]}}n\right).
  \label{eq:H-minus}
\end{equation}
Let \(\mc F^-\) be generated by \(\mathcal A\), \((G_1^-,e_1^-)\),
and \((G_2^-,e_2^-)\). Conditional on \(\mc F^-\),
\((g_1,\eps_1^\circ)\) and \((g_2,\eps_2^\circ)\) are independent,
while \(\what\beta_1^-,\what\beta_2^-\) are fixed.
Hence \(E_1^-,E_2^-\) are independent centered Gaussians, with
\begin{equation}
  n\var(E_1^-\mid\mc F^-)
  =\frac{\norm{\boldsymbol r_1}^2}{n}
       \big(\sigma^2+\norm{\what\beta_2^--\beta_n\opt}^2\big)
  =\big(P[(\pi_n\opt)^{-1}]-\sigma_X^2\big)
       \sigma_{\mathrm{res}}^2(\gamma,\lambda)
       +\op\big(P[(\pi_n\opt)^{-1}]\big).
  \label{eq:H-var}
\end{equation}
Exchanging the folds gives the bound~\eqref{eq:H-minus} for
\(E_2-E_2^-\) and the variance formula~\eqref{eq:H-var} for \(E_2^-\).

\paragraph{Completing the central limit theorem.}
Suppose \(P[(\pi_n\opt)^{-1}]\to\kappa<\infty\).
The remainders in~\eqref{eq:two-fold} and~\eqref{eq:H-minus} are then
\(\Op(n^{-1})\). Define
$M_n:=\sqrt{2n}(\alpha\opt)^\top\left(\frac{\Ptrain+\Peval}{2}-P\right)X_0$,
$U_n:=\sqrt{\frac n2}\bigl(Z_1(s)+Z_2(s)\bigr)$,
 and $H_n
    :=\sqrt{\frac n2}\bigl(E_1^-+E_2^-\bigr)$.
Then $\sqrt{2n}(\what\psi(s)-\psi)
    =M_n+U_n+sH_n+\op(1)$.
Write 
\begin{align*}
v_0&:=(\alpha\opt)^\top\cov(X_0)\alpha\opt, \\
  v_Z(s)&:=V_\DM-v_0
       -2s(1+(1-s)\sigma_X^2)
          \sigma_{\mathrm{fit}}^2(\gamma,\lambda), \\
  v_E&:=(\kappa-\sigma_X^2)
          \sigma_{\mathrm{res}}^2(\gamma,\lambda).
\end{align*}

Conditional on \(\mc F\), the variable \(U_n\) is centered Gaussian,
whereas \(M_n\) and \(H_n\) are measurable.
By \eqref{eq:condvar},
\(\var(U_n\mid\mc F)\cp v_Z(s)\).
Conditional on \(\mc F^-\), the variable \(H_n\) is centered Gaussian
and \(M_n\) is measurable. By \eqref{eq:H-var},
\(\var(H_n\mid\mc F^-)\cp v_E\).
For any fixed \(a,b,c\in\R\), successive conditioning therefore gives
\begin{align*}
  \E\exp\{ i(aM_n+bU_n+cH_n)\}
  &=\exp\{-b^2v_Z(s)/2\} \E\exp\{i(aM_n+cH_n)\}+o(1)\\
&=\exp\{-b^2v_Z(s)/2-c^2v_E/2\}
    \E\exp\{ i aM_n\}+o(1)\\
  &\rightarrow
    \exp\{-(a^2v_0+b^2v_Z(s)+c^2v_E)/2\}.
\end{align*}
The replacements of conditional variances by their limits are valid because the corresponding exponential factors are bounded and converge in probability. The last line uses the ordinary central limit theorem for $M_n$. Thus $M_n,U_n,H_n$ converge jointly to independent centered Gaussians, allowing zero variances.

It follows that
\[
  \sqrt{2n}(\what\psi(s)-\psi)\cd\Normal(0,V(s)),
\]
where
\[
\begin{aligned}
  V(s)&=v_0+v_Z(s)+s^2v_E\\
  &=V_\DM-2s(1+(1-s)\sigma_X^2)\sigma_{\mathrm{fit}}^2(\gamma,\lambda)
       +s^2(\kappa-\sigma_X^2)\sigma_{\mathrm{res}}^2(\gamma,\lambda),
       \end{aligned}
\]
as claimed. At \(s=0\), the terms \(E_1,E_2\) are absent. Conditioning on $\mc F$ and applying the ordinary central limit theorem to $M_n$ therefore proves the direct-method CLT even under deteriorating overlap.

\subsubsection{Derivation of the estimator expansion}
\label{section:expansion-proof}

We prove \Cref{lem:estimator-expansion} using the coefficient and weight
bounds established before its statement. We substitute the normal
equation for \(\what\alpha_1\) into the estimator
and average across folds. Finally, decomposing the inverse-propensity
weights will give \(Z_1(s),E_1\) and their counterparts in the other fold.

\paragraph{Substituting the normal equation and averaging across folds.}
Substitute \eqref{eq:residual-exact} and \eqref{eq:alpha-identity} into
\(\what\psi_1(s)=\Peval\what\mu_1+s\Peval[(A/\pi_n\opt)(Y-\what\mu_1(X))]\).
Collecting the terms gives
\begin{equation}
\begin{aligned}
  \what\psi_1(s)-\psi
  &=(\alpha\opt)^\top(\Peval-P)X_0+\what\beta_1^\top\Peval X_n\\
  &\quad+(1-s)P[X_0]^\top\what\Sigma_{1,0}^{-1}
    \Ptrain[AX_0(\eps-X_n^\top(\what\beta_1-\beta_n\opt))]\\
  &\quad+s\Peval\left[\frac A{\pi_n\opt}
    (\eps-X_n^\top(\what\beta_1-\beta_n\opt))\right]\\
  &\quad+\Op(n^{-1})
    +\Op\left(\frac{|s|\sqrt{P[(\pi_n\opt)^{-1}]}}n\right).
\end{aligned}
  \label{eq:substituted-expansion}
\end{equation}
The two remainder bounds in~\eqref{eq:substituted-expansion} account for
the product
\[
  (\what\alpha_1-\alpha\opt)^\top
  \left((\Peval-P)X_0
    -s\big(\Peval[(A/\pi_n\opt)X_0]-P[X_0]\big)\right).
\]
This product is \(\Op(n^{-1})+\Op(|s|\sqrt{P[(\pi_n\opt)^{-1}]}/n)\), since \(\what\alpha_1-\alpha\opt=\Op(n^{-1/2})\),
\((\Peval-P)X_0=\Op(n^{-1/2})\), and the weighted mean satisfies
\eqref{eq:weighted-llns}. The product is negligible at the root-\(n\) scale
under sufficient overlap, and relative to
\(\sqrt{P[(\pi_n\opt)^{-1}]/n}\) under deteriorating overlap.

The term multiplied by \(1-s\) is computed on \(\Ptrain\), whereas
the term multiplied by \(s\) is computed on \(\Peval\). Both use
\(\what\beta_1\), which is estimated on \(\Ptrain\).

Next exchange 1 and 2 and \(\Ptrain,\Peval\) in
\eqref{eq:substituted-expansion} to obtain \(\what\psi_2(s)-\psi\).
Since \(\what\psi(s)=(\what\psi_1(s)+\what\psi_2(s))/2\), averaging gives
\begin{equation}
\begin{aligned}
  \what\psi(s)-\psi
  &=(\alpha\opt)^\top\left(\frac{\Ptrain+\Peval}{2}-P\right)X_0
     +\frac{\what\beta_2^\top\Ptrain X_n+\what\beta_1^\top\Peval X_n}{2}\\
  &\quad+\frac{1-s}{2}P[X_0]^\top\what\Sigma_{1,0}^{-1}
    \Ptrain[AX_0(\eps-X_n^\top(\what\beta_1-\beta_n\opt))]\\
  &\quad+\frac{s}{2}\Ptrain\left[\frac A{\pi_n\opt}
    (\eps-X_n^\top(\what\beta_2-\beta_n\opt))\right]\\
  &\quad+\frac{1-s}{2}P[X_0]^\top\what\Sigma_{2,0}^{-1}
    \Peval[AX_0(\eps-X_n^\top(\what\beta_2-\beta_n\opt))]\\
  &\quad+\frac{s}{2}\Peval\left[\frac A{\pi_n\opt}
    (\eps-X_n^\top(\what\beta_1-\beta_n\opt))\right]\\
  &\quad+\Op(n^{-1})
    +\Op\left(\frac{|s|\sqrt{P[(\pi_n\opt)^{-1}]}}n\right).
\end{aligned}
  \label{eq:averaged-expansion}
\end{equation}
\paragraph{Decomposing the inverse-propensity weights.}
We now analyze the two residual terms computed on \(\Ptrain\) in
\eqref{eq:averaged-expansion}; exchanging the folds will give the
corresponding calculation for \(\Peval\). These two terms are one half of
\begin{equation}
\begin{aligned}
  (1-s)P[X_0]^\top\what\Sigma_{1,0}^{-1}
  \Ptrain[AX_0(\eps-X_n^\top(\what\beta_1-\beta_n\opt))]
  +s\Ptrain\left[\frac A{\pi_n\opt}
       (\eps-X_n^\top(\what\beta_2-\beta_n\opt))\right].
\end{aligned}
  \label{eq:training-residual-terms}
\end{equation}
The two residual averages are correlated, so we first combine them before
calculating their variance. The first already has weights linear in
\(X_0\). We decompose the inverse-propensity weights in the second into a
linear combination of \(X_0\) and an orthogonal remainder.

Since
\[
  \frac1n\mathbb X_{1,0}^\top\mathbb X_{1,0}=\what\Sigma_{1,0},
  \qquad
  \frac1n\mathbb X_{1,0}^\top\boldsymbol\pi_1^{-1}
    =\Ptrain[(A/\pi_n\opt)X_0],
\]
the projection formula gives
\[
  \boldsymbol\pi_1^{-1}
    =\mathbb X_{1,0}\what\Sigma_{1,0}^{-1}
       \Ptrain[(A/\pi_n\opt)X_0]+\boldsymbol r_1,
  \qquad
  \boldsymbol r_1=(\Id-\Pi_1^0)\boldsymbol\pi_1^{-1}.
\]
In particular, \(\mathbb X_{1,0}^\top\boldsymbol r_1=0\).
Multiplying the decomposition of \(\boldsymbol\pi_1^{-1}\) on the left by
\((\boldsymbol\eps_1-\mathbb X_{1,n}(\what\beta_2-\beta_n\opt))^\top/n\)
gives the weighted residual average
\begin{equation}
\begin{aligned}
  \Ptrain\left[\frac A{\pi_n\opt}
      (\eps-X_n^\top(\what\beta_2-\beta_n\opt))\right]
  &=\Ptrain[(A/\pi_n\opt)X_0]^\top\what\Sigma_{1,0}^{-1}\\
  &\quad\cdot\Ptrain[AX_0(\eps-X_n^\top(\what\beta_2-\beta_n\opt))]+E_1,
\end{aligned}
  \label{eq:weighted-residual-decomposition}
\end{equation}

Here \(E_1\) is the remainder defined in \Cref{lem:estimator-expansion}.
Equation~\eqref{eq:weighted-residual-decomposition} expresses the inverse-propensity
weighted residual as an \(X_0\)-weighted residual plus \(E_1\).
To combine it with the other residual average
in~\eqref{eq:training-residual-terms}, we must replace
\(\Ptrain[(A/\pi_n\opt)X_0]\) by \(P[X_0]\). We bound this error next.

\paragraph{Replacing the weighted \(X_0\) mean.}
The difference between \(\Ptrain[(A/\pi_n\opt)X_0]\) and \(P[X_0]\)
is multiplied by a centered Gaussian residual average. This average has
order \(\Op(n^{-1/2})\), which supplies an additional factor of
\(n^{-1/2}\) in the replacement error.
More precisely, \eqref{eq:weighted-llns} gives
\[
  \Ptrain[(A/\pi_n\opt)X_0]-P[X_0]
    =\Op\left(\sqrt{\frac{P[(\pi_n\opt)^{-1}]}n}\right).
\]
For the residual average, condition on \(\mathcal A\) and the observations
in \(\Peval\). This fixes \(\what\beta_2\), while the \(X_n\) and
\(\eps\) values in \(\Ptrain\) remain independent Gaussians. Consequently
\(\Ptrain[AX_0(\eps-X_n^\top(\what\beta_2-\beta_n\opt))]\)
is centered Gaussian with covariance
\[
  \frac{\sigma^2+\norm{\what\beta_2-\beta_n\opt}^2}{n}
       \what\Sigma_{1,0}=\Op(n^{-1}),
\]
where we used \Cref{lem:coefficient-limits}. The residual average is
therefore \(\Op(n^{-1/2})\). Since
\(\opnorm{\what\Sigma_{1,0}^{-1}}=\Op(1)\), replacing
\(\Ptrain[(A/\pi_n\opt)X_0]\) by \(P[X_0]\) in
\eqref{eq:weighted-residual-decomposition} changes the weighted residual average by
\[
\begin{aligned}
  &\big(\Ptrain[(A/\pi_n\opt)X_0]-P[X_0]\big)^\top
  \what\Sigma_{1,0}^{-1}
  \Ptrain[AX_0(\eps-X_n^\top(\what\beta_2-\beta_n\opt))]\\
  &\qquad=\Op\left(\sqrt{\frac{P[(\pi_n\opt)^{-1}]}n}\right)
       \Op(n^{-1/2})
   =\Op\left(\frac{\sqrt{P[(\pi_n\opt)^{-1}]}}n\right).
\end{aligned}
\]
Under sufficient overlap, replacing the weighted \(X_0\) mean therefore
changes the residual average by \(\Op(n^{-1})\), which is negligible
at the root-\(n\) scale. Under deteriorating overlap, the replacement
error has order \(\Op(\sqrt{P[(\pi_n\opt)^{-1}]}/n)\).

\paragraph{Combining the two residual averages.}
The two terms involving \(X_0\) now have the same weights,
\(A P[X_0]^\top\what\Sigma_{1,0}^{-1}X_0\). Their residuals combine as
\[
\begin{aligned}
  &(1-s)\big(\eps-X_n^\top(\what\beta_1-\beta_n\opt)\big)
     +s\big(\eps-X_n^\top(\what\beta_2-\beta_n\opt)\big)\\
  &\qquad=\eps-X_n^\top\big((1-s)\what\beta_1+s\what\beta_2-\beta_n\opt\big).
\end{aligned}
\]
This is the residual in the definition of \(Z_1(s)\). Thus
\eqref{eq:training-residual-terms}, together with
\(\what\beta_2^\top\Ptrain X_n\), equals \(Z_1(s)+sE_1\) up to
\(\Op(|s|\sqrt{P[(\pi_n\opt)^{-1}]}/n)\). Exchanging the folds and
substituting into~\eqref{eq:averaged-expansion} proves~\eqref{eq:two-fold}.
This completes the proof of \Cref{lem:estimator-expansion}.

\subsubsection{Proof of \Cref{lem:loo}}
\label{section:one-row-proof}

In the ridge and underparameterized least-squares regimes, subtracting the
normal equations before and after removing the first rotated row gives
\[
  \what\beta_1-\what\beta_1^-
  =(G_1^\top G_1+n\lambda\Id)^{-1}g_1
       (\eps_1^\circ-g_1^\top(\what\beta_1^--\beta_n\opt)).
\]
Here \(G_1\) includes the first row \(g_1^\top\).
The inverse has operator norm \(O(n^{-1})\) when \(\lambda>0\), and~\(\Op(n^{-1})\) when \(\lambda=0\), \(\gamma<1\), by the Bai--Yin law
\citep[Theorem~5.11]{baisilverstein2010}. Since \(\norm{g_1}=\Op(\sqrt n)\),
\((G_1^\top G_1+n\lambda\Id)^{-1}g_1\) has norm \(\Op(n^{-1/2})\).

For the minimum-norm interpolator, let \(\Pi_1^-\) project onto the
nullspace of \(G_1^-\). To satisfy the equation for the deleted row,
we add a multiple of \(\Pi_1^-g_1\):
\[
  \what\beta_1-\what\beta_1^-
  =\frac{\Pi_1^-g_1}{g_1^\top\Pi_1^-g_1}
       (\eps_1^\circ-g_1^\top(\what\beta_1^--\beta_n\opt)).
\]
The difference is annihilated by \(G_1^-\) and belongs to
\(\operatorname{range}(G_1^\top)\). The updated coefficient vector
therefore satisfies all training equations and has minimum norm. Conditional on \(G_1^-\), the denominator is chi-squared
with \(d_n-m_1+1\) degrees of freedom. Hence
\(\Pi_1^-g_1/(g_1^\top\Pi_1^-g_1)\) has norm \(\Op(n^{-1/2})\).

Deleting one row leaves \(d_n/(m_1-1)\cp\gamma\) and the effective
penalty unchanged in the limit. Hence \Cref{lem:coefficient-limits}
also applies to \(\what\beta_1^-\),
so \(\norm{\what\beta_1^--\beta_n\opt}=\Op(1)\).
Conditional on the retained rows, the scalar
\(\eps_1^\circ-g_1^\top(\what\beta_1^--\beta_n\opt)\) is centered Gaussian with
variance \(\sigma^2+\norm{\what\beta_1^--\beta_n\opt}^2\), hence \(\Op(1)\).
Multiplying this \(\Op(1)\) scalar by the \(\Op(n^{-1/2})\) vector
proves \(\norm{\what\beta_1-\what\beta_1^-}=\Op(n^{-1/2})\).
Cauchy--Schwarz then shows that deleting a row changes each squared norm
and inner product in~\eqref{eq:three-limits} by \(\op(1)\).

\subsection{C-Learner and TMLE multipliers and variance comparison}
\label{section:proof-of-main}
\label{section:proof-of-comparison}
\label{section:proof-of-multiplier}

For fixed \(s\), \Cref{thm:family-clt} gives the limiting variance
\(V(s)\). The C-Learner and TMLE instead estimate~\(s\) from the data.
To apply the fixed-\(s\) CLT, we must determine how replacing an estimated
multiplier by a constant changes the estimate. By \Cref{lem:reduction},
\[
\begin{gathered}
  \what\psi_{1,\CL}-\what\psi_{1,\DM}
    =\what s_{1,\CL}\Peval\left[\frac A{\pi_n\opt}(Y-\what\mu_1(X))\right],\\
  \Peval\left[\frac A{\pi_n\opt}(Y-\what\mu_1(X))\right]
    =\Op\left(\sqrt{\frac{P[(\pi_n\opt)^{-1}]}n}\right).
\end{gathered}
\]
We will show that \(\what s_{1,\TMLE}\cp1\) and
\(\what s_{1,\CL}\) is asymptotic to
\((1+\sigma_X^2\mathsf{adj}(\gamma,\lambda))/
(P[(\pi_n\opt)^{-1}]+\sigma_X^2\mathsf{adj}(\gamma,\lambda))\).
Under sufficient overlap, both multipliers converge to constants;
\Cref{section:proof-of-regular} shows that replacing them by those
constants changes the estimates by \(\op(n^{-1/2})\).
Under deteriorating overlap, the C-Learner multiplier tends to zero at
rate \(1/P[(\pi_n\opt)^{-1}]\). In \Cref{section:proof-of-degenerate},
we show that multiplying by \(\Peval[(A/\pi_n\opt)(Y-\what\mu_1(X))]\) gives a correction of order
\(1/\sqrt{nP[(\pi_n\opt)^{-1}]}\), whereas the TMLE retains the AIPW
correction of order \(\sqrt{P[(\pi_n\opt)^{-1}]/n}\).

\paragraph{The multiplier calculation for both overlap scales.}
The TMLE multiplier is
\[
  \what s_{1,\TMLE}=\frac{\Peval[(\pi_n\opt)^{-1}]}
    {\Peval[A/(\pi_n\opt)^2]}.
\]
Dividing both averages by \(P[(\pi_n\opt)^{-1}]\) gives a limit of one
for the numerator and denominator by~\eqref{eq:weighted-llns}, so
\(\what s_{1,\TMLE}\cp1\). For the C-Learner, we compute
\(c_1^\top\mathcal I_1^{-1}b_1\) and
\(b_1^\top\mathcal I_1^{-1}b_1\) when \(\mathcal I_1\) is invertible.
Under interpolation, we instead compute \(u_1^\top\Pi_1v_1\) and
\(v_1^\top\Pi_1v_1\) from~\eqref{eq:s1-minnorm}. We use the following Gaussian variance formula to show that the numerator
and denominator concentrate around their conditional means.
\begin{lemma}[Gaussian bilinear forms]
  \label{lem:bilinear}
  Let \(R\) be an \(\mc F\)-measurable symmetric matrix, and let \(u,v\)
  be conditionally centered jointly Gaussian with covariance matrices
  \(\tau_u\Id/n\), \(\tau_v\Id/n\), and cross covariance~\(\tau_{uv}\Id/n\). Then
  \[
    \E[u^\top Rv\mid\mc F]=\frac{\tau_{uv}}n\tr R,\qquad
    \var(u^\top Rv\mid\mc F)
       =\frac{\tau_u\tau_v+\tau_{uv}^2}{n^2}\tr(R^2).
  \]
\end{lemma}
\begin{proof-of-lemma}[\ref{lem:bilinear}]
The mean follows from \(\E[vu^\top\mid\mc F]=\tau_{uv}\Id/n\).
For the variance, Wick's formula gives
\[
  \sum_{i,j,k,l}R_{ij}R_{kl}
  \big(\E[u_iu_k\mid\mc F]\E[v_jv_l\mid\mc F]
     +\E[u_iv_l\mid\mc F]\E[u_kv_j\mid\mc F]\big),
\]
which equals the stated expression.
\end{proof-of-lemma}

\paragraph{The numerator and denominator when \(\mathcal I_1\) is invertible.}
First suppose \(\lambda>0\), or \(\lambda=0\) and~\(\gamma<1\).
Let \(R=(G_1^\top G_1/n+\lambda\Id)^{-1}\).
The block inverse formula for \(\mathcal I_1\) gives
\begin{equation}
  c_1^\top\mathcal I_1^{-1}b_1
      =(c_{1,0})^\top\what\Sigma_{1,0}^{-1}b_{1,0}+u_1^\top Rv_1,\qquad
  b_1^\top\mathcal I_1^{-1}b_1
      =(b_{1,0})^\top\what\Sigma_{1,0}^{-1}b_{1,0}+v_1^\top Rv_1,
  \label{eq:profiled-inverse}
\end{equation}
with \(u_1,v_1\) from~\eqref{eq:uv}.
Condition on \(\mathcal A,G_1,e_1\), so \(R\) is fixed.
The vectors \(b_{1,n},c_{1,n}\) depend only on evaluation observations,
whereas \(C_1=\Ptrain[AX_nX_0^\top]\) depends on
\(\Pi_1^0\mathbb X_{1,n}\). Thus \((b_{1,n},c_{1,n})\) and~\(C_1\)
are independent centered Gaussians, independent of \((G_1,e_1)\)
by \Cref{lem:centering}. Substituting the definitions of
\(u_1,v_1\) in~\eqref{eq:uv} gives the conditional covariances
\begin{equation}
\begin{aligned}
  n\cov(u_1)&=\big(1+(c_{1,0})^\top\what\Sigma_{1,0}^{-1}c_{1,0}\big)\Id,\\
  n\cov(u_1,v_1)&=\big(\Peval[A/\pi_n\opt]
                    +(c_{1,0})^\top\what\Sigma_{1,0}^{-1}b_{1,0}\big)\Id,\\
  n\cov(v_1)&=\big(\Peval[A/(\pi_n\opt)^2]
                    +(b_{1,0})^\top\what\Sigma_{1,0}^{-1}b_{1,0}\big)\Id.
\end{aligned}
  \label{eq:mean-cov}
\end{equation}
The scalar factors in \(n\cov(u_1)\) and \(n\cov(u_1,v_1)\) converge
to \(1+\sigma_X^2\). The factor in \(n\cov(v_1)\) equals
\(P[(\pi_n\opt)^{-1}]+\sigma_X^2+\op(P[(\pi_n\opt)^{-1}])\), so it can
grow as overlap decreases.

\paragraph{Concentrating the numerator and denominator.}
For least squares, the Bai--Yin law
\citep[Theorem~5.11]{baisilverstein2010} gives
\(\lambda_{\min}(G_1^\top G_1/m_1)\cp(1-\sqrt\gamma)^2>0\).
Since \(m_1/n\cp\lim_{n\to\infty}P[\pi_n\opt]>0\), it follows that
\(\opnorm R=\Op(1)\). For ridge, this bound follows directly from
\(\opnorm R\le1/\lambda\).
The Marchenko--Pastur law
\citep[Theorem~3.6]{baisilverstein2010}, together with the lower eigenvalue
bound when \(\lambda=0\), gives
\[
  \frac1n\tr R\cp\eta,
  \qquad
  \eta=\begin{cases}
       \gamma/(1-\gamma),&\lambda=0,\ \gamma<1,\\
       \gamma m_\gamma(t),&\lambda>0.
       \end{cases}
\]
Because \(P[(\pi_n\opt)^{-1}]=o(n)\), \Cref{lem:bilinear} gives
\[
  u_1^\top Rv_1\cp(1+\sigma_X^2)\eta,
  \qquad
  \frac{v_1^\top Rv_1}
       {P[(\pi_n\opt)^{-1}]+\sigma_X^2}\cp\eta.
\]
Specifically, \(\var(u_1^\top Rv_1\mid\mathcal A,G_1,e_1)
=\Op(P[(\pi_n\opt)^{-1}]/n)=\op(1)\).
Dividing \(v_1^\top Rv_1\) by
\(P[(\pi_n\opt)^{-1}]+\sigma_X^2\) makes its conditional variance
\(\Op(n^{-1})\). We normalize \(v_1^\top Rv_1\) because its conditional
mean grows with \(P[(\pi_n\opt)^{-1}]\).

Because \(c_{1,0}\cp P[X_0]\), \(b_{1,0}\cp P[X_0]\), and
\(\what\Sigma_{1,0}\cp\Sigma\), the two quadratic forms in
\(c_{1,0},b_{1,0}\) in~\eqref{eq:profiled-inverse} converge to
\(P[X_0]^\top\Sigma^{-1}P[X_0]=\sigma_X^2\). Adding the limits of
\(u_1^\top Rv_1\) and \(v_1^\top Rv_1\) gives
\[
  c_1^\top\mathcal I_1^{-1}b_1\cp\sigma_X^2+(1+\sigma_X^2)\eta,
  \qquad
  \frac{b_1^\top\mathcal I_1^{-1}b_1}
       {P[(\pi_n\opt)^{-1}]+\sigma_X^2(1+1/\eta)}\cp\eta.
\]
Since \(\mathsf{adj}(\gamma,\lambda)=1+1/\eta\), this proves
\begin{equation}
  \frac{P[(\pi_n\opt)^{-1}]+\sigma_X^2\mathsf{adj}(\gamma,\lambda)}
       {1+\sigma_X^2\mathsf{adj}(\gamma,\lambda)}\what s_{1,\CL}\cp1.
  \label{eq:s-limit}
\end{equation}
It also verifies the adjustment constants in the underparameterized and
ridge rows of \Cref{tab:regime-constants}.

\paragraph{Replacing the inverse by a nullspace projection.}
Under interpolation, the multiplier is
\(u_1^\top\Pi_1v_1/(v_1^\top\Pi_1v_1)\) by~\eqref{eq:s1-minnorm}.
Apply \Cref{lem:bilinear} with \(R=\Pi_1\) and the conditional
covariances in~\eqref{eq:mean-cov}. The projection satisfies
\[
  \opnorm{\Pi_1}=1,\qquad
  \frac{\tr\Pi_1}{nP[\pi_n\opt]}=\frac{d_n-m_1}{nP[\pi_n\opt]}\cp\gamma-1>0.
\]
The numerator \(u_1^\top\Pi_1v_1\) is asymptotic to
\((1+\sigma_X^2)\tr(\Pi_1)/n\). The denominator \(v_1^\top\Pi_1v_1\)
is asymptotic to \((P[(\pi_n\opt)^{-1}]+\sigma_X^2)\tr(\Pi_1)/n\).
Cancelling
\(\tr(\Pi_1)/n\) gives
\[
  \frac{P[(\pi_n\opt)^{-1}]+\sigma_X^2}{1+\sigma_X^2}\what s_{1,\CL}\cp1.
\]
This is~\eqref{eq:s-limit} with \(\mathsf{adj}(\gamma,0)=1\).
Here \(\what\beta_1^C-\what\beta_1\in\operatorname{range}(\Pi_1)\),
so~\eqref{eq:s1-minnorm} expresses the multiplier directly as
\(u_1^\top\Pi_1v_1/(v_1^\top\Pi_1v_1)\).
Exchanging training and evaluation folds proves~\eqref{eq:s-limit}
for \(\what s_{2,\CL}\) and gives~\(\what s_{2,\TMLE}\cp1\).
We can now replace each estimated multiplier by a constant under
sufficient overlap, or use its order under deteriorating overlap.

\subsubsection{Part \ref{item:sufficient-overlap}: sufficient overlap}
\label{section:proof-of-regular}

We will show that TMLE differs from AIPW by \(\op(n^{-1/2})\), and
that the C-Learner has the fixed-multiplier limit at
\((1+\sigma_X^2\mathsf{adj}(\gamma,\lambda))/
(\kappa+\sigma_X^2\mathsf{adj}(\gamma,\lambda))\). We then compare
\(V(s)\) at this multiplier and at the AIPW multiplier \(s=1\).

\paragraph{Replacing the estimated multipliers by their limits.}
Suppose \(P[(\pi_n\opt)^{-1}]\) converges to the finite limit
\(\kappa\) in \Cref{thm:comparison}. Specializing~\eqref{eq:s-limit} to this finite \(\kappa\) gives
\[
  \what s_{j,\TMLE}\cp1,\qquad
  \what s_{j,\CL}\cp
    \frac{1+\sigma_X^2\mathsf{adj}(\gamma,\lambda)}
         {\kappa+\sigma_X^2\mathsf{adj}(\gamma,\lambda)}.
\]
Since \(\what\psi_1(s)\) is affine in \(s\) and
\(\Peval[(A/\pi_n\opt)(Y-\what\mu_1(X))]=\Op(n^{-1/2})\),
with the same bound after exchanging folds,
\[
  \what\psi_{\TMLE}-\what\psi(1)=\op(n^{-1/2}),\qquad
  \what\psi_{\CL}
   -\what\psi\left(\frac{1+\sigma_X^2\mathsf{adj}(\gamma,\lambda)}
                          {\kappa+\sigma_X^2\mathsf{adj}(\gamma,\lambda)}\right)
     =\op(n^{-1/2}).
\]
\Cref{thm:family-clt} and Slutsky's theorem now give all four limiting
variances in \Cref{thm:comparison}.

\paragraph{Comparing the limiting variances.}
Since \(X_0\) includes an intercept, Cauchy--Schwarz applied to \(\Sigma^{-1/2}P[X_0]\) and \(\Sigma^{1/2}e_{\rm int}\), where \(e_{\rm int}:=(1,0,\ldots,0)^\top\), gives
\(\sigma_X^2\ge 1/(\lim_{n\to\infty}P[\pi_n\opt])>1\).
Also \(\kappa\ge\sigma_X^2\), since~\eqref{eq:wperp} identifies
\(\kappa-\sigma_X^2\) as the limit of \(\norm{\boldsymbol r_1}^2/n\).
The coefficient of \(s^2\) in \(V(s)\) is
\(2\sigma_X^2\sigma_{\mathrm{fit}}^2+(\kappa-\sigma_X^2)\sigma_{\mathrm{res}}^2>0\),
because \(\sigma_{\mathrm{fit}}^2>0\) by \Cref{lem:coefficient-limits}.
Hence~\(V(s)\) is strictly convex, with its unique minimizer given
by~\eqref{eq:s-opt}. Since \(\mathsf{adj}(\gamma,\lambda)\ge1\), with equality precisely
under minimum-norm interpolation, and \(\kappa\ge\sigma_X^2>1\),
\[
  \frac{1+\sigma_X^2\mathsf{adj}(\gamma,\lambda)}
       {\kappa+\sigma_X^2\mathsf{adj}(\gamma,\lambda)}
   \ge\frac{1+\sigma_X^2}{\kappa+\sigma_X^2}
   \ge\argmin_{s\in\R}V(s).
\]
The first inequality is strict unless \(\lambda=0\) and \(\gamma>1\), since
\(x\mapsto(1+x)/(\kappa+x)\) is strictly increasing. The second is strict
unless \(\kappa=\sigma_X^2\), by~\eqref{eq:fit-positivity} and
\eqref{eq:s-opt}. The C-Learner's multiplier limit is also strictly
below one. Strict convexity therefore gives \(V_\CL<V(1)\), with
\(V_\CL=\min_sV(s)\) exactly when~\(\kappa=\sigma_X^2\), \(\lambda=0\), and \(\gamma>1\). Completing the square also shows that \(V_\CL<V_\DM\)
if and only if the C-Learner's multiplier limit is less than twice the
minimizer in~\eqref{eq:s-opt}.
Thus the comparison with the direct method need not have a fixed sign.

\paragraph{The condition \(\kappa=\sigma_X^2\).}
For each \(n\), minimizing over linear combinations of \(X_0\) gives
\[
\begin{aligned}
  \inf_{\alpha\in\R^q}P\left[\pi_n\opt
    \left(\frac1{\pi_n\opt}-X_0^\top\alpha\right)^2\right]
  &=P[(\pi_n\opt)^{-1}]
    -P[X_0]^\top\big(P[\pi_n\opt X_0X_0^\top]\big)^{-1}P[X_0]\\
  &\longrightarrow\kappa-\sigma_X^2.
\end{aligned}
\]
Thus \(\kappa=\sigma_X^2\) exactly when this weighted approximation
error tends to zero. This completes the proof of
\Cref{thm:comparison}\ref{item:sufficient-overlap}.

\subsubsection{Part \ref{item:deteriorating-overlap}: deteriorating overlap}
\label{section:proof-of-degenerate}

When \(P[(\pi_n\opt)^{-1}]\to\infty\), equation~\eqref{eq:s-limit}
gives \(\what s_{1,\CL}=\Op(1/P[(\pi_n\opt)^{-1}])\), whereas
\(\what s_{1,\TMLE}\cp1\). We will show that \(\what s_{1,\CL}\Peval[(A/\pi_n\opt)(Y-\what\mu_1(X))]\)
is \(\op(n^{-1/2})\), so the C-Learner shares the direct-method CLT.
For AIPW and TMLE, we will derive the Gaussian limit after multiplication
by~\(\sqrt{2n/P[(\pi_n\opt)^{-1}]}\).

\paragraph{The C-Learner's first-order correction vanishes.}
Now suppose \(P[(\pi_n\opt)^{-1}]\to\infty\). The multiplier
limit~\eqref{eq:s-limit} and the first-order correction rate~\eqref{eq:residual-rate} give
\[
\begin{aligned}
  \what\psi_{1,\CL}-\what\psi_{1,\DM}
  &=\what s_{1,\CL}\Peval\left[\frac A{\pi_n\opt}(Y-\what\mu_1(X))\right]\\
  &=\Op\left(\frac1{\sqrt{nP[(\pi_n\opt)^{-1}]}}\right)
   =\op(n^{-1/2}).
\end{aligned}
\]
Exchanging folds and averaging gives
\(\what\psi_{\CL}-\what\psi_{\DM}=\op(n^{-1/2})\).
The direct-method CLT proved in \Cref{section:proof-of-clt} continues to
hold, so \(\sqrt{2n}(\what\psi_{\DM}-\psi)\) and
\(\sqrt{2n}(\what\psi_{\CL}-\psi)\) have the same Gaussian limit jointly.

\paragraph{The first-order correction determines the slower AIPW limit.}
Take \(s=1\) in~\eqref{eq:substituted-expansion} and substitute
\(\boldsymbol\pi_2^{-1}=\Pi_2^0\boldsymbol\pi_2^{-1}+\boldsymbol r_2\)
into \(\Peval[(A/\pi_n\opt)(\eps-X_n^\top(\what\beta_1-\beta_n\opt))]\). Every term
other than \(E_2\) has order
\(\Op(n^{-1/2})+\Op(\sqrt{P[(\pi_n\opt)^{-1}]}/n)\).
The \(\Op(n^{-1/2})\) and
\(\Op(\sqrt{P[(\pi_n\opt)^{-1}]}/n)\) remainders are negligible relative to
\(\sqrt{P[(\pi_n\opt)^{-1}]/n}\). By~\eqref{eq:H-minus},
\[
  \sqrt{\frac n{P[(\pi_n\opt)^{-1}]}}
   \begin{pmatrix}\what\psi_{1,\AIPW}-\psi\\
                   \what\psi_{2,\AIPW}-\psi\end{pmatrix}
  =\sqrt{\frac n{P[(\pi_n\opt)^{-1}]}}
      \begin{pmatrix}E_2^-\\E_1^-\end{pmatrix}+\op(1).
\]
Given \(\mc F^-\), the two entries on the right are independent centered
Gaussians. Their conditional variances converge to
\(\sigma_{\mathrm{res}}^2(\gamma,\lambda)\) by~\eqref{eq:H-var}, since \(\sigma_X^2\) remains
finite. Therefore the left side converges to
\(\Normal(0,\sigma_{\mathrm{res}}^2(\gamma,\lambda)\Id_2)\).
Averaging \(\what\psi_{1,\AIPW}\) and
\(\what\psi_{2,\AIPW}\) therefore gives
\(\sqrt{2n/P[(\pi_n\opt)^{-1}]}(\what\psi_\AIPW-\psi)
\cd\Normal(0,\sigma_{\mathrm{res}}^2(\gamma,\lambda))\).

\paragraph{TMLE and AIPW are equivalent at the slower scale.}
Finally,
\[
\begin{aligned}
  \what\psi_{1,\TMLE}-\what\psi_{1,\AIPW}
  &=(\what s_{1,\TMLE}-1)\Peval\left[\frac A{\pi_n\opt}(Y-\what\mu_1(X))\right]\\
  &=\op\left(\sqrt{\frac{P[(\pi_n\opt)^{-1}]}n}\right).
\end{aligned}
\]
Exchanging folds and averaging shows that the TMLE has the same limit
jointly with AIPW. This proves the
assertions of \Cref{thm:comparison}\ref{item:deteriorating-overlap}.

%%% Local Variables:
%%% mode: latex
%%% TeX-master: "../main"
%%% End:
     % proofs for Section 4
\section{A Low-Overlap Comparison}
\label{section:theory_simple}

The analysis in \Cref{section:high_dim} compares convergence rates as overlap
deteriorates while the inverse-propensity moment remains finite at every
sample size. Here we give a complementary finite-sample comparison when
\(P[1/\pi\opt(X)]=\infty\). A square-integrable envelope controls the
C-Learner's reported mean even in this regime. AIPW and squared-loss TMLE
can instead have infinite mean squared error with the same known propensity.
This result isolates the effect of propensity weighting without imposing the
Gaussian outcome model of \Cref{section:high_dim}.

Consider binary treatments $A\in\{0,1\}$, covariates $X\in\R^d$, and potential
outcomes $Y(1)$ and $Y(0):=0$. Let $(X,A,Y)$ be i.i.d. observations from $P$,
and let $\psi:=P[Y(1)]$ be finite. We write $P[\cdot]$ for expectation under
$P$ and $\|\cdot\|$ for the Euclidean norm. We study cross-fitted versions of
the three estimators, using the fold notation from \Cref{section:data_splitting}
and the true propensity $\pi\opt(X):=P(A=1\mid X)$ in place of a fitted
propensity. Each evaluation fold contains at least two units. On a fold with
no observed outcomes, we define TMLE to equal its initial plug-in estimate.

The first condition controls the output of the outcome models without requiring
them to be uniformly bounded.
\begin{assumption}[Square-integrable envelope]
\label{ass:envelope}
There is a function $C$ with $P[C(X)^2]<\infty$ and a set of covariate
values of probability one on which $|f(x)|\leq C(x)$ simultaneously for
every $f\in\mathcal F$.
\end{assumption}
\noindent For example, if $P[\|X\|^2]<\infty$, then every fixed linear model
$f_\beta(x)=\beta^\top x$ is square-integrable by Cauchy--Schwarz, and the
maximum absolute value of finitely many such models is also square-integrable.
Thus, a finite collection of linear models admits a square-integrable
envelope.  Although our implementations do not impose an envelope explicitly,
one can intersect a chosen model class with a sufficiently wide envelope.

The second condition requires a constrained solution to exist.
\begin{assumption}[Existence of a constrained solution]
\label{ass:soln_exists}
For a fixed $\epsilon>0$, a measurable C-Learner solution in $\mathcal F$
can be selected with probability at least $1-\epsilon$ over samples of size
$n$.
\end{assumption}
\noindent A sufficient condition for a nonempty feasible set is that
$\mathcal F$ is convex and contains
two constant functions that lie below and above all outcomes in the evaluation
fold, since the constraint has opposite signs at these two functions. For an
evaluation fold of size $m$, choose $a$ and $b$ so that
$P(Y\in[a,b])\geq(1-\epsilon)^{1/m}$, include the constants $a$ and $b$ in
$\mathcal F$, and enlarge an initial square-integrable envelope $C_0$ to
$C(x)=\max\{C_0(x),|a|,|b|\}$. Then all evaluation-fold outcomes lie in
$[a,b]$ with probability at least $1-\epsilon$, which ensures feasibility on
that event. Attainment and measurable selection of a minimizer are also
required by the assumption. For example, restricting to the convex hull of
a finite collection of models, including these constants, gives a compact
coefficient simplex on which the squared-loss objective is continuous.

The final condition specifies the low-overlap regime.
\begin{assumption}[Nonintegrable inverse propensity]
\label{ass:prop_not_square_int}
There is a constant $v_0>0$ such that
$\Var(Y\mid X,A=1)\geq v_0$ for every $X$ in the support of $P$. Moreover,
$\pi\opt(X)>0$ almost surely and
\[
P\left[\frac{1}{\pi\opt(X)}\right]=\infty.
\]
\end{assumption}
\noindent The variance lower bound prevents the residual variation from
vanishing where the propensity is small. The integrability condition allows
the propensity to approach zero substantially more often than under standard
overlap. For example, if $\pi\opt(X)\sim\mathrm{Beta}(a,1)$, then
$P[1/\pi\opt(X)]=\infty$ for $a\leq1$, including the uniform case $a=1$.

% Keep the hyperlink anchor on the page containing the theorem heading.
\Needspace{6\baselineskip}
\begin{theorem}
\label{thm:simple_variance}
  Suppose the true propensity is known and
  \Crefrange{ass:envelope}{ass:prop_not_square_int} hold. Then
  \[
    P[(\what\psi_{\AIPW}-\psi)^2]
    =P[(\what\psi_{\TMLE}-\psi)^2]=\infty.
  \]
  Extend the C-Learner by reporting zero when no measurable solution is
  selected. This estimator has finite mean squared error and agrees with
  the constrained estimator whenever a solution is selected. If this
  probability for the full procedure is at least \(1-\epsilon\), the two
  estimators therefore agree with at least that probability.
\end{theorem}
\noindent When the AIPW or TMLE mean exists, its variance is consequently
infinite. The mean squared error statement also covers cases where that
mean does not exist. The proof below applies to each directed estimate and
to the cross-fitted average. For multiple constrained fits, an existence
bound for each fit yields a bound for joint existence by a union
bound.

\subsection{Proof of \Cref{thm:simple_variance}}
\label{section:theory_simple_more}

We first consider one evaluation fold of size \(m\geq2\), conditional on
its training data. Write \(\what\mu\) for the initial outcome fit and
\(\what\mu^C\) for the constrained fit. Both belong to the class in
\Cref{ass:envelope}. All expectations below may take the value infinity.

\paragraph{The C-Learner.}
Whenever a constrained solution is selected, the reported mean satisfies
\[
  \left|\frac1m\sum_{i=1}^m\what\mu^C(X_i)\right|^2
  \leq\left(\frac1m\sum_{i=1}^m C(X_i)\right)^2
  \leq\frac1m\sum_{i=1}^m C(X_i)^2.
\]
The same bound holds for the extension by zero on the complement of this
event. Its expectation is at most \(P[C(X)^2]<\infty\), even though the
constrained fit depends on the evaluation data. Since \(\psi\) is finite,
this proves finite mean squared error. Jensen's inequality gives the same
conclusion for the cross-fitted average.

\paragraph{AIPW.}
Conditional on the training data and the evaluation covariates and
assignments, only the observed outcomes remain random. Their independence
and \Cref{ass:prop_not_square_int} give
\[
\begin{aligned}
  \var(\what\psi_{k,\AIPW}\mid X_{1:m},A_{1:m},\text{train})
  &=\frac1{m^2}\sum_{i=1}^m
    \frac{A_i}{\pi\opt(X_i)^2}\var(Y_i\mid X_i,A_i=1)\\
  &\geq\frac{v_0}{m^2}\sum_{i=1}^m
    \frac{A_i}{\pi\opt(X_i)^2}.
\end{aligned}
\]
Taking expectations over the evaluation data yields the lower bound
\(v_0P[1/\pi\opt(X)]/m=\infty\) for the conditional mean squared error
of this directed estimate, for every training sample outside a null set.

\paragraph{TMLE.}
The squared-loss fluctuation gives
\[
  \what\psi_{k,\TMLE}
  =\what{P}_{k,n}[\what\mu(X)]
   +\frac{\what{P}_{k,n}[1/\pi\opt(X)]}{\what{P}_{k,n}[A/\pi\opt(X)^2]}
       \what{P}_{k,n}\left[\frac A{\pi\opt(X)}(Y-\what\mu(X))\right]
\]
when the denominator is positive. For this calculation only, write
\(w=\what{P}_{k,n}[1/\pi\opt(X)]\) and
\(t=\what{P}_{k,n}[A/\pi\opt(X)^2]\). On \(t>0\), the preceding conditional
variance calculation gives
\begin{equation}
  \var(\what\psi_{k,\TMLE}\mid X_{1:m},A_{1:m},\text{train})
  \geq\frac{v_0}{m}\frac{w^2}{t}.
  \label{eq:low-overlap-tmle-variance}
\end{equation}
A small propensity at an \emph{unobserved} unit can make \(w\) large
without increasing \(t\). To formalize this, choose \(c>0\) with
\(P(\pi\opt(X)\geq c)>0\), which exists because \(\pi\opt(X)>0\)
almost surely. Consider the event
\[
  A_1=0,\qquad A_2=1,\qquad
  \pi\opt(X_i)\geq c\quad(2\leq i\leq m).
\]
On this event, \(0<t\leq c^{-2}\) and
\(w\geq1/(m\pi\opt(X_1))\). The event involving units \(2,\ldots,m\)
has positive probability and is independent of unit \(1\). Thus the
expectation of the lower bound in
\eqref{eq:low-overlap-tmle-variance} is infinite, since
\[
  P\left[\frac{1-A}{\pi\opt(X)^2}\right]
  =P\left[\frac{1-\pi\opt(X)}{\pi\opt(X)^2}\right]=\infty.
\]
For the last equality, restrict to \(\pi\opt(X)\leq1/2\) and use
\(P[1/\pi\opt(X)]=\infty\). This proves infinite mean squared error
for the directed TMLE without assigning an inverse to a zero denominator.

\paragraph{Cross-fitting.}
Dependence between directed estimates cannot remove these infinite second
moments. Fix one evaluation fold. There exist \(c>0\) and \(b<\infty\)
such that
\[
  P\big(A=1,\ \pi\opt(X)\geq c,\ |Y|\leq b,\ C(X)\leq b\big)>0.
\]
Indeed, \(P(A=1)>0\), and these events increase to \(\{A=1\}\) as
\(c\downarrow0\) and \(b\uparrow\infty\). Restrict all observations
outside the chosen fold to this event. Its probability is positive.
Every other directed AIPW or TMLE estimate is then bounded by a constant
depending only on \(b,c\), because its evaluation outcomes, inverse
propensities, and fitted values are bounded. This bound holds whatever
data were used to train those fits.

Conditional on these outside observations, the chosen fold still has
infinite mean squared error by the preceding arguments. The sum of the
other directed estimates is bounded on the same event, so it cannot cancel
that infinite second moment. For example,
\((x+y)^2\geq x^2/2-y^2\) when \(y\) is bounded. The cross-fitted average
therefore has infinite mean squared error as well. This completes the proof.

%%% Local Variables:
%%% mode: latex
%%% TeX-master: "../main"
%%% End:

%%% Local Variables:
%%% mode: latex
%%% TeX-master: "../main"
%%% End:
      % Appendix D: result and proof
\section{Additional Experimental Details}
\label{section:exp_details}
\begin{yfrevision}\yf{This appendix provides implementation details and
additional experimental results. \Cref{section:xgb-coverage-methods}
describes the confidence-interval and coverage calculations.
\Cref{section:kang-schafer-experiment-details} details the KS implementations
and additional comparisons, including sensitivity to overlap and propensity
truncation. \Cref{section:civilcomments-experiment-details} presents the
CivilComments training settings and results across overlap levels.
\ifdefined\msom\else
Finally, \Cref{section:ihdp_experiment} describes the IHDP data, boosted-tree
results with and without cross-fitting, and neural-network comparisons. To avoid confusion, we explicit mention 
Monte Carlo standard errors when quantifying uncertainty in the performance summaries over 1000 replications.
\fi}\end{yfrevision}
% \subsection{C-Learner with Gradient Boosted Regression Trees: Pseudocode}
% \label{section:xgb_pseudocode}

\begin{yfrevision}
\subsection{Confidence Interval Coverage}
\label{section:xgb-coverage-methods}
\yf{We evaluate coverage using the following approach. First we show how confidence intervals are computed using
cross-fitting, then we explain for each method how the standard error is computed for a given fit.
For evaluation fold $I_k$ of size $n_k$, center the method's working score
within that fold, obtaining $D_{ik}$, and use}
\begin{equation}
\yf{\widehat{\operatorname{se}}^2
 =\sum_k\left(\frac{n_k}{n}\right)^2
   \frac{\sum_{i\in I_k}D_{ik}^2}{n_k(n_k-1)},
 \qquad
 \widehat\psi\ \pm\ 1.96\,\widehat{\operatorname{se}}.}
\label{eq:coverage-fold-variance}
\end{equation}
\yf{The calculations use the full-precision normal quantile, rather than
its rounded value 1.96. A no-cross-fitting run has one fold. Combining
fold variances this way is a conventional plug-in approximation, not an
exact finite-sample independence assertion: swapped nuisance training
couples the folds.}

\yf{For a missing-outcome mean, write $m_i$ for the initial outcome
prediction, $w_i=A_i/\widehat\pi_i$, and $\widehat P_k$ for the fold mean.
The propensity is the same effective, possibly truncated prediction used
by the point estimator. Direct uses $m_i$, IPW uses $w_iY_i$, and AIPW
uses $m_i+w_i(Y_i-m_i)$, each centered within its evaluation fold.
For self-normalized estimators, define
$b_k=\widehat P_k[w]$, the weighted outcome mean
$\widehat\psi_{k,\mathrm{IPW-SN}}=\widehat P_k[wY]/b_k$, and
the residual correction $c_k=\widehat P_k[w(Y-m)]/b_k$. The ratio
delta method gives}
\begin{align*}
\yf{D_{ik}^{\mathrm{IPW-SN}}}
 &\yf{=\frac{w_i}{b_k}
       (Y_i-\widehat\psi_{k,\mathrm{IPW-SN}}),}\\
\yf{D_{ik}^{\mathrm{AIPW-SN}}}
 &\yf{=m_i-\widehat P_k[m]
       +\frac{w_i}{b_k}(Y_i-m_i-c_k).}
\end{align*}
\yf{For TMLE and C-Learner, we substitute the method's updated outcome
prediction for $m_i$ in the AIPW working score and center it.}

\yf{For the IHDP ATE, we use the mean-score formulas separately for
$a=1$ and $a=0$, with weights $A_i/\widehat\pi_i$ and
$(1-A_i)/(1-\widehat\pi_i)$ and the corresponding outcome predictions.
Further, we subtract the control score from the treated score \emph{before} taking
the sample variance, preserving the covariance between arm contributions.
For AIPW-SN both $b_k$ and $c_k$ are arm-specific.}

\yf{For XGB Direct, the fitted-value variance omits model-fitting
uncertainty and is a naive diagnostic. The linear KS calculation,
including Direct's OLS estimation term and the one-fold convention,
is given in \Cref{section:ks-linear-coverage-methods}.}

\yf{Coverage is the fraction of the 1000 intervals containing the
specified target; parentheses give its binomial Monte Carlo standard
error, $\{\widehat c(1-\widehat c)/1000\}^{1/2}$. Interval widths are
reported separately, so high coverage from wide intervals can be identified. 
Additional diagnostics use nominal 90\% intervals and a
pooled-centering sensitivity that includes between-fold point-estimate
variation.}
\end{yfrevision}

% Keep KS float-only pages top-aligned with a finite gap between tables.
% Flush all KS floats before restoring the publisher's defaults below.
\begingroup
\makeatletter
\setlength{\@fptop}{0pt}
\setlength{\@fpsep}{18pt}
\setlength{\@fpbot}{0pt plus 1fil}
\makeatother
\subsection{Experiments for \citet{KangSc07}}
\label{section:kang-schafer-experiment-details}
\ifdefined\pnas
\subsubsection{Data Generating Process}
\label{section:ks_dgp}
The \textit{true} outcome and treatment  mechanisms to depend on covariates $\xi \sim N(0, I) \in \mathbb R^4$ and $\varepsilon\sim N(0,1)$
\begin{align*}
Y(1) & = 210 + 27.4 \xi_{1} + 13.7 \xi_{2} + 13.7 \xi_{3} + 13.7 \xi_{4} + \varepsilon, \\
\pi\opt(\xi) & = \frac{\exp(-\xi_{1} + 0.5 \xi_{2} - 0.25 \xi_{3} - 0.1 \xi_{4})}{1+\exp(-\xi_{1} + 0.5 \xi_{2} - 0.25 \xi_{3} - 0.1 \xi_{4})}.
\end{align*}
The target is $P[Y(1)]=210$, whereas $P[Y(1)\mid A=1]\approx200$, so a naive average of the 
treated outcomes has bias approximately $-10$. For a random sample of 100 data points, 
the true propensity score can be as low as 1\% and as high as 95\%. Instead of observing $\xi$, the modeler observes
$$
X_{1} = \exp(\xi_{1}/2),\;X_{2} = \xi_{2}/(1+\exp(\xi_{1})) + 10,\;X_{3} = (\xi_{1}\xi_{3}/25 + 0.6)^3,\;X_{4} = (\xi_{2} + \xi_{4} + 20)^2.
$$
\else\fi
\begin{yfrevision}
\subsubsection{\yf{Linear Outcome Models}}
\label{section:ks_experiment_details}
\label{section:ks_more_results}

Here, $\what{P}_{\rm train}=\what{P}_{\rm eval}$, as consistent with the original 
paper by \citet{KangSc07}\footnote{Although sample splitting could be better, we have chosen to replicate \cite{KangSc07} as closely as possible.}. 
There is no $\what{P}_{\rm val}$ as there are no hyperparameters to tune.

For the direct method and initial outcome model for AIPW, self-normalized AIPW, and TMLE, we regress $Y$ on $X$ using labeled samples. 
We use all samples to fit a logistic regression for $A$. Both initial models omit the intercept. 
\begin{yfrevision}\yf{All KS results use the covariate construction in \Cref{section:ks}.}\end{yfrevision}

Squared-loss TMLE adds $\widehat\epsilon/\what\pi(X)$ to the initial outcome fit, where the residual constraint gives
\[
  \widehat\epsilon=
  \frac{\sum_i A_i\{Y_i-\what\mu(X_i)\}/\what\pi(X_i)}
       {\sum_i A_i/\what\pi(X_i)^2}.
\]

We run datasets with 200 and 1000 observations, with and without a truncation threshold of $5\%$. For each configuration, we generate 1000 seeds.
%The computation of the direct method, IPW, normalized IPW, AIPW and normalized AIPW areis trivial and follows directly from their formulas discussed in \Cref{section:methods_experiments}.

\begin{yfrevision}\yf{\Cref{tb:linear_simulation_full,tb:linear_simulation_clipped}
give bias, MAE, RMSE, and median absolute error over 1000
replications at each sample size, without a floor and with a common 5\%
lower propensity floor. Without truncation, C-Learner has the lowest
MAE and RMSE among AIPW, AIPW-SN, and TMLE, while the direct
method has lower errors. With truncation, C-Learner's MAE lies between
TMLE and AIPW-SN at both sample sizes. Parentheses report Monte Carlo
standard errors for all four metrics. For RMSE, we use the delta method
applied to the mean squared error. For median absolute error, we report
the sample standard deviation of the medians from 100,000 bootstrap
resamples of the 1000 replication-level absolute errors (random seed
20260916). No models are refitted.}\end{yfrevision}

% Full linear-model error summaries for the KS experiment.
% Generated by doc/experiments-revised/build_audited_tables.py
\begin{yfrevision}
\begin{table}[t]
\centering
\footnotesize
\setlength{\tabcolsep}{5pt}
\begin{tabular}{l*{4}{w{r}{4em}w{r}{4.8em}}}
\toprule
Method & \multicolumn{2}{c}{Bias} & \multicolumn{2}{c}{Mean Abs Err} & \multicolumn{2}{c}{RMSE} & \multicolumn{2}{c}{Median Abs Err} \\
\cmidrule(lr){1-1}\cmidrule(lr){2-3}\cmidrule(lr){4-5}\cmidrule(lr){6-7}\cmidrule(lr){8-9}
\multicolumn{9}{c}{$n=200$} \\
\yf{Direct} & -0.590 & (0.107) & 2.741 & (0.065) & 3.422 & (0.075) & 2.248 & (0.074) \\
\yf{IPW} & 17.267 & (1.708) & 23.129 & (1.637) & 56.669 & (6.379) & 9.463 & (0.429) \\
\yf{IPW-SN} & 3.045 & (0.270) & 5.253 & (0.233) & 9.046 & (0.645) & 3.115 & (0.130) \\
\greymidrule
\yf{AIPW} & -4.453 & (0.324) & 5.634 & (0.305) & 11.177 & (1.143) & 3.389 & (0.116) \\
\yf{AIPW-SN} & -3.561 & (0.195) & 4.737 & (0.168) & 7.121 & (0.352) & 3.333 & (0.108) \\
\yf{TMLE} & -135.691 & (50.381) & 137.957 & (50.374) & 1598.147 & (480.076) & 3.948 & (0.160) \\
\yf{C-Learner} & -2.611 & (0.127) & 3.789 & (0.093) & 4.802 & (0.109) & 3.091 & (0.129) \\
\midrule
\multicolumn{9}{c}{$n=1000$} \\
\yf{Direct} & -0.991 & (0.046) & 1.409 & (0.033) & 1.754 & (0.037) & 1.191 & (0.039) \\
\yf{IPW} & 66.243 & (26.092) & 66.619 & (26.091) & 827.336 & (399.921) & 16.237 & (0.851) \\
\yf{IPW-SN} & 6.723 & (0.334) & 6.968 & (0.329) & 12.524 & (0.956) & 3.956 & (0.185) \\
\greymidrule
\yf{AIPW} & -26.567 & (12.917) & 26.609 & (12.917) & 409.131 & (197.837) & 4.708 & (0.122) \\
\yf{AIPW-SN} & -7.830 & (0.427) & 7.864 & (0.426) & 15.593 & (1.623) & 4.535 & (0.109) \\
\yf{TMLE} & -723.680 & (710.192) & 723.689 & (710.192) & $2.25\times10^{4}$ & ($1.12\times10^{4}$) & 4.375 & (0.094) \\
\yf{C-Learner} & -4.358 & (0.083) & 4.387 & (0.082) & 5.095 & (0.088) & 4.012 & (0.105) \\
\bottomrule
\end{tabular}
\caption{\yf{Comparison with linear outcome models and squared loss, without propensity truncation. 1000 simulations from the KS setting in \Cref{section:ks}, with population mean 210. Linear outcome and logistic propensity models are fitted without intercepts on the same sample. Parentheses give Monte Carlo standard errors: the delta method for RMSE and 100,000 bootstrap resamples of the 1000 simulation replications for median absolute error. See \Cref{section:ks_experiment_details} for details.}}
\label{tb:linear_simulation_full}
\end{table}
\end{yfrevision}

% Generated by doc/experiments-revised/build_audited_tables.py
\begin{yfrevision}
\begin{table}[t]
\centering
\footnotesize
\setlength{\tabcolsep}{5pt}
\begin{tabular}{l*{4}{w{r}{4em}w{r}{4.8em}}}
\toprule
Method & \multicolumn{2}{c}{Bias} & \multicolumn{2}{c}{Mean Abs Err} & \multicolumn{2}{c}{RMSE} & \multicolumn{2}{c}{Median Abs Err} \\
\cmidrule(lr){1-1}\cmidrule(lr){2-3}\cmidrule(lr){4-5}\cmidrule(lr){6-7}\cmidrule(lr){8-9}
\multicolumn{9}{c}{$n=200$} \\
\yf{Direct} & -0.590 & (0.107) & 2.741 & (0.065) & 3.422 & (0.075) & 2.248 & (0.074) \\
\yf{IPW} & 4.254 & (0.402) & 10.260 & (0.273) & 13.400 & (0.343) & 8.150 & (0.414) \\
\yf{IPW-SN} & 1.119 & (0.135) & 3.409 & (0.088) & 4.401 & (0.112) & 2.763 & (0.143) \\
\greymidrule
\yf{AIPW} & -2.567 & (0.126) & 3.762 & (0.091) & 4.733 & (0.104) & 3.068 & (0.129) \\
\yf{AIPW-SN} & -2.502 & (0.124) & 3.696 & (0.089) & 4.639 & (0.101) & 3.018 & (0.142) \\
\yf{TMLE} & -2.223 & (0.114) & 3.407 & (0.080) & 4.242 & (0.089) & 2.833 & (0.105) \\
\yf{C-Learner} & -2.328 & (0.118) & 3.531 & (0.083) & 4.404 & (0.093) & 2.896 & (0.135) \\
\midrule
\multicolumn{9}{c}{$n=1000$} \\
\yf{Direct} & -0.991 & (0.046) & 1.409 & (0.033) & 1.754 & (0.037) & 1.191 & (0.039) \\
\yf{IPW} & 7.190 & (0.181) & 7.746 & (0.156) & 9.180 & (0.165) & 7.296 & (0.207) \\
\yf{IPW-SN} & 1.850 & (0.063) & 2.241 & (0.049) & 2.719 & (0.052) & 2.010 & (0.067) \\
\greymidrule
\yf{AIPW} & -3.374 & (0.056) & 3.406 & (0.054) & 3.804 & (0.052) & 3.386 & (0.071) \\
\yf{AIPW-SN} & -3.308 & (0.055) & 3.341 & (0.053) & 3.731 & (0.051) & 3.316 & (0.071) \\
\yf{TMLE} & -2.935 & (0.050) & 2.974 & (0.047) & 3.331 & (0.047) & 2.955 & (0.055) \\
\yf{C-Learner} & -3.123 & (0.052) & 3.159 & (0.050) & 3.528 & (0.048) & 3.120 & (0.070) \\
\bottomrule
\end{tabular}
\caption{\yf{Comparison with linear outcome models and squared loss, with 5\% propensity truncation. 1000 simulations from the KS setting in \Cref{section:ks}, with population mean 210. Linear outcome and logistic propensity models are fitted without intercepts on the same sample. Parentheses give Monte Carlo standard errors: the delta method for RMSE and 100,000 bootstrap resamples of the 1000 simulation replications for median absolute error. See \Cref{section:ks_experiment_details} for details.}}
\label{tb:linear_simulation_clipped}
\end{table}
\end{yfrevision}

\FloatBarrier

\begin{yfrevision}
\paragraph{\yf{Flipped-observation comparison.}}
\yf{\Cref{tb:flipped_ks} uses the same outcome draws and population
mean, but replaces $A$ by $1-A$. In this flipped-observation experiment,
AIPW, AIPW-SN, TMLE, and C-Learner perform similarly and all have lower
MAE than OLS. The comparison supports this specific robustness pattern,
not superiority over every stabilized or logistic-link alternative.}
\end{yfrevision}

% Generated by doc/experiments-revised/build_audited_tables.py
\begin{yfrevision}
\begin{table}[htbp]
\centering
\footnotesize
\setlength{\tabcolsep}{5pt}
\begin{tabular}{lrrrrrrrr}
\toprule
& \multicolumn{4}{c}{$n=200$} & \multicolumn{4}{c}{$n=1000$} \\
\cmidrule(lr){2-5}\cmidrule(lr){6-9}
Method & \multicolumn{2}{c}{Bias} & \multicolumn{2}{c}{Mean Abs Err} & \multicolumn{2}{c}{Bias} & \multicolumn{2}{c}{Mean Abs Err} \\
\cmidrule(lr){1-1}\cmidrule(lr){2-3}\cmidrule(lr){4-5}\cmidrule(lr){6-7}\cmidrule(lr){8-9}
\yf{Direct} & 4.788 & (0.092) & 4.914 & (0.085) & 4.748 & (0.041) & 4.748 & (0.041) \\
\yf{IPW} & -0.089 & (0.142) & 3.362 & (0.094) & -0.192 & (0.054) & 1.356 & (0.033) \\
\yf{IPW-SN} & 2.487 & (0.098) & 3.242 & (0.073) & 2.373 & (0.044) & 2.423 & (0.041) \\
\greymidrule
\yf{AIPW} & 3.460 & (0.092) & 3.780 & (0.078) & 3.324 & (0.041) & 3.328 & (0.041) \\
\yf{AIPW-SN} & 3.445 & (0.091) & 3.766 & (0.078) & 3.306 & (0.041) & 3.310 & (0.041) \\
\yf{TMLE} & 3.368 & (0.091) & 3.697 & (0.077) & 3.244 & (0.041) & 3.248 & (0.041) \\
\yf{C-Learner} & 3.423 & (0.091) & 3.747 & (0.078) & 3.280 & (0.041) & 3.284 & (0.041) \\
\bottomrule
\end{tabular}
\caption{\yf{Comparison with the observation indicator reversed. 1000 simulations from the KS setting in \Cref{section:ks}, with population mean 210. Linear outcome and logistic propensity models are fitted without intercepts on the same sample. Replace $A$ by $1-A$ while retaining the generated outcome and population mean 210; this is not a separately generated control potential outcome. No propensity truncation; parentheses are Monte Carlo standard errors.}}
\label{tb:flipped_ks}
\end{table}
\end{yfrevision}

\FloatBarrier
\end{yfrevision}

\begin{yfrevision}
\subsubsection{Confidence Interval Coverage for Linear Outcome Models}
\label{section:ks-linear-coverage-methods}

\yf{The linear KS estimators fit and evaluate on the same sample, so
there is one fold containing all $n$ observations. We use the working
scores in \Cref{section:xgb-coverage-methods}.}

\yf{\Cref{tb:coverage_ks} reports the resulting coverage and mean interval
width. Without truncation, coverage at $n=200$ and 1000 is 0.923 and
0.774 for AIPW, 0.842 and 0.576 for TMLE, and 0.892 and 0.742 for
C-Learner. Lower MAE therefore does not ensure nominal coverage.}
\end{yfrevision}

% Generated by doc/experiments-revised/build_audited_tables.py
\begin{yfrevision}
\begin{table}[t]
\centering
\footnotesize
\setlength{\tabcolsep}{5pt}
\begin{tabular}{lrrrrrrrr}
\toprule
& \multicolumn{4}{c}{$n=200$} & \multicolumn{4}{c}{$n=1000$} \\
\cmidrule(lr){2-5}\cmidrule(lr){6-9}
Method & \multicolumn{2}{c}{No truncation} & \multicolumn{2}{c}{5\% truncation} & \multicolumn{2}{c}{No truncation} & \multicolumn{2}{c}{5\% truncation} \\
\cmidrule(lr){1-1}\cmidrule(lr){2-3}\cmidrule(lr){4-5}\cmidrule(lr){6-7}\cmidrule(lr){8-9}
\multicolumn{9}{c}{Coverage} \\
\yf{Direct} & 0.929 & (0.008) & 0.929 & (0.008) & 0.894 & (0.010) & 0.894 & (0.010) \\
\yf{IPW} & 0.999 & (0.001) & 0.999 & (0.001) & 1.000 & (0.000) & 1.000 & (0.000) \\
\yf{IPW-SN} & 0.976 & (0.005) & 0.989 & (0.003) & 0.941 & (0.007) & 0.992 & (0.003) \\
\greymidrule
\yf{AIPW} & 0.923 & (0.008) & 0.911 & (0.009) & 0.774 & (0.013) & 0.641 & (0.015) \\
\yf{AIPW-SN} & 0.884 & (0.010) & 0.899 & (0.010) & 0.662 & (0.015) & 0.617 & (0.015) \\
\yf{TMLE} & 0.842 & (0.012) & 0.889 & (0.010) & 0.576 & (0.016) & 0.658 & (0.015) \\
\yf{C-Learner} & 0.892 & (0.010) & 0.893 & (0.010) & 0.742 & (0.014) & 0.645 & (0.015) \\
\midrule
\multicolumn{9}{c}{Mean interval width} \\
\yf{Direct} & \multicolumn{2}{c}{12.743} & \multicolumn{2}{c}{12.743} & \multicolumn{2}{c}{5.829} & \multicolumn{2}{c}{5.829} \\
\yf{IPW} & \multicolumn{2}{c}{150.684} & \multicolumn{2}{c}{104.600} & \multicolumn{2}{c}{259.073} & \multicolumn{2}{c}{49.842} \\
\yf{IPW-SN} & \multicolumn{2}{c}{25.245} & \multicolumn{2}{c}{21.917} & \multicolumn{2}{c}{20.118} & \multicolumn{2}{c}{10.927} \\
\greymidrule
\yf{AIPW} & \multicolumn{2}{c}{22.257} & \multicolumn{2}{c}{15.307} & \multicolumn{2}{c}{96.053} & \multicolumn{2}{c}{7.923} \\
\yf{AIPW-SN} & \multicolumn{2}{c}{16.548} & \multicolumn{2}{c}{14.661} & \multicolumn{2}{c}{16.348} & \multicolumn{2}{c}{7.594} \\
\yf{TMLE} & \multicolumn{2}{c}{523.569} & \multicolumn{2}{c}{13.673} & \multicolumn{2}{c}{2827.044} & \multicolumn{2}{c}{7.082} \\
\yf{C-Learner} & \multicolumn{2}{c}{15.025} & \multicolumn{2}{c}{14.196} & \multicolumn{2}{c}{11.839} & \multicolumn{2}{c}{7.384} \\
\bottomrule
\end{tabular}
\caption{\yf{Coverage of nominal 95\% confidence intervals with linear outcome models. 1000 simulations from the KS setting in \Cref{section:ks}, with population mean 210. Linear outcome and logistic propensity models are fitted without intercepts on the same sample. Parentheses give binomial Monte Carlo standard errors of coverage; the lower panel reports mean interval widths in outcome units. Direct includes the OLS estimation term; IPW-SN and AIPW-SN use the ratio delta method. See \Cref{section:ks-linear-coverage-methods} for the standard-error formulas and their limitations under misspecification.}}
\label{tb:coverage_ks}
\end{table}
\end{yfrevision}

\FloatBarrier

% Logistic-loss implementations and matched comparison for the KS experiment.
\begin{yfrevision}
\subsubsection{\yf{Matched Comparison with Logistic Loss}}
\label{section:ks-logistic-matched}

\begin{yfrevision}
\yf{As discussed in \Cref{section:ks_linear}, a common procedure to increase instability in the TMLE procedure is to 
model outomes via a logistic link that ensures an evelope for the outcome model. 
A common heuristic is to use observational data to build lower and upper bounds for outcomes and fit 
the outcomes using a fractional logit model. For both methods we follow the description provided in \cite{van2011propensity} and set $Y_{\min}=0.9\min_{i:A_i=1}Y_i$ and
$Y_{\max}=1.1\max_{i:A_i=1}Y_i$, and use
$\tilde Y_i=(Y_i-Y_{\min})/(Y_{\max}-Y_{\min})$ on observed rows.
These empirical limits are a stabilization convention and mispecified, as outcomes are Gaussian. Next we describe the implementation details with the logistic link.}
\end{yfrevision}

\paragraph{TMLE-L.} Abusing notation, define $\sigma(x) := 1/(1+\exp(-x))$, the inverse logit function. As a starting point, one solves for the fractional logistic regression using the scaled outcomes $\tilde Y$:

\[
\min_{\theta} \left\{ - \sum_{i=1}^{n} A_i\left[ \tilde Y_i \log \sigma(\theta^\top X_i) + (1 - \tilde Y_i) \log (1 - \sigma(\theta^\top X_i)) \right] \right\},
\]
which gives $\what\mu(X_i)=\sigma(X_i^\top\widehat\theta)$. The second step is to define a parametric submodel and find the optimal fluctuation parameter $\epsilon$. Define
$$
\what{\mu}(X_i,H_i,\epsilon) :=\sigma\!\left(\log\frac{\what\mu(X_i)}{1-\what\mu(X_i)}+\epsilon H_i\right),\qquad H_i=\frac1{\what\pi(X_i)},
$$
and note that $\what{\mu}(X_i,H_i,\epsilon)$ is a shift of $\what{\mu}(X_i)$ under the logit transformation in the direction of $H_i$. Therefore, under the logistic link, the optimal fluctuation can be fitted by solving
\begin{equation}
\label{eqn:targeting_logistic}
    \epsilon^\star:=\argmin_{\epsilon \in \R}
    ~~-\frac{1}{n}\sum_{i=1}^n  A_i \left(
    \tilde Y_i \log \what{\mu}(X_i,H_i,\epsilon) + (1-\tilde Y_i)\log (1-\what{\mu}(X_i,H_i,\epsilon))\right).
\end{equation}

The first-order condition is
\[
  \sum_i\frac{A_i}{\what\pi(X_i)}
  \{\tilde Y_i-\what\mu(X_i,H_i,\epsilon^\star)\}=0.
\]
The final estimate on the original outcome scale is
$Y_{\min}+(Y_{\max}-Y_{\min})n^{-1}\sum_i
\what\mu(X_i,H_i,\epsilon^\star)$.
To implement TMLE with the logistic link, we use the \texttt{tmle} package available in R~\citep{gruber2012tmle}.

\paragraph{C-Learner-L.} For the C-Learner-L implementation we solve
\begin{align*}
\theta^\star = \argmin_{\theta} \; - &\sum_{i=1}^{n} A_i\left[ \tilde Y_i \log \sigma(\theta^\top X_i) + (1 - \tilde Y_i) \log (1 - \sigma(\theta^\top X_i)) \right]\;\text{ such that }\\
&\sum_{i=1}^{n} \tfrac{A_i}{\hat \pi(X_i)}\left( \tilde Y_i -\sigma(\theta^\top X_i) \right)=0.
\end{align*}

The objective is convex, but the nonlinear equality constraint need not define a convex feasible set. To solve the problem, we use a Sequential Least Squares Programming algorithm available in the R package \texttt{nloptr} \cite{nloptr2023}. The C-Learner-L estimate is likewise transformed back to the original scale:
$Y_{\min}+(Y_{\max}-Y_{\min})n^{-1}\sum_i\sigma(X_i^\top\theta^\star)$.

\yf{We compare C-Learner-L with R \texttt{tmle}
package, version 2.1.1~\citep{gruber2012tmle}, using 1000 paired replications
at each sample size in the KS setup of \Cref{section:ks}.
All three implementations use the same no-intercept logistic propensity fit, covariates,
and observed-outcome normalization described above.}

\yf{The primary TMLE fluctuation is the inverse-propensity-covariate update
in~\eqref{eqn:targeting_logistic}. We supply the common initial outcome and
propensity predictions to \texttt{tmle::tmle()}. 
All three implementations use the same prespecified one-sided propensity
floor: none and 0.05.}

\yf{As a targeting sensitivity, we also compare our results against a commonly used in practice TMLE-L gwt, which 
fits a constant logit shift by
inverse-propensity-weighted loss, rather than using inverse propensity as a
covariate.}

\begin{table}[t]
\centering
\footnotesize
\setlength{\tabcolsep}{3pt}
\begin{tabular}{lrrrrrrrrrrrr}
\toprule
& \multicolumn{6}{c}{\yf{$n=200$}}
& \multicolumn{6}{c}{\yf{$n=1000$}} \\
\cmidrule(lr){2-7}\cmidrule(lr){8-13}
\yf{Truncation} & \multicolumn{2}{c}{\yf{TMLE-L}}
& \multicolumn{2}{c}{\yf{C-Learner-L}}
& \multicolumn{2}{c}{\yf{TMLE-L gwt}}
& \multicolumn{2}{c}{\yf{TMLE-L}}
& \multicolumn{2}{c}{\yf{C-Learner-L}}
& \multicolumn{2}{c}{\yf{TMLE-L gwt}} \\
\cmidrule(lr){1-1}\cmidrule(lr){2-3}\cmidrule(lr){4-5}
\cmidrule(lr){6-7}\cmidrule(lr){8-9}\cmidrule(lr){10-11}\cmidrule(lr){12-13}
\yf{None} & 3.349 & (0.081) & 3.194 & (0.080) & \yf{3.665} & \yf{(0.121)} & 2.099 & (0.045) & 2.650 & (0.066) & \yf{4.357} & \yf{(0.280)} \\
%\yf{2.5\%} & 3.000 & (0.071) & 3.123 & (0.076) & \yf{3.278} & \yf{(0.083)} & 1.884 & (0.041) & 2.147 & (0.046) & \yf{2.330} & \yf{(0.050)} \\
\yf{5\%} & 2.924 & (0.069) & 2.976 & (0.070) & \yf{3.029} & \yf{(0.072)} & 1.739 & (0.038) & 1.836 & (0.040) & \yf{1.911} & \yf{(0.041)} \\
%\yf{Adaptive} & 2.732 & (0.065) & 2.741 & (0.065) & \yf{2.744} & \yf{(0.065)} & 1.822 & (0.040) & 1.983 & (0.043) & \yf{2.103} & \yf{(0.045)} \\
\bottomrule
\end{tabular}
\caption{\yf{Mean absolute error (Monte Carlo standard error) for TMLE-L,
C-Learner-L, and TMLE-L gwt in the KS setup over 1000 replications.
All three implementations use the same initial
outcome specification, propensity estimates, outcome scaling, and truncation.
TMLE-L uses inverse propensity as a targeting covariate; TMLE-L gwt fits a constant logit shift
using inverse-propensity-weighted loss.}}
\label{tab:ks-matched-logistic}
\end{table}

\yf{The results are displayed in \Cref{tab:ks-matched-logistic}. Without truncation, the comparison shows lower MAE for
C-Learner-L at $n=200$: the paired TMLE-L minus C-Learner-L MAE difference
is $0.154$, with standard error $0.046$. At $n=1000$, the difference is
$-0.550$ ($0.054$), favoring TMLE-L. With 5\% clipping, TMLE-L has modestly lower MAE at both sample sizes; the differences remain statistically significant in the paired Monte Carlo comparison. 
TMLE-L gwt has higher MAE than both methods at both sample sizes.
All 1000 replications are retained.}

\end{yfrevision}

\FloatBarrier

\subsubsection{Estimator Performance When Varying Overlap Between Treatment and Control}
\label{section:ks_sensitivity}
In order to test the sensitivity of different asymptotically optimal methods with respect to the propensity score, we modify the probability of treatment in the data generation by introducing a scaling parameter $c$ so that
\begin{align}
\pi\opt(\xi) = \frac{\exp(c(-\xi_{1} + 0.5 \xi_{2} - 0.25 \xi_{3} - 0.1 \xi_{4}))}{1+\exp(c(-\xi_{1} + 0.5 \xi_{2} - 0.25 \xi_{3} - 0.1 \xi_{4}))},
\label{eq:ks_variant_c}
\end{align}
where we take $c \in \{0.25, 0.5, 0.75, 1, 1.25, 1.5, 1.75\}$.
\begin{yfrevision}\yf{Here $c=1$ is the original KS setting. The
covariate and outcome constructions remain fixed as $c$ varies.}\end{yfrevision}

For each value of the parameter $c$, we generate 1000 seeds of sample size 200 with no truncation. We use linear outcome models and logistic propensity models. The results for the mean absolute error are displayed in
\Cref{fig:mae_vs_prop}. \begin{yfrevision}\yf{\Cref{tb:prop_stats_ks,tb:1_prop_stats_ks}
report medians across replications of per-dataset summaries of
$\widehat\pi(X)$ and $1/\widehat\pi(X)$. Each per-dataset statistic
uses all observations.}\end{yfrevision}

% Generated by doc/experiments-revised/build_audited_tables.py
\begin{yfrevision}
\begin{table}[t]
\centering
\footnotesize
\setlength{\tabcolsep}{5pt}
\begin{tabular}{l*{4}{w{r}{5em}}}
\toprule
Scaling $c$ & \multicolumn{1}{c}{SD} & \multicolumn{1}{c}{Min} & \multicolumn{1}{c}{Max} & \multicolumn{1}{c}{CVaR (5\%)} \\
\cmidrule(lr){1-1}\cmidrule(lr){2-2}\cmidrule(lr){3-3}\cmidrule(lr){4-4}\cmidrule(lr){5-5}
0.25 & 0.081 & 0.1921 & 0.6931 & 0.2927 \\
0.50 & 0.1261 & 0.07088 & 0.7656 & 0.183 \\
0.75 & 0.1706 & 0.02127 & 0.8321 & 0.09999 \\
1.00 & 0.2089 & 0.006573 & 0.8813 & 0.05367 \\
1.25 & 0.2375 & 0.001821 & 0.9166 & 0.02769 \\
1.50 & 0.2637 & 0.0005336 & 0.942 & 0.01499 \\
1.75 & 0.2832 & 0.0001644 & 0.9605 & 0.008162 \\
\bottomrule
\end{tabular}
\caption{\yf{Summary statistics for $\widehat\pi(X)$: medians across 1000 KS datasets of size 200 at each $c$, without truncation. Each dataset statistic uses all subjects. CVaR (5\%) averages the smallest 5\% of values.}}
\label{tb:prop_stats_ks}
\end{table}
\end{yfrevision}

% Generated by doc/experiments-revised/build_audited_tables.py
\begin{yfrevision}
\begin{table}[t]
\centering
\footnotesize
\setlength{\tabcolsep}{5pt}
\begin{tabular}{l*{4}{w{r}{5em}}}
\toprule
Scaling $c$ & \multicolumn{1}{c}{SD} & \multicolumn{1}{c}{Min} & \multicolumn{1}{c}{Max} & \multicolumn{1}{c}{CVaR (5\%)} \\
\cmidrule(lr){1-1}\cmidrule(lr){2-2}\cmidrule(lr){3-3}\cmidrule(lr){4-4}\cmidrule(lr){5-5}
0.25 & 0.4678 & 1.443 & 5.205 & 3.576 \\
0.50 & 1.363 & 1.306 & 14.11 & 6.575 \\
0.75 & 4.143 & 1.202 & 47.01 & 15.47 \\
1.00 & 12.8 & 1.135 & 152.1 & 40.75 \\
1.25 & 43.5 & 1.091 & 549 & 118.1 \\
1.50 & 154.7 & 1.062 & 1874 & 369.5 \\
1.75 & 470.6 & 1.041 & 6082 & 1017 \\
\bottomrule
\end{tabular}
\caption{\yf{Summary statistics for $1/\widehat\pi(X)$: medians across 1000 KS datasets of size 200 at each $c$, without truncation. Each dataset statistic uses all subjects. CVaR (5\%) averages the largest 5\% of values.}}
\label{tb:1_prop_stats_ks}
\end{table}
\end{yfrevision}

\FloatBarrier
\ifdefined\response\newpage\newpage\clearpage\pagebreak\else\fi
\begin{yfrevision}
\subsubsection{\yf{Covariate Balancing}}
\label{section:linear_cbps_ks}
\label{section:dual_details}

We describe the implementation of the C-Learner-Dual for linear propensity models under the logistic link. We also discuss the implementation of the benchmark methods. For all models, we also implement their self-normalized version, in which the weights are rescaled to sum to one.

\paragraph{CBPS} \begin{yfrevision}\yf{We use the R \texttt{CBPS}
package, version 0.24, with \texttt{method="over"} and the two-step
over-identified fit in both \Cref{tb:linear_ipw,tb:linear_cbps}.
The formula supplies the four observed covariate columns with
\texttt{-1}, but this package version internally restores a constant
column; its effective propensity model includes an intercept. We use
that package specification. Outcome predictions are shared with the
squared-loss comparison. Every row uses 1000 replications.}
\end{yfrevision}

\paragraph{C-Learner-Dual} First, we parametrize the propensity‐score model by
\[
\pi_{\beta_{\text{C-Learner-Dual}}}(X_i)
\;=\;
\frac{\exp\bigl(X_i^\top \beta_{\text{C-Learner-Dual}}\bigr)}
     {1 + \exp\bigl(X_i^\top \beta_{\text{C-Learner-Dual}}\bigr)},
\]
and fit \(\beta_{\text{C-Learner-Dual}}\) by maximizing the Bernoulli log‐likelihood on the treatment indicators \(A_i\).  Equivalently, we minimize the negative log‐likelihood
\[
-\ell\bigl(\beta_{\text{C-Learner-Dual}}\bigr)
=
- \sum_{i=1}^n \Bigl[A_i\,\eta_i
- \log\bigl(1 + e^{\eta_i}\bigr)\Bigr],
\qquad
\eta_i \;=\; X_i^\top \beta_{\text{C-Learner-Dual}}.
\]
Because this objective is smooth and convex, we also supply its gradient in closed form:
\[
\nabla_{\beta_{\text{C-Learner-Dual}}}
\bigl[-\ell(\beta_{\text{C-Learner-Dual}})\bigr]
=
-\,X^\top\bigl(A - p(\beta_{\text{C-Learner-Dual}})\bigr),
\quad
p(\beta_{\text{C-Learner-Dual}})_i
=
\frac{e^{\eta_i}}{1 + e^{\eta_i}}.
\]

Next, we impose the finite-sample balance constraint
\[
\frac{1}{n}
\sum_{i=1}^n
\Bigl(1 - \frac{A_i}{\pi_{\beta_{\text{C-Learner-Dual}}}(X_i)}\Bigr)
\,\hat\mu(X_i)
\;=\;0,
\]
where \(\hat\mu(X_i)\) is a fixed outcome‐model prediction.  Rearranging shows this is equivalent (up to an additive constant) to
\[
\sum_{i=1}^n
A_i\,\hat\mu(X_i)\,e^{-X_i^\top \beta_{\text{C-Learner-Dual}}}
\;-\;
\sum_{i=1}^n
(1 - A_i)\,\hat\mu(X_i)
\;=\;
0.
\]
We likewise provide the Jacobian of this constraint:
\[
\nabla_{\beta_{\text{C-Learner-Dual}}}
\bigl[\text{constraint}\bigr]
=
-\,\sum_{i=1}^n
A_i\,\hat\mu(X_i)\,e^{-X_i^\top \beta_{\text{C-Learner-Dual}}}\,X_i.
\]

\begin{yfrevision}
\yf{We use a damped Newton
solve of the likelihood/equality-constraint KKT equations, initialized at
the unconstrained logistic MLE. For conditioning, columns are divided by
their root mean square; coefficients are mapped back to the original
coordinates, so the statistical model is unchanged. The analytic
Lagrangian Hessian supplies the Newton step, with backtracking on the
squared KKT residual, at most 200 iterations, and stopping tolerance
$10^{-12}$. All 2000 fits satisfy the recorded stationarity,
mean-constraint, and positive reduced-Hessian checks. The verified coefficient vector
$\widehat\beta_{\text{C-Learner-Dual}}$ defines}
\end{yfrevision}
\[
\widehat\pi^C(X_i)
=
\frac{\exp\bigl(X_i^\top \widehat\beta_{\text{C-Learner-Dual}}\bigr)}
     {1 + \exp\bigl(X_i^\top \widehat\beta_{\text{C-Learner-Dual}}\bigr)},
\]
and the population mean is estimated by
\label{section:dual_details-end}
$$
\widehat\psi
=
\frac{1}{n}
\sum_{i=1}^n
\frac{A_i}{\widehat\pi^C(X_i)} \;Y_i.
$$

\Needspace{12\baselineskip}
\paragraph{Balancing results, linear outcome model}
\begin{yfrevision}
\yf{\Cref{tb:linear_cbps} uses over-identified CBPS propensities with all
1000 replications for every estimator. Relative to the initial
likelihood-fitted propensities in \Cref{tb:linear_simulation}, this lowers
MAE for AIPW, AIPW-SN, squared-loss TMLE, and squared-loss C-Learner at
both sample sizes. At $n=200$, C-Learner has the lowest MAE in this
comparison, 3.184, close to AIPW-SN's 3.202. At $n=1000$, IPW-SN
has the lowest MAE, 1.764.}
\end{yfrevision}
% Generated by doc/experiments-revised/build_audited_tables.py
\begin{yfrevision}
\begin{table}[htbp]
\centering
\footnotesize
\setlength{\tabcolsep}{5pt}
\begin{tabular}{lrrrrrrrr}
\toprule
& \multicolumn{4}{c}{$n=200$} & \multicolumn{4}{c}{$n=1000$} \\
\cmidrule(lr){2-5}\cmidrule(lr){6-9}
Method & \multicolumn{2}{c}{Bias} & \multicolumn{2}{c}{Mean Abs Err} & \multicolumn{2}{c}{Bias} & \multicolumn{2}{c}{Mean Abs Err} \\
\cmidrule(lr){1-1}\cmidrule(lr){2-3}\cmidrule(lr){4-5}\cmidrule(lr){6-7}\cmidrule(lr){8-9}
\yf{IPW} & -6.439 & (0.260) & 8.351 & (0.199) & 1.154 & (0.178) & 3.995 & (0.131) \\
\yf{IPW-SN} & -2.869 & (0.104) & 3.611 & (0.078) & -1.245 & (0.061) & 1.764 & (0.046) \\
\greymidrule
\yf{AIPW} & -1.950 & (0.112) & 3.202 & (0.078) & -3.664 & (0.085) & 3.686 & (0.084) \\
\yf{AIPW-SN} & -1.962 & (0.112) & 3.202 & (0.078) & -3.613 & (0.077) & 3.636 & (0.076) \\
\yf{TMLE} & -29.303 & (9.221) & 30.104 & (9.219) & -7.354 & (0.996) & 7.359 & (0.996) \\
\yf{C-Learner} & -1.968 & (0.110) & 3.184 & (0.076) & -3.437 & (0.059) & 3.460 & (0.057) \\
\bottomrule
\end{tabular}
\caption{\yf{Comparison using CBPS propensities and linear outcome models with squared loss. All methods use the same over-identified CBPS fit with an intercept, and the outcome fits omit intercepts. Each sample size uses 1000 KS simulations without propensity truncation. Standard errors are in parentheses.}}
\label{tb:linear_cbps}
\ifyfreview\par\smallskip
{\footnotesize\yfdelete{Previous table: CBPS-TMLE-L and CBPS-C-Learner-L rows.}\par}\fi
\end{table}
\end{yfrevision}

\label{section:linear_cbps_ks-end}
\end{yfrevision}
\ifdefined\response\newpage\newpage\clearpage\pagebreak\else\fi

\FloatBarrier
\begin{yfrevision}
\subsubsection{\yf{Gradient Boosted Regression Tree Outcome Models}}
\label{section:ks-xgb-projection}
\yf{We use a formal two-fold nuisance cross-fitting procedure presented in
\Cref{section:methods_boosting,alg:boosting} for the missing-outcome mean.
Only the $A=1$ outcome regression is needed;
there is no control-outcome model. The target remains the population mean
$\psi(P)=210$.}

\yf{\paragraph{Data and comparison.}
We generate the KS data in \Cref{section:ks} and R seeds $1,\ldots,1000$
separately at $n=200$ and $n=1000$.
All estimators in the tree comparison share the same outer folds,
initial outcome predictions, and propensity predictions. The two equal outer
halves are stratified by $A$ using outcome-independent random seeds. Each
half has a separate approximately 80/20 stratified inner split; outcomes in
an evaluation half never select the initial outcome model fitted on the
other half.}

\yf{\paragraph{Projection and fluctuation.}
For each evaluation fold $I_k$, we apply \Cref{alg:boosting} with the
opposite half as the training sample. Thus $\widehat\mu_{-k,n}$ and
$\widehat\pi_{-k,n}$ are fitted using the opposite half. Write
$\widehat h_k(x)=1/\widehat\pi_{-k,n}(x)$, where $\widehat\pi_{-k,n}$
denotes the common effective propensity after any prespecified lower floor.
Let $\widehat q_k$ denote the fitted projection direction in
\Cref{alg:boosting}. We compute the fluctuation on the evaluation fold using
\eqref{eq:boosting-projection-calibration} and report the fold-size-weighted
average of the two plug-in estimates. 
Cross-fitting applies to the initial nuisance estimates; the
projection, constraint, and fluctuation use the evaluation fold.}

\yfdelete{Two-fold XGB projection for the missing-outcome mean.}

\yf{\paragraph{Fitting settings.}
Outcome and projection regressions each search 12 combinations of maximum
depth $\{2,3,4\}$, learning rate $\{0.03,0.1\}$, and column subsampling
$\{0.8,1\}$. The shared XGBoost settings are squared loss, histogram trees,
row subsampling 0.8, minimum child weight 1, minimum split-loss penalty 0,
$L_2$ penalty 1, and $L_1$ penalty 0. Validation RMSE selects both the
hyperparameters and the number of trees, with a maximum of 2000 trees and
20-round early stopping. We refit the selected configuration and tree count
on all $A=1$ rows of the corresponding fitting half.
For outcome fitting, the normalization mean and standard deviation use only
observed inner-training outcomes. For the projection, they use only inverse
propensities on observed inner-training rows of the evaluation half.
Normalization constants are retained for the full refit, and predictions
are returned to their original scales. Standard deviations use divisor equal
to the inner-training count. We use Python
XGBoost 3.4.1, with one thread per fit.}

\yf{The propensity model is an $L_1$-penalized logistic regression on the four
observed covariates, using R \texttt{glmnet} with an intercept and its internal
standardization. Within each outer training half, stratified three-fold
cross-validation selects among 50 log-spaced penalties from $10^{-10}$ to 1
by mean validation AUC. Ties within $10^{-12}$ select the largest penalty.
The selected model is refitted on the entire outer training half. These
choices are fixed before examining estimator errors.
This explicit small-sample propensity recipe, including stratification and
tie-breaking, is shared by all methods in the comparison.}

\yf{A common numerical lower guard of $10^{-8}$ is applied to the predicted
propensities, with no upper clipping. Any additional truncation replaces
this by the larger prespecified floor and is applied to every comparator;
the outcome and propensity fits remain fixed while the projection is
retuned and refitted for the resulting inverse-propensity target.
Monte Carlo standard errors are the sample standard deviation of the
replication-level errors (or absolute errors) divided by $\sqrt{1000}$.}

\yf{\paragraph{Matched baselines and interpretation.}
On each evaluation fold, the direct method averages the initial predictions;
AIPW adds $\widehat P_{k,n}[A\widehat h_k(Y-\widehat\mu_{-k,n})]$;
and AIPW-SN divides that correction by
$\widehat P_{k,n}[A\widehat h_k]$. IPW and IPW-SN use the analogous raw and
normalized weighted outcome means. Squared-loss TMLE replaces
$\widehat q_k$ by $\widehat h_k$ in
\eqref{eq:boosting-projection-calibration}. Each method is then averaged across
the same two folds. 
Differently from the TMLE that uses raw inverse propensities, the projection step in the C-Learner agorithm 
 extends a learned tree direction to
those covariates. This distinction permits different responses to extreme
estimated propensities.}
\end{yfrevision}

% Generated by doc/experiments-revised/build_audited_tables.py
\begin{yfrevision}
\begin{table}[t]
\centering
\footnotesize
\setlength{\tabcolsep}{5pt}
\begin{tabular}{lrrrrrr}
\toprule
Method & \multicolumn{2}{c}{0.1\% truncation} & \multicolumn{2}{c}{2.5\% truncation} & \multicolumn{2}{c}{5\% truncation} \\
\cmidrule(lr){1-1}\cmidrule(lr){2-3}\cmidrule(lr){4-5}\cmidrule(lr){6-7}
\multicolumn{7}{c}{$n=200$} \\
\yf{Direct} & 4.724 & (0.090) & 4.724 & (0.090) & 4.724 & (0.090) \\
\yf{IPW} & 41.680 & (4.981) & 20.129 & (0.747) & 15.840 & (0.449) \\
\yf{IPW-SN} & 6.369 & (0.152) & 5.751 & (0.118) & 5.555 & (0.112) \\
\greymidrule
\yf{AIPW} & 7.541 & (1.081) & 4.047 & (0.113) & 3.588 & (0.085) \\
\yf{AIPW-SN} & 3.843 & (0.103) & 3.516 & (0.081) & 3.389 & (0.077) \\
\yf{TMLE} & 4.054 & (0.124) & 3.387 & (0.078) & 3.360 & (0.077) \\
\yf{C-Learner} & 3.529 & (0.080) & 3.481 & (0.078) & 3.428 & (0.077) \\
\midrule
\multicolumn{7}{c}{$n=1000$} \\
\yf{Direct} & 1.777 & (0.034) & 1.777 & (0.034) & 1.777 & (0.034) \\
\yf{IPW} & 21.673 & (1.564) & 8.911 & (0.248) & 6.795 & (0.161) \\
\yf{IPW-SN} & 3.786 & (0.146) & 2.458 & (0.056) & 2.443 & (0.054) \\
\greymidrule
\yf{AIPW} & 2.910 & (0.236) & 1.298 & (0.032) & 1.200 & (0.028) \\
\yf{AIPW-SN} & 1.843 & (0.080) & 1.236 & (0.029) & 1.191 & (0.028) \\
\yf{TMLE} & 1.229 & (0.028) & 1.203 & (0.028) & 1.216 & (0.028) \\
\yf{C-Learner} & 1.263 & (0.029) & 1.229 & (0.028) & 1.232 & (0.028) \\
\bottomrule
\end{tabular}
\caption{\yf{Mean absolute error for the KS tree methods with propensity truncation. Each column uses the same 1000 simulations per sample size and the same two-fold nuisance fits as \Cref{tb:linear_simulation_xgb}. The projection is refitted for each truncation threshold. Standard errors are in parentheses.}}
\label{tb:ks-xgb-projection-truncation}
\end{table}
\end{yfrevision}

\begin{yfrevision}
\yf{\paragraph{Propensity truncation.}
\Cref{tb:ks-xgb-projection-truncation} reports all three additional
truncation thresholds, without selecting a threshold by estimation error.
Initial outcome and propensity estimates are held fixed, and the
projection is refitted for each threshold.}
\end{yfrevision}

\begin{yfrevision}
\yf{\paragraph{Confidence interval coverage.}
\Cref{tb:coverage_ks_xgb,tb:coverage_ks_xgb_clipped} report coverage
and mean interval widths, using the formulas in
\Cref{section:xgb-coverage-methods}. These calculations use the same
point estimates and propensity truncation thresholds as the error tables.
All 1000 replications per sample size are included.}
\yf{Without additional truncation, C-Learner has coverage 0.821 and
0.887 at $n=200$ and 1000, with mean interval widths 17.174 and 5.459.
Thus the lower mean absolute errors do not translate into nominal 95\%
coverage. Bias and the treatment of fitted nuisance models in the
standard errors remain relevant to interval calibration.}
\end{yfrevision}
% Generated by doc/experiments-revised/build_audited_tables.py
\begin{yfrevision}
\begin{table}[t]
\centering
\footnotesize
\setlength{\tabcolsep}{5pt}
\begin{tabular}{lrrrrrr}
\toprule
& \multicolumn{3}{c}{$n=200$} & \multicolumn{3}{c}{$n=1000$} \\
\cmidrule(lr){2-4}\cmidrule(lr){5-7}
Method & \multicolumn{2}{c}{Coverage} & \multicolumn{1}{c}{Mean width} & \multicolumn{2}{c}{Coverage} & \multicolumn{1}{c}{Mean width} \\
\cmidrule(lr){1-1}\cmidrule(lr){2-3}\cmidrule(lr){4-4}\cmidrule(lr){5-6}\cmidrule(lr){7-7}
\yf{Direct} & 0.411 & (0.016) & 7.642 & 0.597 & (0.016) & 4.022 \\
\yf{IPW} & 0.990 & (0.003) & 375.549 & 0.969 & (0.005) & 402.027 \\
\yf{IPW-SN} & 0.726 & (0.014) & 18.987 & 0.824 & (0.012) & 13.189 \\
\greymidrule
\yf{AIPW} & 0.860 & (0.011) & 94.253 & 0.933 & (0.008) & 52.285 \\
\yf{AIPW-SN} & 0.821 & (0.012) & 13.024 & 0.899 & (0.010) & 6.741 \\
\yf{TMLE} & 0.847 & (0.011) & 1617.697 & 0.914 & (0.009) & 41.669 \\
\yf{C-Learner} & 0.821 & (0.012) & 17.174 & 0.887 & (0.010) & 5.459 \\
\bottomrule
\end{tabular}
\caption{\yf{Coverage of nominal 95\% confidence intervals for the KS tree methods with two-fold cross-fitting, over 1000 simulations per sample size. The target is the population mean 210. There is no additional truncation beyond the shared $10^{-8}$ numerical guard. Widths are in outcome units. Direct uses a naive prediction-only variance. TMLE and projection use their updated outcome fits in the influence function. Parentheses give standard errors of coverage. See \Cref{section:xgb-coverage-methods} for the variance formulas.}}
\label{tb:coverage_ks_xgb}
\end{table}
\end{yfrevision}

% Generated by doc/experiments-revised/build_audited_tables.py
\begin{yfrevision}
\begin{table}[t]
\centering
\footnotesize
\setlength{\tabcolsep}{5pt}
\begin{tabular}{lrrrrrr}
\toprule
& \multicolumn{3}{c}{$n=200$} & \multicolumn{3}{c}{$n=1000$} \\
\cmidrule(lr){2-4}\cmidrule(lr){5-7}
Method & \multicolumn{2}{c}{Coverage} & \multicolumn{1}{c}{Mean width} & \multicolumn{2}{c}{Coverage} & \multicolumn{1}{c}{Mean width} \\
\cmidrule(lr){1-1}\cmidrule(lr){2-3}\cmidrule(lr){4-4}\cmidrule(lr){5-6}\cmidrule(lr){7-7}
\multicolumn{7}{c}{0.1\% truncation} \\
\yf{Direct} & 0.411 & (0.016) & 7.642 & 0.597 & (0.016) & 4.022 \\
\yf{IPW} & 0.990 & (0.003) & 177.353 & 0.969 & (0.005) & 93.034 \\
\yf{IPW-SN} & 0.726 & (0.014) & 19.013 & 0.824 & (0.012) & 13.373 \\
\greymidrule
\yf{AIPW} & 0.860 & (0.011) & 28.411 & 0.933 & (0.008) & 12.141 \\
\yf{AIPW-SN} & 0.823 & (0.012) & 13.047 & 0.900 & (0.009) & 6.855 \\
\yf{TMLE} & 0.846 & (0.011) & 16.359 & 0.911 & (0.009) & 5.552 \\
\yf{C-Learner} & 0.821 & (0.012) & 13.614 & 0.887 & (0.010) & 5.334 \\
\midrule
\multicolumn{7}{c}{2.5\% truncation} \\
\yf{Direct} & 0.411 & (0.016) & 7.642 & 0.597 & (0.016) & 4.022 \\
\yf{IPW} & 0.988 & (0.003) & 101.152 & 0.967 & (0.006) & 49.287 \\
\yf{IPW-SN} & 0.749 & (0.014) & 18.934 & 0.851 & (0.011) & 10.849 \\
\greymidrule
\yf{AIPW} & 0.858 & (0.011) & 16.118 & 0.931 & (0.008) & 6.142 \\
\yf{AIPW-SN} & 0.843 & (0.012) & 13.005 & 0.921 & (0.009) & 5.717 \\
\yf{TMLE} & 0.831 & (0.012) & 12.541 & 0.898 & (0.010) & 5.096 \\
\yf{C-Learner} & 0.818 & (0.012) & 12.448 & 0.891 & (0.010) & 5.099 \\
\midrule
\multicolumn{7}{c}{5\% truncation} \\
\yf{Direct} & 0.411 & (0.016) & 7.642 & 0.597 & (0.016) & 4.022 \\
\yf{IPW} & 0.989 & (0.003) & 88.256 & 0.968 & (0.006) & 41.946 \\
\yf{IPW-SN} & 0.749 & (0.014) & 18.167 & 0.836 & (0.012) & 9.537 \\
\greymidrule
\yf{AIPW} & 0.855 & (0.011) & 14.258 & 0.921 & (0.009) & 5.401 \\
\yf{AIPW-SN} & 0.848 & (0.011) & 12.609 & 0.913 & (0.009) & 5.251 \\
\yf{TMLE} & 0.828 & (0.012) & 12.157 & 0.894 & (0.010) & 5.020 \\
\yf{C-Learner} & 0.821 & (0.012) & 12.132 & 0.888 & (0.010) & 5.017 \\
\bottomrule
\end{tabular}
\caption{\yf{KS tree methods: nominal 95\% coverage and mean interval width with propensity truncation. Each panel uses the same 1000 datasets per sample size and the same point estimates as the corresponding error comparison. Propensity and initial outcome fits are fixed across floors; projection fits use each floor's effective weights. Parentheses are binomial Monte Carlo standard errors. No replications are excluded and no floor is selected using coverage.}}
\label{tb:coverage_ks_xgb_clipped}
\end{table}
\end{yfrevision}

\FloatBarrier
\endgroup
\subsection{CivilComments Experiments (\Cref{section:civilcomments})}
\label{section:civilcomments-experiment-details}
\subsubsection{Implementation Details}
\label{section:nlp_details}
% We consider 100 random subsamplings of the full data, with 2000 data points each, and $K=2$ splits, so that each dataset is split into two subsets of 1000 data points each.
% When there are two splits, one is $\what{P}_{\rm train}$ and the other is $\what{P}_{\rm val}=\what{P}_{\rm eval}$.
% Note that for each split we use the entire 1000 points to train $\what\pi_{-k,n}$, and we use only the labeled portion of those 1000 to train $\what\mu_{-k,n}$.

%Note that $\epsilon$ in TMLE and normalization for normalized estimators are calculated using the entire dataset size of 2000.
%: we train $\what \mu, \what\pi$ on one split,
%\tc{Check: how do we do eps for TMLE?}

\paragraph{Model and Training Objectives}
Both propensity and outcome models are a linear layer on top of a DistilBERT featurizer, with either softmax and cross-entropy loss for propensity score, or mean squared error for outcome models.
For training outcome models, we only consider training data on which $A=1$, as only those terms contribute to the loss.
\ifdefined\pnas
    $\what\mu^C$ learned using the procedure described in \Cref{section:methods_nn}.
    We train $\what \mu$
    using squared loss on labeled units, and propensity models $\what \pi$ using the logistic loss.
    On each dataset draw, we use cross-fitting, as described in \Cref{section:theory_text}, with $K=2$ folds, for all estimators.
\fi
\paragraph{Hyperparameters and Settings}
For propensity models, we tried learning rates of $\{10^{-3},\allowbreak10^{-4},\allowbreak10^{-5}\}$
for the setting of $l=10^{-4}$ and found that a learning rate of $10^{-4}$ performed the best in terms of val loss. Because of computational constraints, we also used this learning rate for $l=10^{-2},10^{-3}$.

For outcome models (including C-Learner) in the settings of $l\in\{10^{-2},10^{-3},10^{-4}\}$, we tried learning rates of $\{10^{-3},10^{-4},10^{-5},10^{-6}\}$.
For C-Learner regularization, we tried $\lambda$ taking on values of $\lambda_0 / \what{P}_{\rm eval}[A/\hat\pi(X)]^2$ where $\lambda_0\in\{0,1,4,16,64,256\}$.
Essentially, the penalty is $\lambda_0$ times the square of the bias shift that would be required to satisfy the condition.

Because of computational constraints, for choosing these hyperparameters (learning rate, $\lambda$), we select a subset of dataset draws, and choose the best hyperparameters based on each model's criteria on that small set of dataset draws. Then, after the hyperparameters are chosen, we run over the entire set of dataset draws.

While hyperparameters for the (usual) outcome model are chosen to minimize val loss, in contrast, the hyperparameters (learning rate, regularization $\lambda$) for the C-Learner outcome model were selected to minimize the \emph{magnitude of the constant shift} at the end of the first epoch, as in \Cref{section:methods_nn}.
The idea is that if the size of the constant shift is small, then the regularizer is doing a good job of enforcing the constraint.
Although one might think that larger regularization $\lambda$'s would automatically lead to smaller constant shifts, we do not find this to be the case; the best $\lambda$ in our settings (64) is not the largest value that we try (up to 256).
Furthermore, the set of hyperparameter values we choose between in this process are ones that seem to result in reasonable performance, e.g. in terms of MSE on treated units in the validation set.

The best hyperparameters chosen are in \Cref{tab:cc_hparams}.
We trained all models with batch size 64. Learning rates decayed linearly over 10 epochs. Optimization was done using the AdamW optimizer. Weight decay was fixed at $0.1$.

We choose the training epoch with the best val loss for each model (cross entropy for propensity, and MSE for outcome). For C-Learner, we choose the epoch with the best val MSE.

%\tc{maybe include data on how these perform compared to, like, overall MSE on all units}

\subsubsection{Additional Results}
\label{section:more_nlp_results}
We consider three data generating processes as in \Cref{section:civilcomments}, where treatment assignment is parameterized by $l=10^{-4},10^{-3},10^{-2}$, respectively.
Summary statistics of $\what\pi(X)$ are in \Cref{tab:civilcomments_prop}, and summary statistics of $1/\what\pi(X)$ are in \Cref{tab:civilcomments_invprop}.
Hyperparameters for outcome models are in \Cref{tab:cc_hparams}.
Comparisons of estimators in these settings are in
\Cref{tab:nn_4,tab:nn_3,tab:nn_2}.

\begin{table}[H]
\centering
\footnotesize
\begin{tabular}{lrrrr}
\toprule
Value of $l$ & Min $\what \pi(X)$ & Max $\what \pi(X)$ & Std $\what \pi(X)$ & CVar $\what \pi(X)$ \\
\midrule
$l=10^{-2}$ &0.0014 (0.0001) &         0.94 (0.01) &         0.15 (0.002) &         0.0016 (0.0001)         \\
$l=10^{-3}$ &0.0004 (0.0000) &         0.95 (0.01) &         0.15 (0.002) &         0.0004 (0.0000)         \\
$l=10^{-4}$ &0.0003 (0.0000) &         0.96 (0.00) &         0.16 (0.002) &         0.0003 (0.0000)         \\
\bottomrule
\end{tabular}
\caption{Summary statistics of $\what\pi(X)$ for values of hyperparameter $l=10^{-4},10^{-3},10^{-2}$ for \Cref{section:civilcomments} experiments. We show means across 100 dataset draws for each value of $l$, with standard error in parentheses. CVar refers to the mean of the 5\% of smallest values. }
    \label{tab:civilcomments_prop}
\end{table}
\begin{table}[H]
\centering
\footnotesize
\begin{tabular}{lrrrr}
\toprule
Value of $l$ & Min $1/\what \pi(X)$ & Max $1/\what \pi(X)$ & Std $1/\what \pi(X)$ & CVar $1/\what \pi(X)$ \\
\midrule
$l=10^{-2}$ &1.07 (0.01) &         1124.7\;\; (91.0) &         237.1 (23.4) &         1050.1
 \; (85.3) \\
$l=10^{-3}$ &1.06 (0.01) &         3244.0 (190.5) &         757.7 (54.1) &         3090.9 (183.5) \\
$l=10^{-4}$ &1.05 (0.00) &         4126.9 (214.6) &         980.6 (61.4) &         3960.2 (205.7) \\
\bottomrule
\end{tabular}
\caption{Summary statistics of $1/\what\pi(X)$ for values of hyperparameter $l=10^{-4},10^{-3},10^{-2}$ for \Cref{section:civilcomments} experiments. We show means across 100 dataset draws for each value of $l$, with standard error in parentheses. CVar refers to the mean of the 5\% of largest values. }
    \label{tab:civilcomments_invprop}
\end{table}
\begin{table}[H]
\centering
\footnotesize
\begin{tabular}{lrrrrrr}
\toprule
 & $l=10^{-4}$ & $l=10^{-3}$ & $l=10^{-2}$ \\
 \midrule
Outcome model learning rate & $10^{-3}$ & $10^{-3}$ & $10^{-4}$ \\
C-Learner learning rate (best val MSE) & $10^{-4}$ & $10^{-4}$ & $10^{-4}$ \\
C-Learner learning rate (min bias shift) & $10^{-5}$ & $10^{-5}$ & $10^{-5}$ \\
C-Learner $\lambda$ (best val MSE) & 0 & 16 & 16 \\
C-Learner $\lambda$ (min bias shift) & 64 & 64 & 64 \\
\bottomrule
\end{tabular}
\caption{Hyperparameters for $l=10^{-4},10^{-3},10^{-2}$ for \Cref{section:civilcomments}}
    \label{tab:cc_hparams}
\end{table}

\begin{table}[H]
\centering
\footnotesize
\begin{tabular}{lrrrrrr}
\toprule
Method & \multicolumn{2}{c}{Bias} & \multicolumn{2}{c}{Mean Abs Err} & \multicolumn{2}{c}{Coverage}\\
\cmidrule(lr){1-1}
\cmidrule(lr){2-3}
\cmidrule(lr){4-5}
\cmidrule(lr){6-7}
Direct & 0.173 & (0.008) & 0.177 & (0.007) & 0.010 & (0.001)\\
IPW & 0.504 & (0.084) & 0.546 & (0.081) & 0.760 & (0.018)\\
IPW-SN & 0.114 & (0.017) & 0.153 & (0.014) & 0.890 & (0.010)\\
\greymidrule
AIPW & 0.084 & (0.043) & 0.307 & (0.032) & 0.830 & (0.014)\\
AIPW-SN & 0.116 & (0.018) & 0.161 & (0.014) & 0.850 & (0.013)\\
TMLE & -1.264 & (1.361) & 1.802 & (1.355) & 0.810 & (0.015)\\
C-Learner (best val MSE) & 0.103 & (0.015) & 0.141 & (0.011) & 0.870 & (0.011)\\
C-Learner (min bias shift) & \textbf{0.075} & (0.012) & \textbf{0.115} & (0.008) & 0.900 & (0.009)\\
% Direct & 0.139 & (0.007) & 0.140 & (0.006) & 0.010 & (0.001)\\
% IPW & 0.504 & (0.084) & 0.546 & (0.081) & 0.750 & (0.019)\\
% IPW-SN & 0.114 & (0.017) & 0.153 & (0.014) & 0.840 & (0.013)\\
% \greymidrule
% AIPW & 0.124 & (0.047) & 0.316 & (0.037) & 0.840 & (0.013)\\
% AIPW-SN & 0.110 & (0.017) & 0.147 & (0.014) & 0.890 & (0.010)\\
% TMLE & 0.328 & (0.275) & 0.637 & (0.270) & 0.830 & (0.014)\\
% C-Learner & \textbf{0.075} & (0.012) & \textbf{0.115} & (0.008) & 0.860 & (0.012)\\
\bottomrule
\end{tabular}
\caption{Comparison of estimators in the CivilComments~\citep{civilcomments} semi-synthetic dataset over 100 re-drawn datasets, with $l=10^{-4}$. Confidence intervals were set to achieve 95\% confidence. Asymptotically optimal methods are listed beneath the horizontal divider. We highlight the best-performing overall method in \textbf{bold}. Standard errors are displayed within parentheses to the right of the point estimate.}
    \label{tab:nn_4}
\end{table}

\begin{table}[H]
\centering

\footnotesize
\begin{tabular}{lrrrrrr}
\toprule
Method & \multicolumn{2}{c}{Bias} & \multicolumn{2}{c}{Mean Abs Err} & \multicolumn{2}{c}{Coverage}\\
\cmidrule(lr){1-1}
\cmidrule(lr){2-3}
\cmidrule(lr){4-5}
\cmidrule(lr){6-7}
Direct & 0.147 & (0.028) & 0.200 & (0.025) & 0.000 & (0.000)\\
IPW & 0.417 & (0.064) & 0.437 & (0.063) & 0.840 & (0.013)\\
IPW-SN & 0.079 & (0.015) & 0.120 & (0.013) & 0.910 & (0.008)\\
\greymidrule
AIPW & \textbf{-0.056} & (0.052) & 0.344 & (0.039) & 0.950 & (0.005)\\
AIPW-SN & 0.089 & (0.020) & 0.134 & (0.018) & 0.950 & (0.005)\\
TMLE & 0.074 & (0.025) & 0.144 & (0.022) & 0.940 & (0.006)\\
C-Learner (best val MSE) & 0.067 & (0.012) & 0.110 & (0.009) & 0.980 & (0.002)\\
C-Learner (min bias shift) & 0.060 & (0.011) & \textbf{0.099} & (0.008) & 0.990 & (0.001)\\
% Direct & 0.126 & (0.007) & 0.132 & (0.006) & 0.010 & (0.001)\\
% IPW & 0.417 & (0.064) & 0.437 & (0.063) & 0.840 & (0.013)\\
% IPW-SN & 0.079 & (0.015) & 0.120 & (0.013) & 0.900 & (0.009)\\
% \greymidrule
% AIPW & -0.033 & (0.059) & 0.371 & (0.047) & 0.960 & (0.004)\\
% AIPW-SN & 0.062 & (0.017) & 0.128 & (0.013) & 0.950 & (0.005)\\
% TMLE & 0.061 & (0.022) & 0.132 & (0.018) & 0.950 & (0.005)\\
% C-Learner & \textbf{0.060} & (0.011) & \textbf{0.099} & (0.008) & 0.970 & (0.003)\\
\bottomrule
\end{tabular}
\caption{Comparison of estimators in the CivilComments~\citep{civilcomments} semi-synthetic dataset over 100 re-drawn datasets, with $l=10^{-3}$. Confidence intervals were set to achieve 95\% confidence. Asymptotically optimal methods are listed beneath the horizontal divider. We highlight the best-performing overall method in \textbf{bold}. Standard errors are displayed within parentheses to the right of the point estimate.}
    \label{tab:nn_3}
\end{table}

\begin{table}[H]
\centering

\footnotesize
\begin{tabular}{lrrrrrr}
\toprule
Method & \multicolumn{2}{c}{Bias} & \multicolumn{2}{c}{Mean Abs Err} & \multicolumn{2}{c}{Coverage}\\
\cmidrule(lr){1-1}
\cmidrule(lr){2-3}
\cmidrule(lr){4-5}
\cmidrule(lr){6-7}
Direct & 0.091 & (0.007) & 0.099 & (0.006) & 0.040 & (0.004)\\
IPW & 0.346 & (0.041) & 0.353 & (0.040) & 0.710 & (0.021)\\
IPW-SN & 0.004 & (0.005) & 0.038 & (0.003) & 0.950 & (0.005)\\
\greymidrule
AIPW & -0.172 & (0.029) & 0.220 & (0.026) & 0.890 & (0.010)\\
AIPW-SN & 0.002 & (0.005) & \textbf{0.040} & (0.003) & 1.000 & (0.000)\\
TMLE & 0.027 & (0.005) & 0.042 & (0.003) & 0.990 & (0.001)\\
C-Learner (best val MSE) & 0.009 & (0.007) & 0.053 & (0.004) & 1.000 & (0.000)\\
C-Learner (min bias shift) & \textbf{0.001} & (0.006) & 0.045 & (0.003) & 1.000 & (0.000)\\
% Direct & 0.103 & (0.006) & 0.108 & (0.005) & 0.010 & (0.001)\\
% IPW & 0.346 & (0.041) & 0.353 & (0.040) & 0.770 & (0.018)\\
% IPW-SN & 0.004 & (0.005) & 0.038 & (0.003) & 0.990 & (0.001)\\
% \greymidrule
% AIPW & -0.214 & (0.031) & 0.242 & (0.029) & 0.930 & (0.007)\\
% AIPW-SN & 0.003 & (0.005) & \textbf{0.041} & (0.003) & 1.000 & (0.000)\\
% TMLE & 0.033 & (0.004) & 0.043 & (0.003) & 1.000 & (0.000)\\
% C-Learner & \textbf{0.001} & (0.006) & 0.045 & (0.003) & 1.000 & (0.000)\\
\bottomrule
\end{tabular}
\caption{Comparison of estimators in the CivilComments~\citep{civilcomments} semi-synthetic dataset over 100 re-drawn datasets, with $l=10^{-2}$. Confidence intervals were set to achieve 95\% confidence. Asymptotically optimal methods are listed beneath the horizontal divider. We highlight the best-performing \emph{asymptotically optimal} method in \textbf{bold}. Standard errors are displayed within parentheses to the right of the point estimate. C-Learner with minimum-bias-shift selection has MAE 0.045, compared with 0.040 for AIPW-SN and 0.042 for TMLE.}
    \label{tab:nn_2}
\end{table}
%We also plot histograms for the errors of asymptotically optimal estimators in these settings, as well, in \Cref{fig:clearner_hists_23}
%for $l=10^{-4},10^{-3},10^{-2}$. %For $l=10^{-4}$, see
%\Cref{fig:clearner_hist} in  \Cref{section:civilcomments}).
% \ref{fig:clearner_hist_3},
% \ref{fig:clearner_hist_2}, for $l=10^{-4},0.001,0.01$, respectively.
C-Learner errors are comparatively more stable for $l=10^{-4},10^{-3}$. For $l=10^{-2}$, C-Learner is comparable to other asymptotically optimal methods.

\ifdefined\msom\else
\subsection{IHDP Tabular Dataset}
\label{section:ihdp_experiment}
The Infant Health and Development Program (IHDP) dataset is a tabular semisynthetic dataset~\cite{brooks1992effects} that was first introduced as a benchmark for ATE estimation by \citet{hill2011bayesian}. IHDP is based on a randomized experiment studying the effect of specialized healthcare interventions on the cognitive scores of premature infants with low birth weight. The dataset consists of a continuous outcome variable $Y$, a binary treatment variable $A$, and 25 covariates $X$ that affect the outcome variable and are also correlated with the treatment assignment. We use 1000 datasets %from the NPCI package \citep{cole2008constructing}
as in %with the same setup as
\cite{shi2019adapting,chernozhukov2022riesznet}. Covariates and treatment assignments are identical
 across these realizations. For each realization, the target used to score
ATE error is the average of the supplied conditional-mean contrasts over
all 747 covariates, not an average of realized counterfactual outcome
differences. The supplied conditional means and counterfactual outcomes
are never used for fitting or model selection.
%,neal2020realcause} shalit2017estimating, ???

\Needspace{12\baselineskip}
\subsubsection{Gradient Boosted Regression Trees}
\label{section:ihdp_xgb}
\begin{yfrevision}
\yfdelete{We instantiate C-Learner using XGBoost with a custom objective.
The XGBoost T-Learner has MAE 0.104, compared with 0.098 for RieszNet.}

\yf{We use a T-learner: two separate XGBoost outcome regressions for treated
and control units, applying \Cref{alg:boosting} in
\Cref{section:methods_boosting} to each arm with the same training-fold
propensity model. The inverse-propensity targets are $1/\what\pi(X)$ and
$1/\{1-\what\pi(X)\}$, respectively. Both updated outcome models are
evaluated at every covariate in the evaluation fold, and we subtract the
control plug-in mean from the treated plug-in mean. We report
two-fold results, then a distinct experiment without cross-fitting for
comparison with published benchmarks.}

% \yf{\paragraph{Fixed data and target.}
% We use the 1,000 unique, pre-generated IHDP realizations indexed 1--1,000,
% without regenerating covariates, treatment assignments, or outcomes.
% Each has 747 observations (139 treated and 608 controls), 25 covariates,
% and continuous outcomes. Covariates and treatment assignments are identical
% across these realizations. For each realization, the target used to score
% ATE error is the average of the supplied conditional-mean contrasts over
% all 747 covariates, not an average of realized counterfactual outcome
% differences. The supplied conditional means and counterfactual outcomes
% are never used for fitting or model selection.}

\yf{\paragraph{Two-fold results.}
We fix a treatment-stratified split into halves of sizes 374 and 373.
For each evaluation half, initial outcome regressions and the propensity
regression are fitted and tuned using only the opposite half.
Projection and calibration use the evaluation half.
We swap the roles and aggregate the fold estimates
with weights 374/747 and 373/747.}

\yfdelete{Previous MAEs: direct method 0.188; AIPW, AIPW-SN, and TMLE 0.103;
C-Learner 0.104. The debiased estimators perform similarly because the
propensity weights do not vary as much here.}

\begin{table}[t]
\centering
\yf{\footnotesize
\setlength{\tabcolsep}{5pt}
\begin{tabular}{lrr}
\toprule
Method & Mean Abs Err & RMSE\\
\midrule
Direct & 0.125 (0.004) & 0.185\\
\greymidrule
AIPW & 1.611 (0.038) & 2.016\\
AIPW-SN & 0.322 (0.008) & 0.403\\
TMLE & 0.128 (0.003) & 0.165\\
C-Learner & 0.118 (0.003) & 0.152\\
\bottomrule
\end{tabular}}
\caption{\yf{IHDP results with two-fold cross-fitting of initial outcomes
and propensities, over 1,000 fixed realizations. Projection and fluctuation
are fitted on the evaluation half. Parentheses give Monte Carlo standard
errors of mean absolute error; RMSE is reported without a standard error.
TMLE uses a Gaussian, identity-link fluctuation.}}
\label{tb:ihdp}
\end{table}

\yf{\Cref{tb:ihdp} shows that the C-Learner has the lowest
MAE and RMSE in this experiment. Its MAE reduction relative to TMLE is
0.0103 (approximately 8\%), with a paired Monte Carlo 95\% interval
$[0.0089,0.0117]$. The reduction relative to the direct method is 0.0072,
but its paired interval $[-0.0012,0.0155]$ includes zero.
We do observe low overlap: The largest factual inverse-propensity weight is 1,334.7: one
half-sample propensity fit assigns a treated evaluation unit a probability
of approximately $0.000749$. This estimated-weight instability accompanies
AIPW's large error.}

\yfdelete{Two-fold XGB projection and calibration for the ATE.}

\yfdelete{Hyperparameter tuning uses learning rates 0.01, 0.05, 0.1, 0.2,
feature subsampling 0.5, 0.8, 1, and depths 3, 4, 5. The second phase uses
the first phase's hyperparameters and subsampling of 50 percent.}

\yf{\paragraph{Implementation details.}
We reorder the 25 covariates as in the benchmark loader
(columns 7--25 followed by 1--6), without additional covariate scaling.
For both outcome and projection regressions, we search the same 12
combinations: depth $\{2,3,4\}$, learning rate $\{0.03,0.1\}$,
and column subsampling $\{0.8,1\}$. All candidates use squared loss,
histogram trees, row subsampling 0.8, minimum child weight 1, zero
minimum split-loss penalty, $L_2$ penalty 1, and $L_1$ penalty 0.
We allow at most 2,000 trees and stop after 20 rounds without validation
RMSE improvement. Each arm and each regression type is tuned separately;
the selected number of trees is then refitted on all factual-arm
observations of the appropriate fitting population. Projection tuning
minimizes inverse-propensity prediction error, never outcome loss or
ATE error. We use Python XGBoost 3.4.1 with one thread per fit.}

\yf{Within each half, a fixed treatment-stratified approximately 80/20
inner split gives 298/75 or 299/75 training/validation observations.
Outcome normalization uses the mean and standard deviation of the pooled
inner-training outcomes only, shared by the two outcome arms; predictions
are returned to original outcome units. Each projection target is
normalized using its own factual-arm inner-training mean and standard
deviation.  The propensity model is an $L_1$-penalized logistic regression
on all 25 covariates, with 50 log-spaced penalties from $10^{-10}$ to $1$
selected by unstratified three-fold cross-validation AUC entirely within
the nuisance-fitting half, then refitted there. We use the R \texttt{glmnet}
engine with an intercept and its default internal standardization, and no
external propensity preprocessing. The three-fold propensity validation is model selection,
not an additional outer cross-fitting layer. Since $X$ and $A$ are fixed
across realizations, propensity and projection fits are shared across
them, while outcomes and fluctuations are recomputed for each realization.}

%\yf{For reproducibility, the outer split uses seed 123. With zero-based
%fold index $j\in\{0,1\}$, treatment arm $a\in\{0,1\}$, and
%realization index $r=1,\ldots,1000$, the inner
%outcome partition uses seed $8123+j$, propensity fitting uses $123+j$,
%outcome XGBoost fits use $123+1009r+10j+a$, and projection fits use
%$70123+10j+a$. A projection uses the inner partition already assigned to
%its evaluation half. The no-cross-fitting run merges and sorts the two
%inner-training sets and the two validation sets and uses $j=0$ for model
%seeds. All splits and input files are held fixed and recorded in manifests.}

\yf{For every fold and arm, AIPW adds the weighted residual sum divided
by the whole evaluation-fold size, whereas AIPW-SN divides by that
arm's sum of factual inverse-propensity weights. TMLE applies
\eqref{eq:boosting-projection-calibration} in each arm with the raw
inverse-propensity target replacing the fitted projection direction.
Its fluctuation is least squares with an identity link.}

\yf{\paragraph{Without cross-fitting and comparison with published results.}
In a second experiment, we merge the two halves' inner-training indices
and their inner-validation indices, preserving each row's role and
obtaining a 597/150 split, with approximately 80\% of observations used
for internal training. After tuning, outcome and projection
models are refitted on all factual-arm observations; the logistic
propensity is also fitted on all 747 observations. Calibration and
ATE estimation use those same 747 observations, with one coefficient
per arm. The regression routines,
hyperparameter grids, estimator formulas, and 1,000 target realizations
are otherwise the same as in the two-fold experiment.}

\begin{table}[t]
\centering
\yf{\footnotesize
\setlength{\tabcolsep}{5pt}
\begin{tabular}{lr}
\toprule
Method & Mean Abs Err\\
\midrule
Direct & 0.095 (0.002)\\
\greymidrule
AIPW & 0.093 (0.002)\\
AIPW-SN & 0.093 (0.002)\\
TMLE & 0.093 (0.002)\\
C-Learner & 0.093 (0.002)\\
\midrule
RieszNet (DR), published & 0.110 (0.003)\\
ForestRiesz (DR), published & 0.126 (0.004)\\
\bottomrule
\end{tabular}}
\caption{\yf{IHDP without outer cross-fitting. The five XGB rows summarize
our 1,000-realization experiment. RieszNet and ForestRiesz values
and their reported standard errors are taken from Table~1 of
\citet{chernozhukov2022riesznet}. These are published numerical benchmarks, not matched
paired experiments.}}
\label{tb:ihdp-no-crossfit}
\end{table}

\yf{All five XGBoost-based implementations have lower MAE point estimates
(0.0927--0.0948) than the published RieszNet and ForestRiesz DR results
(0.110 and 0.126), as shown in \Cref{tb:ihdp-no-crossfit}.
 The published RieszNet experiment evaluates on the
full sample after an internal training/validation split, but does not
use our full-sample outcome-refit convention. ForestRiesz uses its own
forest/Riesz fitting procedure. Model selection, seeds, and released
file-selection conventions also differ. We therefore make no paired
significance claim, nor a claim that trees generally dominate these
architectures.}

\end{yfrevision}

\begin{yfrevision}
\yf{\paragraph{Coverage with and without cross-fitting.}
\Cref{tb:coverage_ihdp_xgb} evaluates the same point estimates and
realization-specific target as the error tables. The target is
$747^{-1}\sum_i\{\mu_{1}(X_i)-\mu_{0}(X_i)\}$, computed from the
supplied conditional means solely for evaluation. It is neither the
realized counterfactual-outcome difference nor an independently estimated
superpopulation ATE. Covariates, treatments, and partitions are fixed,
while outcome functions and outcomes vary across realizations. The usual
ATE working score includes a covariate-sampling term, so these hit rates
are fixed-benchmark diagnostics, not an established conditional-cohort
or superpopulation coverage guarantee.}

\yf{For nominal 95\% intervals under two-fold fitting, projected C-Learner
has coverage 0.999 and mean width 1.437, compared with 0.995 and
1.438 for Gaussian TMLE. AIPW also has coverage 0.999, but mean width
6.577. These high hit rates reflect conservative or otherwise
miscalibrated intervals relative to this benchmark, not superior
point estimation or exact 95\% calibration. AIPW-SN undercovers
(0.675), despite using its corrected ratio variance.
Without cross-fitting, AIPW, AIPW-SN, TMLE, and projection have coverage
0.959, 0.976, 0.960, and 0.958, with mean widths approximately
0.94--0.96. The naive Direct intervals have coverage 0.783 and 0.812
in the two protocols. No outcome-model-fitting variance correction is
available for that diagnostic.}

\yf{The interval calculations include all 1000 realizations per protocol.
The pooled-centering sensitivity changes
two-fold Direct coverage from 0.783 to 0.784 and leaves the other
reported IHDP coverage proportions unchanged. A difference between the
two protocols is not evidence that cross-fitting worsens inference in
general: the nuisance fits, estimated weights, training sizes, and
targeting coefficients also differ, as discussed above.}
\end{yfrevision}
% Generated by doc/experiments-revised/build_audited_tables.py
\begin{yfrevision}
\begin{table}[t]
\centering
\footnotesize
\setlength{\tabcolsep}{5pt}
\begin{tabular}{lrrrr}
\toprule
& \multicolumn{2}{c}{Two-fold} & \multicolumn{2}{c}{No cross-fitting} \\
Method & Coverage & Mean width & Coverage & Mean width \\
\midrule
\yf{Direct} & 0.783 (0.013) & 0.767 & 0.812 (0.012) & 0.836 \\
\greymidrule
\yf{AIPW} & 0.999 (0.001) & 6.577 & 0.959 (0.006) & 0.941 \\
\yf{AIPW-SN} & 0.675 (0.015) & 1.115 & 0.976 (0.005) & 0.958 \\
\yf{TMLE} & 0.995 (0.002) & 1.438 & 0.960 (0.006) & 0.941 \\
\yf{C-Learner} & 0.999 (0.001) & 1.437 & 0.958 (0.006) & 0.941 \\
\bottomrule
\end{tabular}
\caption{\yf{IHDP nominal 95\% confidence interval coverage over the same 1000 fixed benchmark realizations in each fitting protocol. Coverage is measured against each realization's supplied mean conditional treatment contrast on the 747 subjects, not a superpopulation target. The conventional ATE score includes covariate variation despite this fixed-cohort benchmark. Parentheses are binomial Monte Carlo standard errors; widths are in outcome units. Direct's prediction-only variance is naive. Higher coverage with wider intervals is not a method ranking.}}
\label{tb:coverage_ihdp_xgb}
\end{table}
\end{yfrevision}

\subsubsection{Neural Networks}
\label{section:ihdp_nn}
We also instantiate C-Learner as neural networks (\Cref{section:methods_nn}) on the IHDP dataset.
In particular, we demonstrate how C-Learner can use Riesz representers~\citep{chernozhukov2022riesznet}, or equivalently, propensity models in the ATE setting, that are
learned by other methods.
This demonstrates the versatility of C-Learner as it can use different methods for learning Riesz representers.
In \Cref{tb:ihdp_nn}, we demonstrate C-Learner using Riesz representers
learned by RieszNet~\citep{chernozhukov2022riesznet}. C-Learner achieves very similar performance to RieszNet while using the same Riesz representer. All methods use an outcome model trained in the usual way, except for RieszNet and C-Learner. In all settings where applicable, we use the fitted Riesz representer $A/\what\pi(X)-(1-A)/(1-\what\pi(X))$ learned by RieszNet.

We emphasize that this is not meant to be a performance comparison between RieszNet and C-Learner, for two reasons: first, we present C-Learner as a debiasing framework that is compatible with other methods for causal inference, including RieszNet; and second, we found that propensity scores fitted to the IHDP dataset are not extreme, so we expect C-Learner to perform as well as existing methods, but not to outperform existing methods.

\begin{table}[H]
\centering
\footnotesize
\begin{tabular}{lrrrrrr}
\toprule
Method & \multicolumn{2}{c}{Bias} & \multicolumn{2}{c}{Mean Abs Err} & \multicolumn{2}{c}{Coverage} \\
\cmidrule(lr){1-1}
\cmidrule(lr){2-3}
\cmidrule(lr){4-5}
\cmidrule(lr){6-7}
Direct & -0.005 & (-0.000) & 0.118 & (0.003) & 0.783 & (0.025)\\
IPW & -0.789 & (-0.025) & 0.903 & (0.034) & 0.455 & (0.014)\\
IPW-SN & -0.449 & (-0.014) & 0.654 & (0.035) & 0.819 & (0.026)\\
\greymidrule
AIPW & -0.044 & (-0.001) & 0.106 & (0.003) & 0.940 & (0.030)\\
AIPW-SN & -0.042 & (-0.001) & 0.106 & (0.003) & 0.961 & (0.030)\\
RieszNet (``DR'') & \textbf{-0.033} & (-0.001) & \textbf{0.098} & (0.003) & 0.972 & (0.031)\\
C-Learner & -0.038 & (-0.001) & 0.098 & (0.002) & 0.955 & (0.030)\\
\bottomrule
\end{tabular}
\caption{Comparison of neural network based estimators in the IHDP semi-synthetic dataset over 1000 simulations. Asymptotically optimal methods are listed beneath the horizontal divider. We highlight the best-performing \emph{asymptotically optimal} method in \textbf{bold}. Standard errors are displayed within parentheses to the right of the point estimate. All methods use an outcome model trained in the usual way, except for RieszNet and C-Learner. In all settings where applicable, we use the Riesz representer learned by RieszNet.}
\label{tb:ihdp_nn}
\end{table}

\paragraph{Training Procedure}
We use lightly modified RieszNet code \citep{chernozhukov2022riesznet} to generate the Riesz representers used in all methods.
We use the same neural network architecture for RieszNet outcome models for C-Learner.
\begin{yfrevision}
\yfdelete{We use the same training, validation, and evaluation splits as in the XGBoost subsection.}
\yf{We use an 80/20 training--validation split and evaluate on the full sample, without cross-fitting.}
\end{yfrevision}
Hyperparameters (learning rate, $\lambda$) are selected for best mean squared error on $\what{P}_{\rm val}$ for each individual dataset.
For the training objective, we largely follow \Cref{section:methods_nn}, with the following modifications for the full ATE rather than assuming that $Y(0)=0$: we use the mean squared error across the full dataset (as in RieszNet), and we have separate penalties and bias shifts for $A=1$ and $A=0$.
As in RieszNet training, there is a phase of ``pre-training'' with a higher learning rate, and then a phase with a lower learning rate; for ``pre-training'' we sweep over learning rates of $\{10^{-3},10^{-4}\}$ for 100 epochs; regular training uses a learning rate of $10^{-5}$ for 600 epochs. Both use the Adam optimizer. For $\lambda$ we sweep over values of
$\{0,0.01,0.02,0.04,0.1,0.2,0.4,1,2,4,8\}$.
Our implementation of RieszNet as a baseline is almost identical to that of the original paper, except for how we set the random seed, for initializing neural networks and for choosing train and test split.

\paragraph{Comparisons with results in RieszNet paper}
Our mean absolute error results in \Cref{section:ihdp_nn},
including the ones using regular RieszNet (including their algorithm code and hyperparameters), often outperform the results reported in the original RieszNet paper \citep{chernozhukov2022riesznet}.
One difference between our code and the RieszNet code is the random-seed setting, which affects both the train/test split and neural-network weight initialization. We compare the supplied RieszNet seed setting with a fixed seed of 123 for every dataset. The RieszNet mean absolute error is smaller with the fixed seed.
%In particular, all methods with seeds set the second way have better mean absolute error than from the original RieszNet code:
Compare \Cref{tab:ihdp_orig} and \Cref{tab:ihdp_seeds}.
We use that fixed-seed setting for the neural-network results in \Cref{tb:ihdp_nn}.
\begin{yfrevision}
\begin{table}[H]
\centering
\footnotesize
\begin{tabular}{lr}
\toprule
Method & Mean absolute error (s.e.) \\
 \midrule
Doubly robust & 0.115 (0.003) \\
Direct & 0.124 (0.004) \\
IPW & 0.801 (0.040) \\
\bottomrule
\end{tabular}
\caption{\yfdelete{Original RieszNet IHDP results}
\yf{Our run of the RieszNet code with its supplied seed setting}}
    \label{tab:ihdp_orig}
\end{table}
\end{yfrevision}

\begin{table}[H]
\centering
\footnotesize
\begin{tabular}{lr}
\toprule
Method & Mean absolute error (s.e.) \\
 \midrule
Doubly robust & 0.098 (0.003) \\
Direct & 0.118 (0.003) \\
IPW & 0.903 (0.034) \\
\bottomrule
\end{tabular}
\caption{RieszNet IHDP results, with fixed seeds}
    \label{tab:ihdp_seeds}
\end{table}

\paragraph{Relative Performance}
We've seen in \Cref{section:civilcomments} that C-Learner tends to outperform one-step estimation and targeting in settings with low overlap. Here, it merely performs about the same as one-step estimation and targeting. It seems as though overlap is not very low in IHDP, however, based on propensity scores from trained models not taking on extreme values.
Using RieszNet's learned Riesz representers, the smallest value for $\textrm{min}(\what\pi,1-\what\pi)$ across all IHDP datasets is 0.11, which is a fair bit of overlap, especially in comparison to, say, the settings with low overlap in \Cref{section:civilcomments} where C-Learner outperformed one-step estimation and targeting methods.
%\tc{numbers maybe}
Thus, here it is expected that C-Learner matches the performance of other methods. We further emphasize that C-Learner is compatible with methods such as RieszNet, making C-Learner generally applicable.
\begin{yfrevision}
\yfdelete{Across all experiments, C-Learner is comparable to or better than one-step estimation and targeting.}
\yf{In this neural-network experiment, C-Learner has similar MAE to RieszNet.}
\end{yfrevision}

% We train our C-Learner very similarly to outcome models in RieszNet in \cite{chernozhukov2022riesznet}, and use the same neural net architecture and data splits.
% The primary difference is in the training objective: instead of using the training objective from RieszNet to learn both the Riesz representer and outcome model (as in \cite{chernozhukov2022riesznet}),
% we use the Riesz representer learned from RieszNet and our loss in \ref{eq:c_learner_def} \tc{should probably explain the penalty}.
% We also largely use the same hyperparameters as in \cite{chernozhukov2022riesznet}, except that (1) we try a range of penalty weights to encourage the constraint, (2) we try multiple learning rates to accommodate any changes from the constraint, and (3) we use fewer maximum training epochs (300 instead of 600). For each dataset within IHDP, we choose the best penalty weight and learning rate based on the validation split.

%\tc{more details about the objective, since it is slightly different?}

%\tc{their CIs are calculated weird; how much to disucss}

\fi

%%% Local Variables:
%%% mode: latex
%%% TeX-master: "../main"
%%% End:

\end{document}